\documentclass{article}

\usepackage{amsmath,amsfonts,bm}

\def\eqref#1{equation~\ref{#1}}

\def\1{\bm{1}}

\newcommand{\test}{\mathcal{D_{\mathrm{test}}}}

\def\vtheta{{\bm{\theta}}}

\def\ve{{\bm{e}}}

\def\vx{{\bm{x}}}
\def\vy{{\bm{y}}}
\def\vz{{\bm{z}}}

\DeclareMathAlphabet{\mathsfit}{\encodingdefault}{\sfdefault}{m}{sl}
\SetMathAlphabet{\mathsfit}{bold}{\encodingdefault}{\sfdefault}{bx}{n}

\DeclareMathOperator*{\argmax}{arg\,max}
\DeclareMathOperator*{\argmin}{arg\,min}

\mathchardef\mhyphen="2D
\usepackage{paralist}
\usepackage{tikz,pgfplots}
\pgfplotsset{compat=1.18}
\usepackage{subcaption}
\usepackage{adjustbox}
\usepackage{multirow}
\usepackage{morefloats}

\usepackage{pgfkeys}
    \newenvironment{customlegend}[1][]{%
        \begingroup
        \csname pgfplots@init@cleared@structures\endcsname
        \pgfplotsset{#1}%
    }{%
        \csname pgfplots@createlegend\endcsname
        \endgroup
    }%
    \def\addlegendimage{\csname pgfplots@addlegendimage\endcsname}
\usetikzlibrary{backgrounds}

\usepackage{float}

\usepackage{microtype}
\usepackage{graphicx}
\usepackage{booktabs} 

\usepackage{hyperref}

\usepackage[accepted]{icml2023}

\usepackage{amsmath}
\usepackage{amssymb}
\usepackage{mathtools}
\usepackage{amsthm}

\usepackage[capitalize,noabbrev]{cleveref}

\theoremstyle{plain}
\newtheorem{theorem}{Theorem}[section]

\newtheorem{lemma}[theorem]{Lemma}
\newtheorem{corollary}[theorem]{Corollary}
\theoremstyle{definition}
\newtheorem{definition}[theorem]{Definition}

\theoremstyle{remark}

\usepackage[textsize=tiny]{todonotes}

\icmltitlerunning{Computational Asymmetries in Robust Classification}

\usepackage{crossreftools}
\newcommand{\EUFlag}{%
  \begingroup\normalfont
  \includegraphics[height=1.5\fontcharht\font`\B]{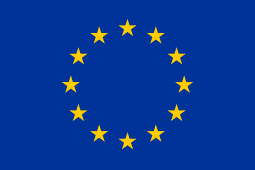}%
  \endgroup
}

\begin{document}

\twocolumn[
\icmltitle{Computational Asymmetries in Robust Classification}



\icmlsetsymbol{equal}{*}

\begin{icmlauthorlist}
\icmlauthor{Samuele Marro}{yyy}
\icmlauthor{Michele Lombardi}{yyy}
\end{icmlauthorlist}

\icmlaffiliation{yyy}{Department of Computer Science, University of Bologna}

\icmlcorrespondingauthor{Samuele Marro}{samuele.marro@unibo.it}

\icmlkeywords{ICML, Machine Learning, Adversarial Attacks, Complexity, Polynomial Hierarchy}

\vskip 0.3in
]



\printAffiliationsAndNotice{}  

\begin{abstract}
In the context of adversarial robustness, we make three strongly related contributions.
First, we prove that while attacking ReLU classifiers is $\mathit{NP}$-hard, ensuring their robustness at training time is $\Sigma^2_P$-hard (even on a single example).
This asymmetry provides a rationale for the fact that robust classifications approaches are frequently fooled in the literature.
Second, we show that inference-time robustness certificates are not affected by this asymmetry, by introducing a proof-of-concept approach named Counter-Attack (CA). 
Indeed, CA displays a reversed asymmetry: running the defense is $\mathit{NP}$-hard, while attacking it is $\Sigma_2^P$-hard.
Finally, motivated by our previous result, we argue that adversarial attacks can be used in the context of robustness certification, and provide an empirical evaluation of their effectiveness.
As a byproduct of this process, we also release UG100, a benchmark dataset for adversarial attacks.

\end{abstract}

\section{Introduction}
\label{sec:introduction}

Adversarial attacks, i.e. algorithms designed to fool machine learning models, represent a significant threat to the applicability of such models in real-world contexts \citep{brown2017adversarial,brendel2019accurate, wu2020making}. Despite years of research effort, countermeasures (i.e. ``defenses'') to adversarial attacks are frequently fooled by applying small tweaks to existing techniques \citep{carlini2016defensive, carlini2017adversarial, he2017adversarial, hosseini2019odds, tramer2020adaptive,croce2022evaluating}.
We argue that this pattern is due to differences between the fundamental mathematical problems that defenses and attacks need to tackle, and we investigate this topic by providing three contributions.

First, we prove a set of theoretical results about the complexity of attack and training-time defense problems, including the fact that \emph{attacking a ReLU classifier is $\mathit{NP}$-hard in the general case, while finding a parameter set that makes a ReLU classifier robust on even a single input is $\Sigma_2^P$-hard}. To the best of our knowledge, this is the first complexity bound for general ReLU classifiers, and the main contribution of this work.
We also provide more general bounds for non-polynomial classifiers, and show in particular that an $A$-time classifier can be attacked in $\mathit{NP}^A$ time.
Instead of using a PAC-like formalization, we rely on a worst-case semantic of robustness. This approach results in a formalization that is both more easier to deal with and independent of data distribution assumptions, while still \emph{providing a rationale for difficulties in training robust classifiers} that are well-known in the related literature. Our proofs also lay the ground work for identifying tractable classes of defenses.


Second, we prove by means of an example that \emph{inference-time defenses can sidestep the asymmetry}.
Our witness is a proof-of-concept approach, referred to as Counter-Attack (CA), that evaluates robustness on the fly for a specific input (w.r.t. to a maximum distance $\varepsilon$) by running an adversarial attack.
Properties enjoyed by this technique are likely to extend to other inference-time defense methods, if they are based on similar principles.
Notably, when built over an exact attack, 
\emph{generating a certificate is $\mathit{NP}$-hard} in the worst case, \emph{$\varepsilon$-bounded attacks are impossible}, and \emph{attacking using perturbations of magnitude $\varepsilon^\prime > \varepsilon$ is $\Sigma_2^P$-hard}.
On the other hand, using a non-exact attack results in partial guarantees (no false positives for heuristic attacks, no false negatives for bounding techniques).


Finally, since our results emphasize the connection between verification and attack problems, we provide an empirical investigation of the use of heuristic attacks for verification.
\emph{We found heuristic attacks to be high-quality approximators for exact decision boundary distances}: a pool of seven heuristic attacks provided an accurate (average over-estimate between 2.04\% and 4.65\%) and predictable (average $R^2 > 0.99$) approximation of the true optimum for small-scale Neural Networks trained on the MNIST and CIFAR10 datasets.
We release\footnote{All our code, models, and data are available
under MIT license at \href{https://github.com/samuelemarro/counter-attack}{https://github.com/samuelemarro/counter-attack}.} our benchmarks and adversarial examples (both exact and heuristic) in a new dataset, named UG100.



Overall, we hope our contributions can support future research by highlighting potential structural challenges, pointing out key sources of complexity, inspiring research on heuristics and tractable classes, and suggesting alternative perspectives on how to build robust classifiers.


\section{Background and Formalization} 
\label{sec:background}





In this section, we introduce key definitions (adapted from \citet{dreossi2019formalization}) that we will use to frame our results.
Our aim is to capture the key traits shared by most of the literature on adversarial attacks, so as to identify properties that are valid under broad assumptions.





\paragraph{Adversarial Attacks and Robustness}
We start by defining the concept of \emph{adversarial example}, which intuitively represents a modification of a legitimate input that is so limited as to be inconsequential for a human observer, but sufficient to mislead a target model.
Formally, let $f: X \to \{1, \dots, N\}$ be a discrete classifier. Let $B_p(\vx, \varepsilon) = \{ \vx^\prime \in X \, | \, \|\vx - \vx^\prime\|_p \leq \varepsilon\}$ be a $L^p$ ball of radius $\varepsilon$ and center $\vx$. Then we have:
%

\begin{definition}[Adversarial Example]
\label{def:adversarialExample}
Given an input $\vx$, a threshold $\varepsilon$, and a $L^p$ norm\footnote{We use the term ``norm'' for $0 < p < 1$ even if in such cases the $L^p$ function is not subadditive.}, an adversarial example is an input $\vx' \in B_p(\vx, \varepsilon)$ such that $f(\vx^\prime) \in C(\vx)$, where $C(\vx) \subseteq \{1, \dots, N\} \setminus \{f(\vx)\}$.
\end{definition} 

This definition is a simplification compared to human perception, but it is adequate for a sufficiently small $\varepsilon$, and it is adopted in most of the relevant literature.
An \emph{adversarial attack} can then be viewed as an optimization procedure that attempts to find an adversarial example.
We define an adversarial attack for a classifier $f$ as a function $a_{f,p}: X \to X$ that solves the following optimization problem:
\begin{equation}
     \argmin_{\vx^\prime \in X} \{\|\vx^\prime - \vx\|_p \mid f(\vx^\prime) \in C(\vx)\}
     \label{eqn:attackAsOptimization}
\end{equation}
The attack is considered successful if the returned solution $\vx^\prime = a_{f, p}(\vx)$ also satisfies $\|\vx^\prime - \vx\|_p \leq \varepsilon$.
We say that an attack is \emph{exact} if it solves \Cref{eqn:attackAsOptimization} to optimality (or, in the case of its decision variant, if it succeeds if and only if a solution exists); otherwise, we say that the attack is \emph{heuristic}.
An attack is said to be \emph{targeted} if $C(\vx) = C_{t,y^\prime}(\vx) = \{y'\}$ with $y^\prime \neq f(\vx)$; it is instead \emph{untargeted} if $C_u(\vx) = \{1, \dots, N\} \setminus \{f(\vx)\}$. We define the \emph{decision boundary distance} $d_p^*(\vx)$ of a given input $\vx$
as the minimum $L^p$ distance between $\vx$ and another input $\vx'$ such that $f(\vx) \neq f(\vx')$. This is also the value of $\|a_{f, p}(\vx) - \vx\|_p$ for an exact, untargeted, attack. 

Intuitively, a classifier is \emph{robust w.r.t. an example $\vx$} iff $\vx$ cannot be successfully attacked. Formally:

\begin{definition}[($\varepsilon$, $p$)-Local Robustness]
\label{def:localRobustness}
A discrete classifier $f$ is ($\varepsilon$, $p$)-locally robust w.r.t. an example $\vx \in X$ iff $\forall \vx' \in B_p(\vx, \varepsilon)$ we have $f(\vx^\prime) = f(\vx)$.
\end{definition}

Under this definition, finding a parameter set $\vtheta$ that makes a classifier $f_\vtheta$ robust on $\vx_0$ can be seen as solving the following constraint satisfaction problem:
\begin{equation}
    \label{eqn:findingRobust}
    \text{find } \vtheta \text{ s. t. } \forall \vx^\prime \in B_p(\vx_0, \varepsilon) . f_\vtheta (x^\prime) = f_\vtheta(x)
\end{equation}

which usually features an additional constraint on the minimum clean accuracy of the model (although we make no assumptions on this front). Note that classifiers are usually expected to be robust on more than one point. However, we will show that the computational asymmetry exists even if we require robustness on a single point.

A common optimization reformulation of \cref{eqn:findingRobust}, which enforces robustness \emph{and} accuracy, is the nested optimization problem used for adversarial training in \citet{madry2018towards}. Specifically, if we have a single ground truth data point $\langle \vx_0, y\rangle$, the optimization problem is:

\begin{equation}
    \label{eqn:robustTraining}
    \argmin_{\vtheta} \max_{\vx^\prime \in B_p(\vx_0, \varepsilon)} \mathcal L (\vtheta, \vx^\prime, y_0)
\end{equation}

where $\mathcal L$ is a proxy for $f_\vtheta(\vx^\prime) = y$ (e.g. the cross entropy loss between $f_\vtheta(\vx^\prime)$ and $y$). The link between $\exists \forall$ queries (such as that in \Cref{eqn:findingRobust} and nested optimization problems (such as that in \Cref{eqn:robustTraining}) underlies the intuition of several of our theoretical results (see \Cref{sec:intuition}).

\paragraph{ReLU Networks and FSFP Spaces}
%
Additionally, our results rely on definitions of ReLU networks and FSFP spaces.
\begin{definition}[ReLU network]
A ReLU network is a composition of sum, multiplication by a constant, and ReLU activation, where $ReLU: \mathbb{R} \to \mathbb{R}_0^+$ is defined as $ReLU(x) = max(x, 0)$.
\end{definition}
Note that any hardness result for ReLU classifiers also extends to general classifiers.

Fixed-Size Fixed-Precision (FSFP) spaces, on the other hand, capture two common assumptions about real-world input spaces: all inputs can be represented with the same number of bits and there exists a positive minorant of the distance between inputs.

\begin{definition}[Fixed-Size Fixed-Precision space]
Given a real $p > 0$, a space $X \subseteq \mathbb{R}^n$ is FSFP if there exists a $\nu \in \mathbb{N}$ such that $\forall \vx. |r(\vx^\prime)| \leq \nu$ (where $|r(\vx)|$ is the size of the representation of $\vx$) and there exists a $\mu \in \mathbb{R}$ such that $\mu > 0$ and $\forall \vx , \vx^\prime \in X. \left ( \|\vx^\prime - \vx\|_p < \mu \implies \vx = \vx^\prime \right )$.
\end{definition}

Examples of FSFP spaces include most image encodings, as well as 32-bit and 64-bit IEE754 tensors. Examples of non-FSFP spaces include the set of all rational numbers in an interval. 
Similarly to ReLU networks, hardness results for FSFP spaces also apply to more general spaces.

\paragraph{$\Sigma_2^P$ Complexity} Several of our theoretical results concern complexity classes in the Polynomial Hierarchy such as $\Sigma_2^P$. $\Sigma_2^P$ is the class of problems that can be solved in $\mathit{NP}$ time if we have an oracle that solves an $\mathit{NP}$-time problem in $O(1)$.  $\Sigma_2^P$-hard problems include finding a strong Nash equilibrium \citep{gottlob2011pure} and $\mathit{co}\Pi_2 \text{3SAT}$ \citep{stockmeyer1976polynomial}. A notable conjecture is the Polynomial Hierarchy conjecture \citep{stockmeyer1976polynomial}, a generalization of the $\mathit{P} \neq \mathit{NP}$ conjecture which states that the Polynomial Hierarchy does not collapse  (i.e. $P \subsetneq \mathit{NP} \subsetneq \Sigma_2^P \subsetneq \Sigma_3^P \dots$). In other words, under broad assumptions, we cannot solve a $\Sigma_2^P$-hard problem efficiently even if we can solve $\mathit{NP}$-hard problems in constant time.



\section{An Asymmetrical Setting}

In this section, we prove the existence of a structural asymmetry between the computational classes of attack and training-time defense problems (barring the collapse of the Polynomial Hierarchy) by studying their decision versions\footnote{Note that hardness results for decision problems trivially extend to their corresponding optimization variants.}. 
While the asymmetry is worst-case in nature, it holds under broad assumptions and provides an explanation for why attacks seem to outperform defenses in practice.

\subsection{Intuition}
\label{sec:intuition}

The intuition behind our theorems consists in three main observations:
\begin{compactitem}
    \item ReLU networks, due to their expressive power, are capable of computing input-output relations that are \emph{at least as complex} as Boolean formulae;
    \item Attacking usually requires solving an optimization problem, whose decision variant (finding \emph{any} adversarial example) can be expressed as an $\exists$ query;
    \item Training a robust classifier, on the other hand, usually requires solving a nested optimization problem, whose decision variant (finding \emph{any} robust parameter set) can be expressed as an $\exists \forall$ query.
\end{compactitem}

From these considerations, we show that solving $3\text{SAT}$ can be reduced to attacking the ReLU classifier that computes the corresponding Boolean formula, and thus that attacking a ReLU classifier is $NP$-hard (\Cref{the:attackNp}).

We then prove that, given a 3CNF formula $z(\vx, \vy)$, it is possible to build a ReLU classifier $f_{\vx}(\vy)$ (where $\vx$ are parameters and $\vy$ are inputs) that computes the same formula. We use this result to prove that $co\Pi_23\text{SAT}$ (a subclass of $TQBF$ that is known to be $\Sigma_2^P$-hard) can be reduced to finding a parameter set that makes $f$ robust, which means that the latter is $\Sigma_2^P$-hard (\Cref{the:defenseSigma}).

Note that, when performing the reductions, we choose the ReLU networks that we need to solve the corresponding problem without considering how likely they are to arise in natural settings. This approach (which is common in proofs by reduction) allows us to study the worst-case complexity of both tasks without making assumptions on the training distribution or the specifics of the learning algorithm.
Studying the average-case complexity of such tasks would of course be of great importance, however: 1) such an approach would require to introduce assumptions about the training distribution; and 2) despite the recent advancements in fields such as PAC learning, average case proof in this setting are still very difficult to obtain except in very specific cases (see \Cref{sec:related}). We hope that our theoretical contributions will allow future researchers to extend our work to average-case results.

In short, while our theorems rely on specific instances of ReLU classifiers, they capture very general phenomena: ReLU networks can learn functions that are at least as complex as Boolean formulae, and robust training requires solving a nested optimization problem. The proofs thus provide an intuition on the formal mechanisms that underly the computational asymmetries, while at the same time outlining directions for studying tractable classes (since both $3SAT$ and $TQBF$ are extensively studied in the literature).


\subsection{Preliminaries}

We begin by extending the work of \citet{katz2017reluplex}, who showed that proving linear properties of ReLU networks is $\mathit{NP}$-complete. Specifically, we prove that the theorem holds even in the special case of adversarial attacks:

\begingroup
\makeatletter
\apptocmd{\thetheorem}{\unless\ifx\protect\@unexpandable@protect\protect\footnote{The proofs of all our theorems and corollaries can be found in  the appendices.}\fi}{}{}
\makeatother
\begin{theorem}[Untargeted $L^\infty$ attacks against ReLU classifiers are $\mathit{NP}$-complete] 
\label{the:attackNp}
Let $U\mhyphen{}ATT_p$ be the set of all tuples $\langle \vx, \varepsilon, f \rangle$ such that:
\begin{equation}
    \exists \vx^\prime \in B_p(\vx, \varepsilon) . f (\vx^\prime) \neq f(\vx)
\end{equation}

where $\vx \in X$, $X$ is a FSFP space and $f$ is a ReLU classifier. 
Then $U\mhyphen{}ATT_\infty$ is $\mathit{NP}$-complete.
\end{theorem}

\endgroup
\begin{corollary}
\label{cor:lpAttack}
For every $0 < p \leq \infty$, $U\mhyphen{}ATT_p$ is $\mathit{NP}$-complete.
\end{corollary}
\begin{corollary} 
\label{cor:attackTargeted}
Targeted $L^p$ attacks (for $0 < p \leq \infty)$ against ReLU classifiers are $\mathit{NP}$-complete.
\end{corollary}
\begin{corollary}
\label{cor:attackPolynomial}
\Cref{the:attackNp} holds even if we consider the more general set of polynomial-time classifiers w.r.t. the size of the tuple.
\end{corollary}

A consequence of \Cref{the:attackNp} is that the complementary task of attacking, i.e. proving that no adversarial example exists (which is equivalent to proving that the classifier is locally robust on an input), is $\mathit{coNP}$-complete.

We then provide a more general upper bound that holds for classifiers in any complexity class:

\begin{theorem}[Untargeted $L^p$ attacks against $A$-time classifiers are in $NP^A$]
\label{the:attackRelative}
Let $A$ be a complexity class, let $f$ be a classifier, let $Z_f = \{ \langle \vx, y \rangle \mid y = f(\vx), \vx \in X \}$ and let $U\mhyphen{}ATT_p(f) = \{\langle \vx, \varepsilon, g \rangle \in U\mhyphen{}ATT^\prime_p \mid g = f\}$, where $U\mhyphen{}ATT^\prime_p$ is the same as $U\mhyphen{}ATT_p$ but without the ReLU classifier restriction. If $Z_f \in A$, then for every $0 < p \leq \infty$, $U\mhyphen{}ATT_p(f) \in NP^A$. 
\end{theorem} 
\begin{corollary}
\label{cor:polynomialHierarchy}
For every $0 < p \leq \infty$, if $Z_f \in \Sigma_n^P$, then $U\mhyphen{}ATT_p(f) \in \Sigma_{n+1}^P$.
\end{corollary}

As a consequence, if $Z_f \in P$, then $U\mhyphen{}ATT_p(f) \in NP$.
Informally, \Cref{the:attackNp} establishes that, under broad assumptions, evaluating and attacking a general classifier are in complexity classes that are strongly conjectured to be distinct, with the attack problem being the harder one. Note that, in some special cases, one can obtain polynomial-time classifiers with polynomial-time attacks by placing additional restrictions on the input distribution and/or the structure of the classifier. Refer to \Cref{sec:related} for an overview of such approaches.

\subsection{Complexity of Robust Training}

We then proceed to prove our main result, i.e. that \emph{finding a robust parameter set}, as formalized by our semantic, is in a distinct complexity class compared to the attack problem.


\begin{theorem}[Finding a set of parameters that make a ReLU network $(\varepsilon, p)$-locally robust on an input is $\Sigma_2^P$-complete] 
\label{the:defenseSigma}
Let $PL\mhyphen{}ROB_p$ be the set of tuples $\langle \vx, \varepsilon, f_\vtheta, v \rangle$ such that:
\begin{equation}
    \label{eqn:plRobCondition}
    \exists \vtheta^\prime . \left (v_f(\vtheta^\prime) = 1 \implies \forall \vx^\prime \in B_p(\vx, \varepsilon) . f_{\vtheta^\prime}(\vx^\prime) = f_{\vtheta^\prime}(\vx) \right )
\end{equation}
where $\vx \in X$, $X$ is a FSFP space and $v_f$ is a polynomial-time function that is 1 iff the input is a valid parameter set for $f$.
Then $PL\mhyphen{}ROB_\infty$ is $\Sigma_2^P$-complete.
\end{theorem}
\begin{corollary}
\label{cor:lpDefense}
$PL\mhyphen{}ROB_p$ is $\Sigma_2^P$-complete for all $0 < p \leq \infty$.
\end{corollary}
\begin{corollary}
\label{cor:defensePolynomial}
\Cref{the:defenseSigma} holds even if, instead of ReLU classifiers, we consider the more general set of polynomial-time classifiers w.r.t. the size of the tuple.
\end{corollary}





The $\Sigma_2^P$ complexity class includes $\mathit{NP}$ and is conjectured to be strictly harder (as part of the Polynomial Hierarchy conjecture).
In other words, if the Polynomial Hierarchy conjecture holds, \textbf{robustly training a general ReLU classifier is strictly harder than attacking it}.
Note that our results hold \emph{in the worst-case}, meaning there can be specific circumstances under which guaranteed robustness could be achieved with reasonable effort.
However, in research fields where similar asymmetries are found, they tend to translate into practically meaningful difficulty gaps: for example, $\exists \forall$ Quantified Booolean Formula problems (which are $\Sigma_2^P$-complete) are in practice much harder to solve than pure SAT problems (which are $\mathit{NP}$-complete).

We conjecture this is also the case for our result, as it mirrors the key elements in the SAT/TQBF analogy.
First, generic classifiers can learn (and are known to learn) \emph{complex input-output mappings with many local optima}.
Second, while attacks rely on existential quantification (finding an example), \emph{achieving robustness requires addressing a universally quantified problem} (since we need to guarantee the same prediction on all neighboring points).

\subsection{Relevance of the Result and Related Work}
\label{sec:related}

In this section we discuss the significance of our results, both on the theoretical and the practical side.


\paragraph{Theoretical Relevance}



As we mentioned, results about polynomial-time attack and/or robustness certificates are available, but under restrictive assumptions.
For example, \citet{mahloujifar2019can} showed that there exist exact polynomial-time attacks against classifiers trained on product distributions.
Similarly, \citet{awasthi2019robustness} showed that for degree-2 polynomial threshold functions there exists a polynomial-time algorithm that either proves that the model is robust or finds an adversarial example.

Other complexity lower bounds also exist, but again they apply under specific conditions. \citet{degwekar2019computational}, extending the work of \citet{bubeck2018adversarial} and \citet{bubeck2019adversarial}, showed that there exist certain cryptography-inspired classification tasks such that learning a classifier with a robust accuracy of 99\% is as hard as solving the Learning Parity with Noise problem (which is $\mathit{NP}$-hard). On the other hand, \citet{song2021cryptographic} showed that learning a single periodic neuron over noisy isotropic Gaussian distributions in polynomial time would imply that the Shortest Vector Problem (conjectured to be $\mathit{NP}$-hard) can be solved in polynomial time.

Finally, \citet{garg2020adversarially} provided an average-case complexity analysis, by introducing assumptions on the data-generation process. In particular, by requiring attackers to provide a valid cryptographic signature for inputs, it is possible to prevent attacks with limited computational resources from fooling the model in polynomial time.

Compared to the above results, both \Cref{the:attackNp}
and \Cref{the:defenseSigma} apply to a wider class of models.
In fact, to the best of our knowledge, \textbf{\Cref{the:defenseSigma} is the first robust training complexity bound for general ReLU classifiers}.




\paragraph{Empirical Relevance}

\Cref{the:attackNp,the:defenseSigma} imply that training-time defenses can be strictly (and significantly) harder than attacks.
This result is consistent with a recurring pattern in the literature where new defenses are routinely broken.
For example, defensive distillation \citep{papernot2016distillation} was broken by \citet{carlini2016defensive}.
Carlini also showed that several adversarial example detectors \citep{carlini2017adversarial}, as well as model-based purifiers \citep{carlini2017magnet} can be fooled.
Similarly, \citet{he2017adversarial} showed that ensembles of weak defenses can be fooled, while the defense of \citet{roth2019odds} was fooled by \citet{hosseini2019odds}.
Finally, \citet{tramer2020adaptive} and \citet{croce2022evaluating} broke a variety of adaptive defenses.

While our theorems formally hold only in the worst case, they rely at their core on two properties that can be expected to be practically relevant, and namely: 1) that NNs can learn response surfaces that are as complex as Boolean formulas, and 2) that robustness involves universal rather then existential quantification.
For this reason, we think that \textbf{the asymmetry we identified can provide valuable insight into a large body of empirical work}.

\subsection{Additional Sources of Asymmetry}
\label{sec:additionalSources}

On top of our identified structural difference, there are additional factors that may provide an advantage to the attacker, despite the fact that they lack a formal characterization at the moment of writing.
We review them in this section, both as promising directions for future theoretical research, and since awareness of them can support efforts to build more robust defenses.

First, the attacker can gather information about the target model, e.g. by using genuine queries \citep{papernot2017practical}, while the defender does not have such an advantage. As a result, the defender often needs to either make assumptions about adversarial examples \citep{hendrycks2016early, roth2019odds} or train models to identify common properties \citep{feinman2017detecting, grosse2017statistical}. These assumptions can be exploited, such as in the case of \citet{carlini2017adversarial}, who generated adversarial examples that did not have the expected properties.

Second, the attacker can focus on one input at the time, while the defender has to guarantee robustness on a large subset of the input space.
This weakness can be exploited: for example, MagNet \citep{meng2017magnet} relies on a model of the entire genuine distribution, which can be sometimes inaccurate. \citet{carlini2017magnet} broke MagNet by searching
for examples that were both classified differently and mistakenly considered genuine. 

Finally, defenses cannot significantly compromise the accuracy of a model. Adversarial training, for example, often reduces the clean accuracy of the model \citep{madry2018towards}, leading to a trade-off between accuracy and robustness.

All of these factors can, depending on the application context, exacerbate the effects of the structural asymmetry; for this reason, minimizing their impact represents another important research direction.







\section{Sidestepping the Asymmetry}
\label{sec:counterAttack}

An important aspect of our theoretical results is that they apply only to building robust classifiers at training time.
This leaves open the possibility to \emph{sidestep the asymmetry by focusing on defenses that operate at inference time}. Here, we prove that this indeed the case by means of an example, and characterize its properties since they can be expected to hold for other systems based on the same principles.

Our witness is a proof-of-concept robustness checker, called Counter-Attack (CA), that relies on adversarial attacks to compute robustness certificates at inference time, w.r.t. to a maximum $p$-norm $\varepsilon$.
CA can compute certificates in $\mathit{NP}$-time, and attacking it beyond its intended certification radius is $\Sigma_2^P$-hard, proving that \textbf{inference-time defenses can flip the attack-defense asymmetry}.
While an argument can be made that CA is usable as it is, our main aim is to pave the ground for future approaches with the same strengths, and hopefully having better scalability.




\subsection{Inference-Time Defenses can Flip the Asymmetry: the Case of Counter-Attack}



The main idea in CA is to evaluate robustness on a case-by-case basis, flagging inputs as potentially unsafe if a robust answer cannot be provided. Specifically, given a norm-order $p$ and threshold $\varepsilon$, CA operates as follows:
\begin{compactitem}
    \item For a given input $\vx$, we determine if the model is $(\varepsilon, p)$-locally robust by running an untargeted adversarial attack on $x$;
    \item If the attack succeeds, we flag the input.
\end{compactitem}
In a practical usage scenario, flagged inputs would then be processed by a slower, but more robust, model (e.g. a human) or rejected; this behavior is similar to that of approaches for learning with rejection, but with a semantic tied to adversarial robustness\footnote{Note that the learning-with-rejection approach usually involves some form of confidence score; while the decision boundary distance might be seen as a sort of score, it does not have a probabilistic interpretation. Studying CA under this light represents a promising research direction.}.

Similarly, it is possible to draw comparisons between robust transductive learning (e.g. the work of \citet{chen2021towards}) and CA. While the two techniques use different approaches, we believe that parts of our analysis might be adapted to study existing applications of transductive learning to robust classification. Refer to \Cref{sec:transductiveLearning} for a more in-depth comparison.






Finally, note that the flagging rate depends on the model robustness: a model that is locally robust on the whole input distribution would have a flagging rate of 0, while in the opposite case all inputs would be flagged.
As a consequence, this form of inference-time defense is best thought of as a \emph{complement} to training-time robustness approaches, designed to catch those cases that are hard to handle due to \Cref{the:defenseSigma}.
A technique such as CA would indeed benefit from most advances in the field of adversarial robustness: training-time defenses for a better flagging rate, and attack algorithms for more effective and efficient certificates.

\subsection{Formal Properties}

The formal properties of the CA approach depend on the kind of attack used to perform the robustness check.
Specifically, when used with an exact attack,
such as those from \citet{carlini2017provably} and \citet{tjeng2019evaluating}, CA provides formal robustness guarantees for an arbitrary $p$ and $\varepsilon$:
\begin{theorem}
\label{the:caGuarantees}
Let $0 < p \leq \infty$ and let $\varepsilon > 0$. Let $f: X \to \{1, \dots, N\}$ be a classifier and let $a$ be an exact attack.
Let $f_{CA}^a: X \to \{1, \dots, N\} \cup \{ \star \}$ be defined as:
\begin{equation}
    f_{CA}^a (\vx) = \begin{cases}
        f(\vx) & \|a_{f,p}(\vx) - \vx\|_p > \varepsilon \\
        \star & \text{otherwise}
    \end{cases}
\end{equation}
Then $\forall \vx \in X$ an $L^p$ attack on $\vx$ with radius greater than or equal to $\varepsilon$ and with $\star \not \in C(\vx)$ fails.
\end{theorem}
The notation $f^a_{CA}(\vx)$ refers to the classifier $f$ combined with CA, relying on attack $a$. The condition $\star \not \in C(\vx)$ requires that the input generated by the attack should not be flagged by CA. Intuitively, CA guarantees robustness due to the fact that, if $\vx^\prime$ is an adversarial example for an input $\vx$, $\vx$ is also an adversarial example for $\vx^\prime$, which means that $\vx^\prime$ will be flagged.

Due to the properties of $L^p$ norms, CA also guarantees a degree of robustness against attacks with a different norm:

\begin{corollary}
\label{cor:caGuaranteesLp}
Let $1 \leq p \leq \infty$ and let $\varepsilon > 0$. Let $f$ be a classifier on inputs with $n$ elements that uses CA with norm $p$ and radius $\varepsilon$. Then for all inputs and for all $1 \leq r < p$, $L^r$ attacks of radius greater than or equal to $\varepsilon$ and with $\star \not \in C(\vx)$ will fail. Similarly, for all inputs and for all $r > p$, $L^r$ attacks of radius greater than or equal to $n^{\frac{1}{r} - \frac{1}{p}} \varepsilon$ and with $\star \not \in C(\vx)$ will fail (treating $\frac{1}{\infty}$ as 0).
\end{corollary}


Note that since the only expensive step in CA consists in applying an adversarial attack to an input, the complexity is the same as that of a regular attack. 


\paragraph{Attacking with a Higher Radius} 

In addition to robustness guarantees for a chosen $\varepsilon$, CA provides a form of computational robustness even beyond its intended radius. To prove this statement, we first formalize the task of attacking CA (referred to as Counter-CA, or CCA). 
%
%
This involves finding, given a starting point $\vx$, an input $\vx^\prime \in B_p(\vx, \varepsilon^\prime)$ that is adversarial but not flagged by CA, i.e. such that $f(\vx^\prime) \in C(\vx) \land \forall \vx^{\prime\prime} \in B_p(\vx^\prime, \varepsilon) . f(\vx^{\prime\prime}) = f(\vx^\prime)$.
Note that, \emph{for $\varepsilon^\prime \leq \varepsilon$, no solution exists}, since $\vx \in B_p(\vx^\prime, \varepsilon)$ and $f(\vx) \neq f(\vx^\prime)$.

\begin{theorem}[Attacking CA with a higher radius is $\Sigma_2^P$-complete]
\label{the:foolingCA}
Let $CCA_p$ be the set of all tuples $\langle \vx, \varepsilon, \varepsilon^\prime, C, f \rangle$ such that:
\begin{equation}
\label{eqn:definitionCCA}
\begin{gathered}
  \exists \vx^\prime \in B_p(\vx, \varepsilon^\prime) . \\ \left ( f(\vx^\prime) \in C(\vx) \land \forall \vx^{\prime\prime} \in B_p(\vx^\prime, \varepsilon) . f(\vx^{\prime\prime}) = f(\vx^\prime) \right )  
\end{gathered}
\end{equation}

where $\vx \in X$, $X$ is a FSFP space, $\varepsilon^\prime > \varepsilon$, $f(\vx) \not \in C(\vx)$ $f$ is a ReLU classifier and whether an output is in $C(\vx^*)$ for some $\vx^*$ can be decided in polynomial time.
Then $CCA_\infty$ is $\Sigma_2^P$-complete.
\end{theorem}
\begin{corollary}
\label{cor:ccaLp}
$CCA_p$ is $\Sigma_2^P$-complete for all $0 < p \leq \infty$.
\end{corollary}
\begin{corollary}
\label{cor:foolingPolynomial}
\Cref{the:foolingCA} also holds if, instead of ReLU classifiers, we consider the more general set of polynomial-time classifiers w.r.t. the size of the tuple.
\end{corollary}

In other words, under our assumptions, fooling CA can be harder than running it, thus flipping the computational asymmetry.
\Cref{cor:polynomialHierarchy} also implies that it is impossible to obtain a better gap between running the model and attacking it, from a Polynomial Hierarchy point of view (e.g. a $P$-time model that is $\Sigma_2^P$-hard to attack).
%
%
%
%
%
%
%
%
Note that, due to the worst-case semantic of \Cref{the:foolingCA}, fooling CA can be expected to be easy in practice when $\varepsilon^\prime \gg \varepsilon$: this is however a very extreme case, where the threshold might have been poorly chosen or the adversarial examples might be very different from genuine examples.


\paragraph{Partial Robustness}

While using exact attacks with CA is necessary for the best formal behavior, the approach remains capable of providing partial guarantees when used with either heuristic or lower-bounding approaches.


In particular, if a heuristic attack returns an example $\vx^\prime$ with $\|\vx - \vx^\prime\|_p \leq \varepsilon$, then $f$ is guaranteed to be locally non-robust on $\vx$. However, a heuristic attack failing to find an adversarial example does not guarantee that the model is locally robust.

Conversely, if we replace the attack with an optimization method capable of returning a lower bound $lb(\vx)$ on the decision boundary distance (e.g. a Mathematical Programming solver), we get the opposite result: if the method proves that $lb(\vx) > \varepsilon$, then $f$ is locally robust on $x$, but $f$ might be robust even if the method fails to prove it.

In other words, with heuristic attacks false positives are impossible, while with lower-bound methods false negatives are impossible.
Note that these two methods can be combined to improve scalability while retaining some formal guarantees.

These considerations provide further motivation for research in heuristic attacks, since every improvement in that field could lead to more reliable or faster robustness ``certificates''.
Additionally, they emphasize the potential of lower bounding techniques (e.g. guaranteed approximation algorithms) as efficient certification tools.
%
%
Finally, while we think that CA is an interesting technique per-se, we reiterate that the main appeal of the approach is to prove by means of an example that it is possible to circumvent the computational asymmetry we identified. We hope that future work will expand on this research direction, developing approaches that are both more efficient and with more formal guarantees.

\section{An Evaluation of Adversarial Attacks as Certification Tools}
\label{sec:experimentalEvaluation}

CA highlights an interesting aspect of adversarial attacks: since attacking a classifier and certifying its local robustness are complementary tasks, \textbf{adversarial attacks can be used to build inference-time certification techniques}.
This observation raises interest in evaluating existing (heuristic) attack algorithms in terms of their ability to serve as defenses (of which CA is just one of many possible applications). For example, in contexts where provable robustness is too resource-intensive, one could use sufficiently powerful heuristic attacks to determine with great accuracy if the model is locally robust (but without formal guarantees).

From this point of view, it should be noted that checking robustness \emph{only requires evaluating the decision boundary distance}, and not necessarily finding the adversarial example that is closest to an input $\vx$, i.e. the optimal solution of \Cref{eqn:attackAsOptimization}.
As a consequence, an attack does not need to perform well to be usable as a defense, but just to come \emph{predictably close} to the decision boundary.
For example, an algorithm that consistently overestimates the decision boundary distance by a 10\% factor would be as good as an exact attack for many practical purposes, since we could simply apply a correction to obtain an exact estimate.
This kind of evaluation is natural when viewing the issue from the perspective of our CA method, but to the best of our knowledge it has never been observed in the literature.

In this section, we thus empirically evaluate the quality of heuristic attacks.
Specifically, we test whether $\|\vx - \vx_h\|_p$, where $\vx_h$ is an adversarial example found by a heuristic attack, is predictably close to the true decision boundary distance $d_p^*(\vx)$. To the best of our knowledge, the only other work that performed a somewhat similar evaluation is \citet{carlini2017provably}, which evaluated the optimality of the Carlini \& Wagner attack on 90 MNIST samples for a $\sim$20k parameter network.

Consistently with \citet{athalye2018obfuscated} and \citet{weng2018towards}, we focus on the $L^\infty$ norm. 
Additionally, we focus on \emph{pools} of heuristic attacks. The underlying rationale is that different adversarial attacks should be able to cover for their reciprocal blind spots, providing a more reliable estimate.
%
Since this evaluation is empirical, it requires sampling from a chosen distribution, in our case specific classifiers and the MNIST \citep{mnist} and CIFAR10 \citep{cifar10} datasets. This means that the results are not guaranteed for other distributions, or for other defended models: studying how adversarial attacks fare in these cases is an important topic for future work.



\paragraph{Experimental Setup}
\label{sec:experimentalSetup}

We randomly selected $\sim$2.3k samples each from the test set of two datasets, MNIST and CIFAR10. We used three architectures per dataset (named A, B and C), each trained in three settings, namely standard training, PGD adversarial training \citep{madry2018towards} and PGD adversarial training with ReLU loss and pruning \citep{xiao2019training} (from now on referred to as ReLU training), for a total of nine configurations per dataset.

Since our analysis requires computing exact decision boundary distances, and size and depth both have a strong adverse impact on solver times, we used small and relatively shallow networks with parameters between $\sim$2k and $\sim$80k.
For this reason, the natural accuracy for standard training are significantly below the state of the art (89.63\% - 95.87\% on MNIST and 47.85\% - 55.81\% on CIFAR10). Adversarial training also had a negative effect on natural accuracies (84.54\% - 94.24\% on MNIST and 45.19\% - 51.35\% on CIFAR10), similarly to ReLU training (83.69\% - 93.57\% on MNIST and 32.27\% - 37.33\% on CIFAR10).
Note that using reachability analysis tools for NNs, such as \cite{gehr2018ai2}, capable of providing \emph{upper bounds} on the decision boundary in a reasonable time would not be sufficient for our goal: indeed both lower and upper bounds on the decision boundary distance could be arbitrarily far from $d^*(\vx)$, thus preventing us from drawing any firm conclusion.

We first ran a pool of heuristic attacks on each example, namely BIM \citep{kurakin2016adversarial}, Brendel \& Bethge \citep{brendel2019accurate}, Carlini \& Wagner \citep{carlini2017towards}, Deepfool \citep{moosavi2016deepfool}, Fast Gradient \citep{goodfellow2015explaining} and PGD \citep{madry2018towards}, in addition to simply adding uniform noise to the input.
%
Our main choice of attack parameters (from now on referred to as the ``strong'' parameter set) prioritizes finding adversarial examples at the expense of computational time.
For each example, we considered the nearest feasible adversarial example found by any attack in the pool.
We then ran the exact solver-based attack MIPVerify \citep{tjeng2019evaluating}, which is able to find the nearest adversarial example to a given input.
The entire process (including test runs) required $\sim$45k core-hours on an HPC cluster. Each node of the cluster has 384 GB of RAM and features two Intel CascadeLake 8260 CPUs, each with 24 cores and a clock frequency of 2.4GHz. 
We removed the examples for which MIPVerify crashed in at least one setting, obtaining 2241 examples for MNIST and 2269 for CIFAR10. We also excluded from our analysis all adversarial examples for which MIPVerify did not find optimal bounds (atol = 1e-5, rtol = 1e-10), which represent on average 11.95\% of the examples for MNIST and 16.30\% for CIFAR10.
Additionally, we ran the same heuristic attacks with a faster parameter set (from now on referred to as the ``balanced'' set) on a single machine with an AMD Ryzen 5 1600X six-core 3.6 GHz processor, 16 GBs of RAM and an NVIDIA GTX 1060 6 GB GPU. The process took approximately 8 hours.
Refer to \Cref{app:fullExperimentalSetup} for a more comprehensive overview of our experimental setup.



\paragraph{Distance Approximation}
\label{sec:results}

Across all settings, the mean distance found by the strong attack pool is 4.09\textpm2.02\%\ higher for MNIST and 2.21\textpm1.16\% higher for CIFAR10 than the one found by MIPVerify. For 79.81\textpm15.70\% of the MNIST instances and 98.40\textpm1.63\% of the CIFAR10 ones, the absolute difference is less than 1/255, which is the minimum distance in 8-bit image formats.
The balanced attack pool performs similarly, finding distances that are on average 4.65\textpm2.16\% higher for MNIST and 2.04\textpm1.13\% higher for CIFAR10. The difference is below 1/255 for 77.78\textpm16.08\% of MNIST examples and 98.74\textpm1.13\% of CIFAR10 examples. 
%
%
We compare the distances found by the strong attack pool for MNIST A and CIFAR10 (using standard training) with the true decision bound distances in \Cref{fig:scatter}. Refer to \Cref{app:extraResults} for the full data.

For all datasets, architectures and training techniques there appears to be a \textbf{strong, linear, correlation between the distance of the output of the heuristic attacks and the true decision boundary distance}.
We chose to measure this by training a linear regression model linking the two distances.
For the strong parameter set, we find that the average $R^2$ across all settings is 0.992\textpm0.004 for MNIST and 0.997\textpm0.003 for CIFAR10. The balanced parameter set performs similarly, achieving an $R^2$ of 0.990\textpm0.006 for MNIST and 0.998\textpm0.002 for CIFAR10.
From these results, we conjecture that increasing the computational budget of heuristic attacks does not necessarily improve predictability, although further tests would be needed to confirm such a claim.
Note that such a linear model can also be used to correct decision boudary distance overestimates in the context of heuristic CA. Another (possibly more reliable) procedure would consist in using quantile fitting; results for this approach are reported in \Cref{app:quantileCalibration}.

\definecolor{general}{HTML}{648FFF}
\definecolor{general2}{HTML}{FE6100}

\begin{figure}[ht]
\centering
\begin{subfigure}[]{0.4\textwidth}
\centering
    \begin{adjustbox}{width=\textwidth}
    \begin{tikzpicture}[transform shape]
        \begin{axis}[,
        xlabel={True Distance},
        ylabel={Estimated Distance},
        xlabel near ticks,
        ylabel near ticks,
        xtick pos=left,
        ytick pos=left,
        scatter/classes={%
        a={mark=o,draw=red}}
        ]
        \addplot [scatter src = explicit symbolic, only marks, mark size=0.5pt, draw=general2, opacity=0.2] table [x=True Distance, y=Estimated Distance, col sep = semicolon] {data/scatter/mnist-a-standard-strong.csv};
        \addplot[draw=black!90!white] coordinates {(0, 0) (0.08, 0.08)};
        \end{axis}
    ;
    \end{tikzpicture}
    \end{adjustbox}
\caption[Ex2]%
{{\small MNIST A}}    
\end{subfigure}
\hfill
\begin{subfigure}[]{0.4\textwidth}  
\centering 
    \begin{adjustbox}{width=\textwidth}
    \begin{tikzpicture}[transform shape]
        \begin{axis}[,
        xlabel={True Distance},
        ylabel={Estimated Distance},
        xlabel near ticks,
        ylabel near ticks,
        xtick pos=left,
        ytick pos=left,
        scatter/classes={%
        a={mark=o,draw=red}}
        ]
        \addplot [scatter src = explicit symbolic,only marks, mark size=0.5pt, draw=general2, opacity=0.2] table [x=True Distance, y=Estimated Distance, col sep = semicolon] {data/scatter/cifar10-a-standard-strong.csv};
        \addplot[draw=black!90!white] coordinates {(0, 0) (0.035, 0.035)};
        \end{axis}
    ;
    \end{tikzpicture}
    \end{adjustbox}
\caption[]%
{{\small CIFAR10 A}}  

\end{subfigure}
   \caption{Distances of the nearest adversarial example found by the strong attack pool compared to those found by MIPVerify on MNIST A and CIFAR10 A with standard training. The black line represents the theoretical optimum. Note that no samples are below the black line.} 
   
\label{fig:scatter}
\end{figure}
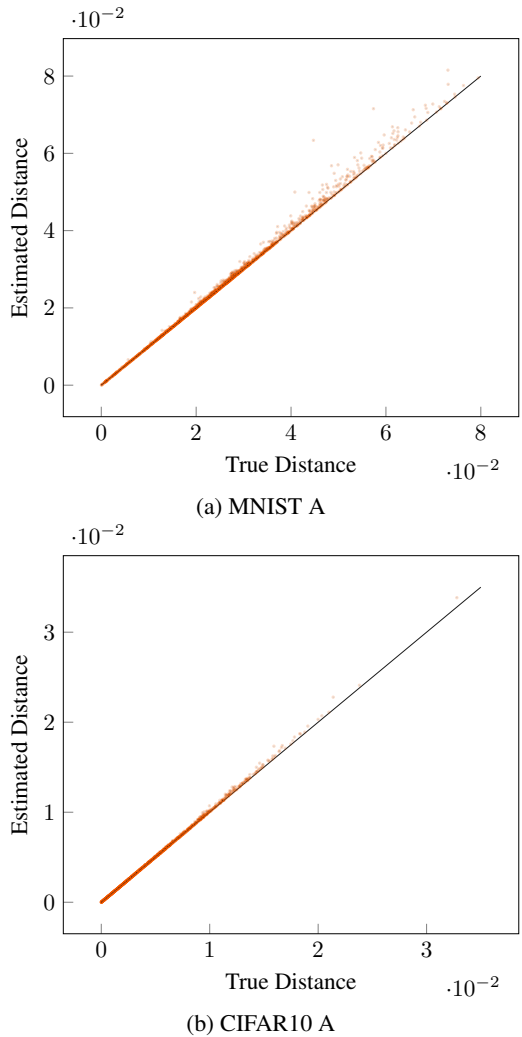

\paragraph{Attack Pool Ablation Study}
Due to the nontrivial computational requirements of running several attacks on the same input, we now study whether it is possible to drop some attacks from the pool without compromising its predictability.
Specifically, we consider all possible pools of size $n$ (with a success rate of 100\%) and pick the one with the highest average $R^2$ value over all architectures and training techniques.
As shown in \Cref{fig:ablation}, adding attacks \emph{does} increase predictability, although with diminishing returns. For example, the pool composed of the Basic Iterative Method, the Brendel \& Bethge Attack and the Carlini \& Wagner attack achieves on its own a $R^2$ value of 0.988\textpm0.004 for MNIST+strong, 0.986\textpm0.005 for MNIST+balanced, 0.935\textpm0.048 for CIFAR10+strong and 0.993\textpm0.003 for CIFAR10+balanced.
%
%
%
Moreover, dropping both the Fast Gradient Sign Method and uniform noise leads to negligible ($\ll 0.001$) absolute variations in the mean $R^2$.
%
%
These findings suggest that, as far as consistency is concerned, \textbf{the choice of attacks represents a more important factor than the number of attacks} in a pool. Refer to \Cref{app:ablation} for a more in-depth overview of how different attack selections affect consistency and accuracy.

\paragraph{Efficient Attacks}
We then explore if it is possible to increase the efficiency of attacks by optimizing for fast, rather than accurate, results. 
We pick three new parameter sets (namely Fast-100, Fast-1k and Fast-10k) designed to find the nearest adversarial examples within the respective number of calls to the model. 
We find that while Deepfool is not the strongest adversarial attack (see \Cref{app:extraResults}), it provides adequate results in very few model calls. For details on these results see \Cref{app:convergence}. 

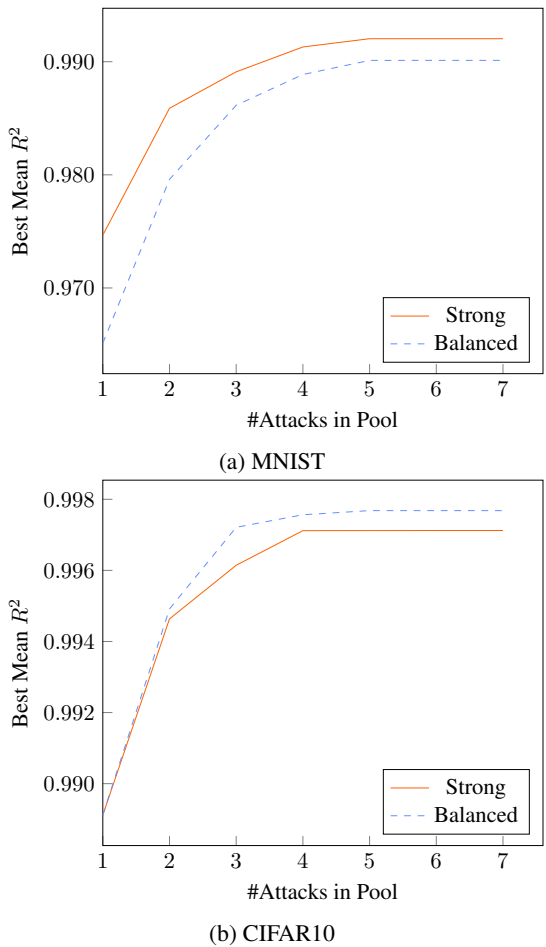
\begin{figure}
\centering
\begin{subfigure}[]{0.42\textwidth}
\centering
    \begin{adjustbox}{width=\textwidth}
    \begin{tikzpicture}[transform shape]
        \begin{axis}[
        xmin=1,
        xlabel={\#Attacks in Pool},
        ylabel={Best Mean $R^2$},
        xlabel near ticks,
        ylabel near ticks,
        xtick pos=left,
        ytick pos=left,
        y tick label style={
            /pgf/number format/.cd,
            fixed,
            fixed zerofill,
            precision=3,
            /tikz/.cd
        },
        legend pos=south east,
        ]
        \addplot [draw=general2] table [x=Pool Size, y=R2, col sep = semicolon] {data/ablation/mnist-strong.csv};
        \addplot [draw=general, dashed] table [x=Pool Size, y=R2, col sep = semicolon] {data/ablation/mnist-reduced.csv};
        \addlegendentry{Strong}
        \addlegendentry{Balanced}
        \end{axis}
    ;
    \end{tikzpicture}
    \end{adjustbox}
\caption[Ex2]%
{{\small MNIST}}    
\end{subfigure}
\hfill
\begin{subfigure}[]{0.42\textwidth}  
\centering 
    \begin{adjustbox}{width=\textwidth}
    \begin{tikzpicture}[transform shape]
        \begin{axis}[
        xmin=1,
        xlabel={\#Attacks in Pool},
        ylabel={Best Mean $R^2$},
        xlabel near ticks,
        ylabel near ticks,
        xtick pos=left,
        ytick pos=left,
        y tick label style={
            /pgf/number format/.cd,
            fixed,
            fixed zerofill,
            precision=3,
            /tikz/.cd
        },
        legend pos = south east
        ]
        \addplot [draw=general2] table [x=Pool Size, y=R2, col sep = semicolon] {data/ablation/cifar10-strong.csv};
        \addplot [draw=general, dashed] table [x=Pool Size, y=R2, col sep = semicolon] {data/ablation/cifar10-reduced.csv};
        \addlegendentry{Strong}
        \addlegendentry{Balanced}
        \end{axis}
    ;
    \end{tikzpicture}
    \end{adjustbox}
\caption[]%
{{\small CIFAR10}}  

\end{subfigure}
   \caption{Best mean $R^2$ value in relation to the number of attacks in the pool.}
   
\label{fig:ablation}
\end{figure}












\paragraph{UG100 Dataset}

We collect all the adversarial examples found by both MIPVerify and the heuristic attacks into a new dataset, which we name UG100. UG100 can be used to benchmark new adversarial attacks. Specifically, we can determine how strong an attack is by comparing it to both the theoretical optimum and heuristic attack pools. Another potential application involves studying factors that affect whether adversarial attacks perform sub-optimally.

\section{Conclusion}

In this work, we provided three contribution in the context of adversarial robustness.

First, we proved that attacking a ReLU classifier is $\mathit{NP}$-hard, while training a robust model of the same type is $\Sigma_2^P$-hard.
This result implies that defending is in the worst case harder than attacking; moreover, due to the broad applicability assumptions and the structure of its proof, it represents a reasonable explanation for the difficulty gap often encountered when building robust classifiers. The intuition behind our proofs can also help to pave the way for research into more tractable classes.

Second, we showed how inference-time techniques can sidestep the aforementioned computational asymmetry, by introducing a proof-of-concept defense called Counter Attack (CA). The central idea in CA is to check robustness by relying on adversarial attacks themselves: this strategy provides robustness guarantees, can invert the computational asymmetry, and may serve as the basis for devising more advanced inference-time defenses.

Finally, motivated by the last observation, we provided an empirical evaluation of heuristic attacks in terms of their ability to consistently approximate the decision boundary distance. We found that state-of-the-art heuristic attacks are indeed very reliable approximators of the decision boundary distance, suggesting that even heuristic attacks might be used in defensive contexts.

Our theoretical results highlight a structural challenge in adversarial ML, one that could be sidestepped through not only our CA approach, but potentially many more.
Additionally, we showed that adversarial attacks can also play a role in asymmetry-free robustness, thus opening up new research directions on their defensive applications.
We hope that our observations, combined with our formal analysis and our UG100 benchmark, can serve as the starting point for future research into these two important areas.

\section*{Acknowledgements}

\EUFlag \space The project leading to this application has received funding from the European Union’s Horizon Europe research and innovation programme under grant agreement No. 101070149. We also acknowledge the CINECA award under the ISCRA initiative, for the availability of high performance computing resources and support. Finally, we thank Andrea Borghesi, Andrea Iacco and Rebecca Montanari for their advice and support.

\bibliography{bibliography}
\bibliographystyle{icml2023}

\newpage
\appendix
\onecolumn

\mathchardef\mhyphen="2D

\appendix

\section{Proof Preliminaries}

\subsection{Notation}

We use $f_i$ to denote the $i$-th output of a network. We define $f$ as
\begin{equation}
    f(\vx) = \arg\max_{i} \{ f_i(\vx) \}
\end{equation}
for situations where multiple outputs are equal to the maximum, we use the class with the lowest index.

\subsection{\texorpdfstring{$\mu$}{mu} Arithmetic}

Given two FSFP spaces $X$ and $X^\prime$ with distance minorants $\mu$ and $\mu^\prime$, we can compute new positive minorants after applying functions to the spaces as follows:
\begin{itemize}
    \item Sum of two vectors: $\mu_{X+X^\prime} = min(\mu, \mu^\prime)$;
    \item Multiplication by a constant: $\mu_{\alpha X} = |\alpha| \mu$;
    \item ReLU: $\mu_{ReLU(X)} = \mu$.
\end{itemize}

Since it is possible to compute the distance minorant of a space transformed by any of these functions in polynomial time, it is also possible to compute the distance minorant of a space transformed by any composition of such functions in polynomial time.


\subsection{Functions}

We now provide an overview of several functions that can be obtained by using linear combinations and ReLUs.

\paragraph{$\max$}

\citet{carlini2017provably} showed that we can implement the $\max$ function using linear combinations and ReLUs as follows:

\begin{equation}
    \max(x, y) = ReLU(x - y) + y
\end{equation}

We can also obtain an $n$-ary version of $\max$ by chaining multiple instances together.


\paragraph{$step$} If $X$ is a FSFP space, then the following scalar function:
\begin{equation}
    step_0(x) = \frac{1}{\mu} \left (ReLU(x) - ReLU(x - \mu) \right)
\end{equation}
is such that $\forall i . \forall \vx \in X$,  $step_0(x_i)$ is 0 for $x_i \leq 0$ and 1 for $x_i > 0$. 

Similarly, let $step_1$ be defined as follows:
\begin{equation}
    step_1(x) = \frac{1}{\mu} \left ( ReLU(x + \mu) - ReLU(x) \right )
\end{equation}

Note that $\forall i . \forall \vx \in X$, $step_1(x_i) = 0$ for $x_i < 0$ and $step_1(x_i) = 1$ for $x_i \geq 0$.

\paragraph{Boolean Functions}

We then define the Boolean functions $not: \{0, 1\} \to \{0, 1\}$, $and: \{0, 1\}^2 \to \{0, 1\}$, $or: \{0, 1\}^2 \to \{0, 1\}$ and $if: \{0, 1\}^3 \to \{0, 1\}$ as follows:
\begin{align}
    not(x) &= 1 - x \\
    and(x, y) &= step_1(x + y - 2) \\
    or(x, y) &= step_1(x + y) \\
    if(a, b, c) &= or(and(not(a), b), and(a, c))
\end{align}

where $if(a, b, c)$ returns $b$ if $a = 0$ and $c$ otherwise.

Note that we can obtain $n$-ary variants of $and$ and $or$ by chaining multiple instances together.

\paragraph{$cnf_3$} Given a set $\vz = \{\{z_{1,1}, \dots, z_{1,3}\}, \dots, \{z_{n, 1}, z_{n,3}\}\}$ of Boolean atoms (i.e. $z_{i,j}(\vx) = x_k$ or $\lnot x_k$ for a certain $k$) defined on an $n$-long Boolean vector $\vx$, $cnf_3(\vz)$ returns the following Boolean function:
\begin{equation}
    cnf^\prime_3(\vx) = \bigwedge_{i = 1, \dots, n}\bigvee_{j = 1, \dots, 3}z_{i, j}(\vx)
\end{equation}

We refer to $\vz$ as a 3CNF formula.

Since $cnf^\prime_3$ only uses negation, conjunction and disjunction, it can be implemented using respectively $neg$, $and$ and $or$.
Note that, given $\vz$, we can build $cnf^\prime_3$ in polynomial time w.r.t. the size of $\vz$.

\paragraph{Comparison Functions}

We can use $step_0$, $step_1$ and $neg$ to obtain comparison functions as follows:
\begin{align}
    geq(x, k) &= step_1(x - k) \\
    gt(x, k) &= step_0(x, k) \\
    leq(x, k) &= not(gt(x, k)) \\
    lt(x, k) &= not(geq(x, k)) \\
    eq(x, k) &= and(geq(x, k), leq(x, k))
\end{align}


\noindent
Moreover, we define $open : \mathbb{R}^3 \to \{0, 1\}$ as follows:
\begin{equation}
    open(x, a, b) = and(gt(x, a), lt(x, b))
\end{equation}




\section{Proof of \Cref{the:attackNp}}

\subsection{\texorpdfstring{$U\mhyphen{}ATT_\infty \in NP$}{U-ATT Infinity in NP}}

To prove that $U\mhyphen{}ATT_\infty \in NP$, we show that there exists a polynomial certificate for $U\mhyphen{}ATT$ that can be checked in polynomial time. The certificate is the value of $\vx^\prime$, which will have a representation of the same size as $\vx$ (due to the FSFP space assumption) and can be checked by verifying:
\begin{itemize}
    \item $\|\vx - \vx'\|_\infty \leq \varepsilon$, which can be checked in linear time;
    \item $f_\vtheta(\vx') \neq f(\vx)$, which can be checked in polynomial time.
\end{itemize}

\subsection{\texorpdfstring{$U\mhyphen{}ATT_\infty$ is $NP$-Hard}{U-ATT Infinity is NP-Hard}}
\label{sec:attNpHard}

We will prove that $U\mhyphen{}ATT_\infty$ is $NP$-Hard by showing that $3SAT \leq U\mhyphen{}ATT_\infty$.

Given a set of 3CNF clauses $\vz = \{\{z_{11}, z_{12}, z_{13}\}, \dots, \{z_{m1}, z_{m2}, z_{m3} \}\}$ defined on $n$ Boolean variables $x_1, \dots, x_n$, we construct the following query $q(\vz)$ for $U\mhyphen{}ATT_\infty$: 
\begin{equation}
    q(\vz) = \langle \vx^{(s)}, \frac{1}{2}, f \rangle
\end{equation}

where $\vx^{(s)} = \left ( \frac{1}{2}, \dots, \frac{1}{2} \right )$ is a vector with $n$ elements.
Verifying $q(\vz) \in U\mhyphen{}ATT_\infty$ is equivalent to checking:
\begin{equation}
\label{eqn:untAttCondition}
    \exists \vx^\prime \in B_\infty \left (x_{s}, \frac{1}{2} \right ) . f(\vx^\prime) \neq f(\vx^{(s)})
\end{equation}

Note that $\vx \in B_\infty \left (\vx^{(s)}, \frac{1}{2} \right )$ is equivalent to $\vx \in [0, 1]^n$.

\paragraph{Truth Values} We will encode the truth values of $\hat{\vx}$ as follows:
\begin{align}
    x^\prime_i \in \left [0, \frac{1}{2} \right ] \iff \hat{x}_i = 0 \\
    x^\prime_i \in \left (\frac{1}{2}, 1 \right ] \iff \hat{x}_i = 1
\end{align}

We can obtain the truth value of a scalar variable by using $isT(x_i) = gt \left (x_i, \frac{1}{2} \right )$. Let $bin(\vx) = or (isT(x_1), \dots, isT(x_n))$.


\paragraph{Definition of $f$} We define $f$ as follows:
\begin{align}
    f_1(\vx) &= and(not(isx^{(s)}(\vx)), cnf^\prime_3(bin(\vx))) \\
    f_0(\vx) &= not(f_1(\vx))
\end{align}

where $cnf^\prime_3 = cnf_3(\vz)$ and $isx^{(s)}$ is defined as follows:
\begin{equation}
    isx^{(s)}(\vx) = and \left (eq \left (x_1, \frac{1}{2} \right ), \dots, eq \left (x_n, \frac{1}{2} \right ) \right )
\end{equation}

Note that $f$ is designed such that $f(\vx^{(s)}) = 0$, while for $\vx^\prime \neq \vx^{(s)}$, $f(\vx^\prime) = 1$ iff the formula $\vz$ is true for the variable assignment $bin(\vx^\prime)$.


\begin{lemma}
$\vz \in 3SAT \implies q(\vz) \in U\mhyphen{}ATT_\infty$
\end{lemma}
\begin{proof}
Let $\vz \in 3SAT$. Therefore $\exists \vx^* \in \{0, 1\}^n$ such that $cnf_3(\vz)(\vx^*) = 1$. Since $bin(\vx^*) = \vx^*$ and $\vx^* \neq \vx^{(s)}$, $f(\vx^*) = 1$, which means that it is a valid solution for \Cref{eqn:untAttCondition}. From this we can conclude that $q(\vz) \in U\mhyphen{}ATT_\infty$.
\end{proof}

\begin{lemma}
$q(\vz) \in U\mhyphen{}ATT_\infty \implies \vz \in 3SAT$
\end{lemma}

\begin{proof}
Since $q(\vz) \in U\mhyphen{}ATT_\infty$, $\exists \vx^* \in [0, 1]^n \setminus \{\vx^{(s)}\}$ that is a solution to \Cref{eqn:untAttCondition} (i.e. $f(\vx^*) = 1$). Then $cnf^\prime_3(bin(\vx^*)) = 1$, which means that there exists a $\hat{\vx}$ (i.e. $bin(\vx^*)$) such that $cnf^\prime_3(\hat{\vx}) = 1$. From this we can conclude that $\vz \in 3SAT$.
\end{proof}

\noindent
Since:
\begin{itemize}
    \item $q(\vz)$ can be computed in polynomial time;
    \item $\vz \in 3SAT \implies q(\vz) \in U\mhyphen{}ATT_\infty$;
    \item $q(\vz) \in U\mhyphen{}ATT \implies \vz \in 3SAT$.
\end{itemize}
we can conclude that $3SAT \leq U\mhyphen{}ATT_\infty$.

\subsection{Proof of \Cref{cor:lpAttack}}

\subsubsection{\texorpdfstring{$U\mhyphen{}ATT_p \in NP$}{U-ATT p in NP}}
The proof is identical to the one for $U\mhyphen{}ATT_\infty$.

\subsubsection{\texorpdfstring{$U\mhyphen{}ATT_p$ is $NP$-Hard}{U-ATT p is NP-Hard}}





The proof that $q(\vz) \in U\mhyphen{}ATT_p \implies \vz \in 3SAT$ is very similar to the one for $U\mhyphen{}ATT_\infty$. Since $q(\vz) \in U\mhyphen{}ATT_p$, we know that $\exists \vx^* \in B_p(\vx^{(s)}, \varepsilon) \setminus \{\vx^{(s)}\} . f(\vx^*) = 1$, which means that there exists a $\hat{\vx}$ (i.e. $bin(\vx^*)$) such that $cnf^\prime_3(\hat{\vx}) = 1$. From this we can conclude that $\vz \in 3SAT$.

The proof that $\vz \in 3SAT \implies q(\vz) \in U\mhyphen{}ATT_p$ is slightly different, due to the fact that since $\vx^* \not \in B_p(\vx^{(s)}, \frac{1}{2})$ we need to use a different input to prove that $\exists \vx^\prime \in B_p(\vx^{(s)}) . f(\vx^\prime) = 1$.

Let $0 < p < \infty$. Given a positive integer $n$ and a real $0 < p < \infty$, let $\rho_{p,n}(r)$ be a positive minorant of the $L^\infty$ norm of a vector on the $L^p$ sphere of radius $r$. For example, for $n = 2$, $p = 2$ and $r = 1$, any positive value less than or equal to $\frac{\sqrt{2}}{2}$ is suitable. Note that, for $0 < p < \infty$ and $n, r > 0$, $\rho_{p,n}(r) < r$. 

Let $\vz \in 3SAT$. Therefore $\exists \vx^* \in \{0, 1\}^n$ such that $cnf_3(\vz)(\vx^*) = 1$. Let $\vx^{**}$ be defined as:
\begin{equation}
    x^{**}_i = \begin{cases}\frac{1}{2} - \rho_{p, n}\left(\frac{1}{2}\right) & x^*_i = 0 \\
    \frac{1}{2} + \rho_{p, n}\left(\frac{1}{2}\right) & x^*_i = 1 \end{cases}
\end{equation}

By construction, $\vx^{**} \in B_p \left (\vx^{(s)}, \rho_{p, n}\left(\frac{1}{2}\right) \right )$. Additionally, $bin(\vx^{**}) = \vx^*$, and since we know that $\vz$ is true for the variable assignment $\vx^*$, we can conclude that $f(\vx^{**}) = 1$, which means that $\vx^{**}$ is a valid solution for \Cref{eqn:untAttCondition}. From this we can conclude that $q(\vz) \in U\mhyphen{}ATT_p$.

\subsection{Proof of \Cref{cor:attackTargeted}}

The proof is identical to the proof of \Cref{the:attackNp} (for $p = \infty$) and \Cref{cor:lpAttack} (for $0 < p < \infty$), with the exception of requiring $f(\vx^\prime) = 1$.

\subsection{Proof of \Cref{cor:attackPolynomial}}

The proof that attacking a polynomial-time classifier is in $NP$ is the same as that for \Cref{the:attackNp}.

Attacking a polynomial-time classifier is $NP$-hard due to the fact that the ReLU networks defined in the proof of \Cref{the:attackNp} are polynomial-time classifiers. Since attacking a general polynomial-time classifier is a generalization of attacking a ReLU polynomial-time classifier, the problem is $NP$-hard.

\section{Proof of \Cref{the:attackRelative}}

Proving that $U\mhyphen{}ATT_p(f) \in NP^A$ means proving that it can be solved in polynomial time by a non-deterministic Turing machine with an oracle that can solve a problem in $A$. Since $Z_f \in A$, we can do so by picking a non-deterministic Turing machine with access to an oracle that solves $Z_f$.
We then generate non-deterministically the adversarial example and return the output of the oracle. Due to the FSFP assumption, we know that the size of this input is the same as the size of the starting point, which means that it can be generated non-deterministically in polynomial time. Therefore, $U\mhyphen{}ATT_p(f) \in NP^A$. 

\subsection{Proof of \Cref{cor:polynomialHierarchy}}

Follows directly from \Cref{the:attackRelative} and the definition of $\Sigma_n^P$.

\section{Proof of \Cref{the:defenseSigma}}

\subsection{Preliminaries}

$\Pi_2^P 3SAT$ is the set of all $\vz$ such that:
\begin{equation}
    \forall \hat{\vx} \exists \hat{\vy} . R(\hat{\vx}, \hat{\vy})
\end{equation}

where $R(\hat{\vx}, \hat{\vy}) = cnf_3(\vz)(\hat{x}_1, \dots, \hat{x}_n, \hat{y}_1, \dots, \hat{y}_n)$.

\citet{stockmeyer1976polynomial} showed that $\Pi_2 3SAT$ (also known as $\forall\exists 3 SAT$) is $\Pi_2^P$-complete. Therefore, $co\Pi_2{}3SAT$, which is defined as the set of all $\vz$ such that:
\begin{equation}
\label{eqn:coForallExists3SatCondition}
    \exists \hat{\vx} \forall \hat{\vy} \lnot R(\hat{\vx}, \hat{\vy})
\end{equation}
is $\Sigma_2^P$-complete.

\subsection{\texorpdfstring{$PL\mhyphen{}ROB_\infty \in \Sigma_2^P$}{PL-ROB Infinity in Sigma2P}}
\label{sec:plRobInSigma}

$PL\mhyphen{}ROB_\infty \in \Sigma_2^P$ if there exists a problem $A \in P$ and a polynomial $q$ such that $\forall \Gamma = \langle \vx, \varepsilon, f_\vtheta, v_f \rangle$:
\begin{equation}
    \Gamma \in PL\mhyphen{}ROB \iff \exists \vy . |\vy| \leq q(|\Gamma|) \land (\forall \vz . (|\vz| \leq q(|\Gamma|) \implies \langle \Gamma, \vy, \vz \rangle \in A ))
\end{equation}
This can be proven by setting $\vy = \vtheta^\prime$, $\vz = \vx^\prime$ and $A$ as the set of triplets $\langle \Gamma, \vtheta^\prime, \vx^\prime\rangle$ such that all of the following are true:
\begin{itemize}
    \item $v_f(\vtheta^\prime) = 1$;
    \item $\|\vx - \vx^\prime\|_\infty \leq \varepsilon$;
    \item $f_\vtheta(\vx) = f_\vtheta(\vx^\prime)$.
\end{itemize}
Since all properties can be checked in polynomial time, $A \in P$ and thus $PL\mhyphen{}ROB_\infty \in \Sigma_2^P$.

\subsection{\texorpdfstring{$PL\mhyphen{}ROB_\infty$ is $\Sigma_2^P$-Hard}{PL-ROB Infinity is Sigma2P-Hard}}

We will prove that $PL\mhyphen{}ROB_\infty$ is $\Sigma_2^P$-hard by showing that $co\Pi_2{}3SAT \leq  PL\mhyphen{}ROB_\infty$.

Let $n_{\hat{\vx}}$ be the length of $\hat{\vx}$ and let $n_{\hat{\vy}}$ be the length of $\hat{\vy}$.

Given a set $\vz$ of 3CNF clauses, we construct the following query $q(\vz)$ for $PL\mhyphen{}ROB$:
\begin{equation}
    q(\vz) = \langle \vx^{(s)}, \frac{1}{2}, f_\vtheta, v_f \rangle
\end{equation}
where $\vx^{(s)} = \left (\frac{1}{2}, \dots, \frac{1}{2} \right )$ is a vector with $n_{\hat{\vy}}$ elements and $v_f(\vtheta) = 1 \iff \vtheta \in \{0, 1\}^{n_{\hat{\vx}}}$. Note that $\vtheta^\prime \in \{0, 1\}^{n_{\hat{\vx}}}$ can be checked in polynomial time w.r.t. the size of the input.

\paragraph{Truth Values} We will encode the truth values of $\hat{\vx}$ as a set of binary parameters $\vtheta^\prime$, while we will encode the truth values of $\hat{\vy}$ using $\vx^\prime$ through the same technique mentioned in \Cref{sec:attNpHard}.

\paragraph{Definition of $f_\vtheta$}
We define $f_\vtheta$ as follows:
\begin{itemize}
    \item $f_{\vtheta, 1}(\vx) = and(not(isx^{(s)}(\vx)), cnf_3^{\prime\prime}(\vtheta, \vx))$, where $cnf_3^{\prime\prime}$ is defined over $\vtheta$ and $bin(\vx)$ using the same technique mentioned in \Cref{sec:attNpHard} and $isx^{(s)}(\vx) = and_{i=1,\dots,n}eq(x_i, \frac{1}{2})$;
    \item $f_{\vtheta, 0}(\vx) = not(f_{\vtheta,1}(\vx))$.
\end{itemize}

Note that $f_\vtheta(\vx^{(s)}) = 0$ for all choices of $\vtheta$. Additionally, $f_\vtheta$ is designed such that:
\begin{equation}
    \forall \vx^\prime \in B_\infty \left (\vx^{(s)}, \frac{1}{2} \right ) \setminus \{\vx^{(s)}\} . \forall \vtheta^\prime . \left ( v_f(\vtheta^\prime) = 1 \implies (f_{\vtheta^\prime}(\vx^\prime) = 1 \iff R(\vtheta^\prime, bin(\vx^\prime))) \right )
\end{equation}

\begin{lemma}
$\vz \in co\Pi_2{}3SAT \implies q(\vz) \in PL\mhyphen{}ROB_\infty$
\end{lemma}
\begin{proof}
Since $\vz \in co\Pi_2{}3SAT$, there exists a Boolean vector $\vx^*$ such that $\forall \hat{\vy}. \lnot R(\vx^*, \hat{\vy})$.

Then both of the following statements are true:
\begin{itemize}
    \item $v_f(\vx^*) = 1$, since $\vx^* \in \{0, 1\}^{n_{\hat{\vx}}}$;
    \item $\forall \vx^\prime \in B_\infty(\vx^{(s)}, \varepsilon) . f_{\vx^*}(\vx^\prime) = 0$, since $f_{\vx^*}(\vx^\prime) = 1 \iff R(\vx^*, bin(\vx^\prime))$;
\end{itemize} 
Therefore, $\vx^*$ is a valid solution for \Cref{eqn:plRobCondition} and thus $q(\vz) \in PL\mhyphen{}ROB_\infty$.
\end{proof}

\begin{lemma}
$q(\vz) \in PL\mhyphen{}ROB_\infty \implies \vz \in co\Pi_2{}3SAT$
\end{lemma}
\begin{proof}
Since $q(\vz) \in PL\mhyphen{}ROB_\infty$, there exists a $\vtheta^*$ such that:
\begin{equation}
    v_f(\vtheta) = 1 \land \forall \vx^\prime \in B_\infty(\vx^{(s)}, \varepsilon) . f_{\vtheta^*}(\vx^\prime) = f_{\vtheta^*}(\vx^{(s)})
\end{equation}

Note that $\vtheta^* \in \{0, 1\}^{n_{\hat{\vx}}}$, since $v_f(\vtheta^*) = 1$.
Moreover, $\forall \hat{\vy} . \lnot R(\vtheta^*, \hat{\vy})$, since $bin(\hat{\vy}) = \hat{\vy}$ and $f_{\vtheta^*}(\hat{\vy}) = 1 \iff R(\vtheta^*, \hat{\vy})$.

Therefore, $\vtheta^*$ is a valid solution for \Cref{eqn:coForallExists3SatCondition}, which implies that $\vz \in co\Pi_2{}3SAT$.

\end{proof}

\noindent
Since:
\begin{itemize}
    \item $q(\vz)$ can be computed in polynomial time;
    \item $\vz \in co\Pi_2{}3SAT \implies q(\vz) \in PL\mhyphen{}ROB_\infty$;
    \item $q(\vz) \in PL\mhyphen{}ROB_\infty \implies \vz \in co\Pi_2{}3SAT$.
\end{itemize}

we can conclude that $co\Pi_2{}3SAT \leq PL\mhyphen{}ROB_\infty$.

\subsection{Proof of \Cref{cor:lpDefense}}

\subsubsection{\texorpdfstring{$PL\mhyphen{}ROB_p \in \Sigma_2^P$}{PL-ROB p in Sigma2P}}
The proof is identical to the one for $PL\mhyphen{}ROB_\infty$.

\subsubsection{\texorpdfstring{$PL\mhyphen{}ROB_p$ is $\Sigma_2^P$-Hard}{PL-ROB p is Sigma2P-Hard}}

We follow the same approach used in the proof for \Cref{cor:lpAttack}.

\paragraph{Proof of $q(\vz) \in PL\mhyphen{}ROB_p \implies \vz \in co\Pi_2{}3SAT$} If $q(\vz) \in PL\mhyphen{}ROB_p$, it means that $\exists \vtheta^* . \left (v_f(\vtheta^*) = 1 \implies \forall \vx^\prime \in B_p \left (\vx^{(s)}, \frac{1}{2} \right ) . f(\vx^\prime) = 0 \right )$. Then $\forall \hat{\vy}$, there exists a corresponding input $\vy^{**} \in B_p \left (\vx^{(s)}, \frac{1}{2} \right )$ defined as follows:
\begin{equation}
    y^{**}_i = \begin{cases}\frac{1}{2} - \rho_{p, n}\left(\frac{1}{2}\right) & \hat{y}_i = 0 \\
    \frac{1}{2} + \rho_{p, n}\left(\frac{1}{2}\right) & \hat{y}_i = 1 \end{cases}
\end{equation}

such that $e^{(y)}(y^{**}) = \hat{y}$.
Since $\vy^{**} \in B_p \left (\vx^{(s)}, \frac{1}{2} \right )$, $cnf_3^{\prime\prime}(\vtheta^*, bin(\vy^{**})) = 0$, which means that $R(\vtheta^*, \hat{\vy})$ is false. In other words, $\exists \vtheta^* . \forall \hat{\vy} . \lnot R(\vtheta^*, \hat{\vy})$, i.e. $\vz \in co\Pi_2{}3SAT$.

\paragraph{Proof of $\vz \in co\Pi_2{}3SAT \implies q(\vz) \in PL\mhyphen{}ROB_p$} The proof is very similar to the corresponding one for \Cref{the:defenseSigma}.

If $\vz \in co\Pi_2{}3SAT$, then $\exists \hat{\vx}^* . \forall \hat{\vy} . \lnot R(\hat{\vx}, \hat{\vy})$. Set $\vtheta^* = \hat{\vx}^*$. We know that $f_\vtheta^*(\vx^{(s)}) = 0$. We also know that $\forall \vx^\prime \in B_p \left ( \vx^{(s)}, \frac{1}{2} \right ) \setminus \{\vx^{(s)}\} . \left ( f_{\vtheta^*}(\vx) = 1 \iff cnf_3^{\prime\prime}(\vtheta^*, \vx^\prime) = 1 \right )$. In other words, $\forall \vx^\prime \in B_p \left ( \vx^{(s)}, \frac{1}{2} \right ) \setminus \{\vx^{(s)}\} . \left ( f_{\vtheta^*}(\vx^\prime) = 1 \iff R(\vtheta^*, bin(\vx^\prime)) \right )$. Since $R(\vtheta^*, \hat{\vy})$ is false for all choices of $\hat{\vy}$, $\forall \vx^\prime \in B_p \left ( \vx^{(s)}, \frac{1}{2} \right ) \setminus \{\vx^{(s)}\} . f_{\vtheta^*} (\vx^\prime) = 0$. Given the fact that $f_{\vtheta^*}(\vx^{(s)}) = 0$, we can conclude that $\vtheta^*$ satisfies \Cref{eqn:plRobCondition}.

\subsection{Proof of \Cref{cor:defensePolynomial}}

Similarly to the proof of \Cref{cor:attackPolynomial}, it follows from the fact that ReLU classifiers are polynomial-time classifiers (w.r.t. the size of the tuple).

\section{Proof of \Cref{the:caGuarantees}}

There are two cases:
\begin{itemize}
    \item $\forall \vx^\prime \in B_p(\vx, \varepsilon) . f(\vx^\prime) = f(\vx)$: then the attack fails because $f(\vx) \not \in C(\vx)$;
    \item $\exists \vx^\prime \in B_p(\vx, \varepsilon) . f(\vx^\prime) \neq f(\vx)$: then due to the symmetry of the $L^p$ norm $\vx \in B_p(\vx^\prime, \varepsilon)$. Since $f(\vx) \neq f(\vx^\prime)$, $\vx$ is a valid adversarial example for $\vx^\prime$, which means that $f(\vx^\prime) = \star$. Since $\star \not \in C(\vx)$, the attack fails. 
\end{itemize}

\subsection{Proof of \Cref{cor:caGuaranteesLp}}

Assume that $\forall \vx .||\vx||_r \geq \eta ||\vx||_p$ and fix $\vx^{(s)} \in X$. Let $\vx^\prime \in B_r(\vx^{(s)}, \eta\varepsilon)$ be an adversarial example. Then $||\vx^\prime - \vx^{(s)}||_r \leq \eta\varepsilon$, and thus $\eta ||\vx^\prime - \vx^{(s)}||_p \leq \eta\varepsilon$. Dividing by $\eta$, we get $||\vx^\prime - \vx^{(s)}||_p \leq \varepsilon$, which means that $\vx^{(s)}$ is a valid adversarial example for $\vx^\prime$ and thus $\vx^\prime$ is rejected by $p$-CA.

We now proceed to find the values of $\eta$.

\subsubsection{\texorpdfstring{$1 \leq r < p$}{1 <= r < p}}

We will prove that $||\vx||_r \geq ||\vx||_p$.

\paragraph{Case $p < \infty$}
Consider $\ve = \frac{\vx}{||\vx||_p}$. $\ve$ is such that $||\ve||_p = 1$ and for all $i$ we have $|\ve_i| \leq 1$. Since $r < p$, for all $0 \leq t \leq 1$ we have $|t|^p \leq |t|^r$. Therefore:
\begin{equation}
    ||\ve||_r = \left (\sum_{i = 1}^n |\ve_i|^r \right)^{1/r} \geq \left (\sum_{i = 1}^n |\ve_i|^p \right)^{1/r} = ||\ve||^{p/r}_p = 1
\end{equation}
Then, since $||\ve||_r \geq 1$:
\begin{equation}
    ||\vx||_r = || \; ||\vx||_p \ve ||_r = ||\vx||_p ||\ve||_r \geq ||\vx||_p
\end{equation}

\paragraph{Case $p = \infty$} Since $||\vx||_r \geq ||\vx||_p$ for all $r < p$ and since the expressions on both sides of the inequality are compositions of continuous functions, as $p \to \infty$ we get $||\vx||_r \geq ||\vx||_\infty$.

\subsubsection{\texorpdfstring{$r > p$}{r > p}}

We will prove that $||\vx||_r \geq n^{\frac{1}{r} - \frac{1}{p}} ||\vx||_p$.

\paragraph{Case $r < \infty$}
Hölder's inequality states that, given $\alpha, \beta \geq 1$ such that $\frac{1}{\alpha} + \frac{1}{\beta} = 1$ and given $f$ and $g$, we have:
\begin{equation}
    ||fg||_1 \leq ||f||_\alpha ||g||_\beta
\end{equation}

Setting $\alpha = \frac{r}{r - p}$, $\beta = \frac{r}{p}$, $f = (1, \dots, 1)$ and $g = \left (x_1^p, \dots, x_n^p \right )$, we know that:
\begin{itemize}
    \item $||fg||_1 = \sum_{i=1}^n (1 \cdot x_i^p) = ||\vx||_p^p$;
    \item $||f||_\alpha = \left (\sum_{i=1}^n 1 \right )^{1/\alpha} = n^{1/\alpha}$;
    \item $||g||_\beta = \left (\sum_{i = 1} x_i^{pr/p} \right )^{p/r} = \left (\sum_{i = 1} x_i^r \right )^{p/r} = ||\vx||_r^p$.
\end{itemize}
Therefore $||\vx||_p^p \leq n^{1/\alpha} ||\vx||_r^p$.
Raising both sides to the power of $1/p$, we get $||\vx||_p \leq n^{1/(p\alpha)} ||\vx||_r$. Therefore:
\begin{equation}
    ||\vx||_p \leq n^{(r - p)/(pr)} ||\vx||_r = n^{\frac{1}{p} - \frac{1}{r}} ||\vx||_r
\end{equation}
Dividing by $n^{\frac{1}{p} - \frac{1}{r}}$ we get:
\begin{equation}
    n^{\frac{1}{r} - \frac{1}{p}} ||\vx||_p \leq ||\vx||_r
\end{equation}

\paragraph{Case $r = \infty$}
Since the expressions on both sides of the inequality are compositions of continuous functions, as $r \to \infty$ we get $||\vx||_\infty \geq n^{-\frac{1}{p}}||x||_p$.

\section{Proof of \Cref{the:foolingCA}}




\subsection{\texorpdfstring{$CCA_\infty \in \Sigma_2^P$}{CCA Infinity in Sigma2P}}

$CCA_\infty \in \Sigma_2^P$ iff there exists a problem $A \in P$ and a polynomial $p$ such that $\forall \Gamma = \langle \vx, \varepsilon, \varepsilon^\prime, C, f\rangle$:
\begin{equation}
    \Gamma \in CCA_\infty \iff \exists \vy . |\vy| \leq p \left (|\Gamma| \right ) \land \left (\forall \vz . (|\vz| \leq p(|\Gamma|) \implies  \langle \Gamma, \vy, \vz \rangle \in A ) \right ) 
\end{equation}

This can be proven by setting $\vy = \vx^\prime$,$\vz = \vx^{\prime\prime}$ and $A$ as the set of all triplets $\langle \Gamma, \vx^\prime, \vx^{\prime\prime} \rangle$ such that all of the following are true:
\begin{itemize}
    \item $\|\vx - \vx^\prime\|_\infty \leq \varepsilon^\prime$
    \item $f(\vx^\prime) \in C(\vx)$
    \item $\|\vx^{\prime\prime} - \vx^\prime\|_\infty \leq \varepsilon$
    \item $f(\vx^{\prime\prime}) = f(\vx^\prime)$
\end{itemize}

Since all properties can be checked in polynomial time, $A \in P$. 

\subsection{\texorpdfstring{$CCA_\infty$ is $\Sigma_2^P$-Hard}{CCA Infinity is Sigma2P-Hard}}

We will show that $CCA_\infty$ is $\Sigma_2^P$-hard by proving that $co\Pi_2{}3SAT \leq CCA_\infty$.

First, suppose that the length of $\hat{\vx}$ and $\hat{\vy}$ differ. In that case, we pad the shortest one with additional variables that will not be used.

Let $n$ be the maximum of the lengths of $\hat{\vx}$ and $\hat{\vy}$.

Given a set $\vz$ of 3CNF clauses, we construct the following query $q(\vz)$ for $CCA_\infty$:
\begin{equation}
    q(\vz) = \langle \vx^{(s)}, \gamma, \frac{1}{2}, C_u, h \rangle
\end{equation}

where $\frac{1}{4} < \gamma < \frac{1}{2}$ and $\vx^{(s)} = \left (\frac{1}{2}, \dots, \frac{1}{2}\right)$ is a vector with $n$ elements.
Verifying $q(\vz) \in CCA_\infty$ is equivalent to checking:

\begin{equation}
    \label{eqn:ccaExplicitEquation}
    \exists \vx^\prime \in B \left (x_{s}, \frac{1}{2} \right ) . \left ( h(\vx^\prime) \neq h(\vx) \land \left ( \forall \vx^{\prime\prime} \in B \left (\vx^\prime, \frac{1}{4} \right )  .\; h(\vx^{\prime\prime}) = h(\vx^\prime) \right ) \right )
\end{equation}

Note that $\vx^\prime \in [0, 1]^n$.

\paragraph{Truth Values} We will encode the truth values of $\hat{\vx}$ and $\hat{\vy}$ as follows:
\begin{equation}
\label{eqn:truthEncoding}
    \begin{split}
        x^{\prime\prime}_i \in \left (0, \frac{1}{4} \right ) & \iff \hat{x}_i = 0 \land \hat{y}_i = 0 \\
        x^{\prime\prime}_i \in \left (\frac{1}{4}, \frac{1}{2} \right ) & \iff \hat{x}_i = 0 \land \hat{y}_i = 1 \\
        x^{\prime\prime}_i \in \left (\frac{1}{2}, \frac{3}{4} \right ) & \iff \hat{x}_i = 1 \land \hat{y}_i = 0 \\
        x^{\prime\prime}_i \in \left (\frac{3}{4}, 1 \right ) & \iff \hat{x}_i = 1 \land \hat{y}_i = 1 \\
    \end{split}
\end{equation}

Let $e_{\hat{\vx}i}(\vx) = gt \left (x_i, \frac{1}{2} \right )$. Let:
\begin{equation}
    e_{\hat{\vy}i}(\vx) = or \left (open \left (x_i, \frac{1}{4}, \frac{1}{2} \right ), open \left (x_i, \frac{3}{4}, 1 \right ) \right )
\end{equation}

Note that $e_{\hat{\vx}i}(x^{\prime\prime}_i)$ returns the truth value of $\hat{x}_i$ and $e_{\hat{\vy}i}(x^{\prime\prime}_i)$ returns the truth value of $\hat{y}_i$ (as long as the input is within one of the ranges described in \Cref{eqn:truthEncoding}).

\paragraph{Invalid Encodings} All the encodings other than the ones described in \Cref{eqn:truthEncoding} are not valid.
We define $inv_F$ as follows:
\begin{equation}
    inv_F(\vx) = or_{i = 1, \dots, n}or(out(x_i), edge(x_i))
\end{equation}

where $out(x_i) = or(leq(x_i, 0), geq(x_i, 1))$ and
\begin{equation}
    edge(x_i) = or \left (eq \left (x_i, \frac{1}{4} \right), eq \left (x_i, \frac{1}{2} \right ),  eq \left (x_i, \frac{3}{4} \right ) \right)
\end{equation}

On the other hand, we define $inv_T$ as follows:
\begin{equation}
    inv_T(\vx) = or_{i = 1, \dots, n}eq \left (x_i, \frac{1}{2} \right )
\end{equation}

\paragraph{Definition of $h$} Let $g$ be a Boolean formula defined over $e^{(x)}(\vx)$ and $e^{(y)}(\vx)$ that returns the value of $R$ (using the same technique as $cnf^\prime_3$). 

We define $h$ as a two-class classifier, where:
\begin{equation}
    h_1(\vx) = or(inv_T(\vx), and(not(inv_F(\vx)), g(\vx)))
\end{equation}

and $h_0(\vx) = not(h_1(\vx))$.

Note that:
\begin{itemize}
    \item If $x_i = \frac{1}{2}$ for some $i$, the top class is 1; therefore, $h(\vx^{(s)}) = 1$;
    \item Otherwise, if $\vx$ is not a valid encoding, the top class is 0;
    \item Otherwise, the top class is 1 if $R(e^{(x)}(\vx), e^{(y)}(\vx))$ is true and 0 otherwise.
\end{itemize}

\begin{lemma}
$\vz \in co\Pi_2{}3SAT \implies q(\vz) \in CCA_\infty$
\end{lemma}

\begin{proof}
If $\vz \in co\Pi_2{}3SAT$, then there exists a Boolean vector $\vx^*$ such that $\forall \hat{\vy}. \lnot R(\vx^*, \hat{\vy})$.

We now prove that setting $\vx^\prime = \vx^*$ satisfies \Cref{eqn:definitionCCA}.
First, note that $h(\vx^*) = 0$, which satisfies $h(\vx^\prime) \neq h(\vx)$. Then we need to verify that $\forall \vx^{\prime\prime} \in B_\infty(\vx^*, \gamma) . h(\vx) = 0$.

For every $\vx^{\prime\prime} \in B_\infty(\vx^*, \gamma)$, we know that $\vx^{\prime\prime} \in \left ([-\gamma, \gamma] \cup [1 - \gamma, 1 + \gamma] \right )^n$. There are thus two cases:
\begin{itemize}
    \item $\vx^{\prime\prime}$ is not a valid encoding, i.e. $x^{\prime\prime}_i \leq 0 \lor x^{\prime\prime}_i \geq 1 \lor x^{\prime\prime}_i \in \left \{ \frac{1}{4}, \frac{3}{4} \right \}$ for some $i$. Then $h(\vx^{\prime\prime}) = 0$. Note that, since $\gamma < \frac{1}{2}$, $\frac{1}{2} \not \in [-\gamma, \gamma] \cup [1 - \gamma, 1 + \gamma]$, so it is not possible for $\vx^{\prime\prime}$ to be an invalid encoding that is classified as 1;
    \item $\vx^{\prime\prime}$ is a valid encoding. Then, since $\gamma < \frac{1}{2}$, $e^{(x)}(\vx^{\prime\prime}) = \vx^*$. Since $h(\vx^{\prime\prime}) = 1$ iff $R(e^{(x)}(\vx^{\prime\prime}),e^{(y)}(\vx^{\prime\prime}))$ is true and since $R(\vx^*, \hat{\vy})$ is false for all choices of $\hat{\vy}$, $h(\vx^{\prime\prime}) = 0$.
\end{itemize}

Therefore, $\vx^*$ satisfies \Cref{eqn:ccaExplicitEquation} and thus $q(\vz) \in CCA_\infty$.

\end{proof}

\begin{lemma}
$q(\vz) \in CCA_\infty \implies \vz \in co\Pi_2{}3SAT$
\end{lemma}

\begin{proof}
Since $q(\vz) \in CCA$, there exists a $\vx^* \in B \left(\vx^{(s)}, \frac{1}{2} \right)$ such that $h(\vx^*) \neq h(\vx^{(s)})$ and $\forall \vx^{\prime\prime} \in B_\infty(\vx^*, \gamma) . h(\vx^{\prime\prime}) = h(\vx^\prime)$.
We will prove that $e^{(x)}(\vx^*)$ is a solution to $co\Pi_2{}3SAT$.

Since $h(\vx^{(s)}) = 1$, $h(\vx^*) = 0$, which means that $\forall \vx^{\prime\prime} \in B_\infty(\vx^*, \gamma) . h(\vx^{\prime\prime}) = 0$.

We know that $\vx^* \in B_\infty \left (\vx^{(s)}, \frac{1}{2} \right ) = [0, 1]^n$.
We first prove by contradiction that $\vx^* \in \left ( [0, \frac{1}{2} - \gamma) \cup (\frac{1}{2} + \gamma, 1] \right )^n$.
If $x^*_i \in [\frac{1}{2} - \gamma, \frac{1}{2} + \gamma ]$ for some $i$, then the vector $\vx^{(w)}$ defined as follows:
\begin{equation}
    x^{(w)}_j = \begin{cases}
        \frac{1}{2} & i = j \\
        x^*_i & \text{otherwise}
    \end{cases}
\end{equation}
is such that $\vx^{(w)} \in B_\infty(x^*_i, \gamma)$ and $h \left (\vx^{(w)} \right ) = 1$ (since $inv_T \left (\vx^{(w)} \right ) = 1$). This contradicts the fact that $\forall \vx^{\prime\prime} \in B_p(\vx^*, \gamma) . h(\vx) = 0$.
Therefore, $\vx^* \in \left ( [0, \frac{1}{2} - \gamma) \cup (\frac{1}{2} + \gamma, 1] \right )^n$.

As a consequence, $\forall \vx^{\prime\prime} \in B_\infty(\vx^*, \gamma) . e^{(x)}(\vx^{\prime\prime}) = e^{(x)}(\vx^*)$.

We now prove that $\forall \hat{\vy}^* . \exists \vx^{\prime\prime *} \in B_\infty(\vx^*, \gamma)$ such that $e^{(y)}(\vx^{\prime\prime *}) = \hat{\vy}^*$. We can construct such $\vx^{\prime\prime *}$ as follows. For every $i$:
\begin{itemize}
    \item If $e^{(x)}(\vx^*) = 0$ and $e^{(y)}(\vx^*) = 0$, set $\vx^{\prime\prime *}_i$ equal to a value in $\left (0, \frac{1}{4} \right )$;
    \item If $e^{(x)}(\vx^*) = 0$ and $e^{(y)}(\vx^*) = 1$, set $\vx^{\prime\prime *}_i$ equal to a value in $\left (\frac{1}{4}, \gamma \right )$;
    \item If $e^{(x)}(\vx^*) = 1$ and $e^{(y)}(\vx^*) = 0$, set $\vx^{\prime\prime *}_i$ equal to a value in $\left (1 - \gamma, \frac{3}{4} \right )$;
    \item If $e^{(x)}(\vx^*) = 1$ and $e^{(y)}(\vx^*) = 1$, set $\vx^{\prime\prime *}_i$ equal to a value in $\left (\frac{3}{4}, 1 \right )$.
\end{itemize}

By doing so, we have obtained a $\vx^{\prime\prime *}$ such that $\vx^{\prime\prime *} \in B_\infty(\vx^*, \gamma)$ and $e^{(y)}(\vx^{\prime\prime *}) = \hat{\vy}^*$.

Since:
\begin{itemize}
    \item $e^{(x)}(\vx^{\prime\prime}) = e^{(x)}(\vx^*) $ for all $\vx^{\prime\prime}$;
    \item $h(\vx^{\prime\prime}) = 0$ for all $\vx^{\prime\prime}$;
    \item $h(\vx^{\prime\prime}) = 1$ iff $R(e^{(x)}(\vx^{\prime\prime})), e^{(y)}(\vx^{\prime\prime}))$ is true;
\end{itemize}
$R(e^{(x)}(\vx^*), \hat{\vy}^*)$ is false for all choices of $\hat{\vy}^*$. In other words, $\hat{\vx}^*$ is a solution to \Cref{eqn:coForallExists3SatCondition} and thus $\vz \in co\Pi_2{}3SAT$.

\end{proof}

\noindent
Since:
\begin{itemize}
    \item $q(\vz)$ can be computed in polynomial time;
    \item $\vz \in co\Pi_2{}3SAT \implies q(\vz) \in CCA_\infty$;
    \item $q(\vz) \in CCA_\infty \implies \vz \in co\Pi_2{}3SAT$;
\end{itemize}
we can conclude that $co\Pi_2{}3SAT \leq CCA$.

\subsection{Proof of \Cref{cor:ccaLp}}

The proof of $CCA_p \in \Sigma_2^P$ is the same as the one for \Cref{the:foolingCA}.

For the hardness proof, we follow a more involved approach compared to those for \Cref{cor:lpAttack,cor:lpDefense}.

First, let $\varepsilon_{\rho_{p,n}}$ be the value of epsilon such that $\rho_{p, n}\left(\varepsilon_{\rho_{p,n}}\right) = \frac{1}{2}$. In other words, $B_p(\vx^{(s)}, \varepsilon_{\rho_{p,n}})$ is an $L^p$ ball that contains $[0, 1]^n$, while the intersection of the corresponding $L^p$ sphere and $[0, 1]^n$ is the set $\{0, 1\}^n$ (for $p < \infty$).

Let $inv^\prime_T(\vx)$ be defined as follows:
\begin{equation}
    inv^\prime_T(\vx) = or_{i = 1, \dots, n} \left (or \left (eq \left (x_i, \frac{1}{2} \right ), leq(x_i, 0), geq(x_i, 1) \right ) \right )
\end{equation}

Let $inv^\prime_F(\vx)$ be defined as follows:
\begin{equation}
    inv^\prime_F(\vx) = or_{i = 1, \dots, n} \left (or \left (eq \left (x_i, \frac{1}{4} \right ), eq \left (x_i, \frac{3}{4} \right ) \right) \right )
\end{equation}

We define $h^\prime$ as follows:
\begin{equation}
    h^\prime_1 = or(inv^\prime_T(\vx), and(not(inv^\prime_F(\vx)), g(\vx))
\end{equation}
with $h^\prime_0(\vx) = not(h^\prime_1(\vx))$.

Note that:
\begin{itemize}
    \item If $x_i \in (-\infty, 0] \cup \{\frac{1}{2}\} \cup [1, \infty)$ for some $i$, then the top class is 1;
    \item Otherwise, if $\vx$ is not a valid encoding, the top class is 0;
    \item Otherwise, the top class is 1 if $R(e^{(x)}(\vx), e^{(y)}(\vx))$ is true and 0 otherwise.
\end{itemize}

Finally, let $\frac{1}{8} < \gamma^\prime < \frac{1}{4}$.
Our query is thus:
\begin{equation}
    q(\vz) = \langle \vx^{(s)}, \gamma^\prime, \frac{1}{2}, C_u, h^\prime \rangle
\end{equation}

\paragraph{Proof of $\vz \in co\Pi_2{}3SAT \implies q(\vz) \in CCA_p$} If $\vz \in co\Pi_2{}3SAT$, then $\exists \vx^*. \forall \hat{\vy} . \lnot R(\vx^*, \hat{\vy})$. Let $\vx^{**}$ be defined as follows:
\begin{equation}
    x^{**}_i = \begin{cases} \frac{1}{4} & x^*_i = 0 \\
    \frac{3}{4} & x^*_i = 1\end{cases}
\end{equation}

Note that:
\begin{itemize}
    \item $\vx^{**} \in B_p \left (\vx^{(s)}, \varepsilon_{\rho_{p,n}} \right )$;
    \item $e^{(x)}(\vx^{**}) = \vx^*$;
    \item $f(\vx^{**}) = 0$, since $\vx^{**} \in \{\frac{1}{4}, \frac{3}{4}\}^n$;
    \item Since $\gamma^\prime < \frac{1}{4}$, there is no $i$ such that $\exists \vx^{\prime\prime} \in B_p(\vx^{**}, \gamma^\prime) . x^{\prime\prime}_i \in (-\infty, 0] \cup \left \{\frac{1}{2} \right \} \cup [1, \infty)$;
    \item For all $\vx^{\prime\prime} \in B_p(\vx^{**}, \gamma^\prime)$:
    \begin{itemize}
        \item If $\vx^{\prime\prime}$ is not a valid encoding (i.e. $x^{\prime\prime}_i \in \{\frac{1}{4}, \frac{3}{4}\}$ for some $i$), then $h^\prime(\vx^{\prime\prime}) = 0$;
        \item Otherwise, $h^\prime(\vx^{\prime\prime}) = 1$ iff $R(e^{(x)}(\vx^{\prime\prime}), e^{(y)}(\vx^{\prime\prime}))$ is true.
    \end{itemize}
\end{itemize}

Therefore, since $\forall \hat{\vy}. \lnot R(\vx^*, \hat{\vy})$, we know that $\forall \vx^{\prime\prime} \in B_p(\vx^{**}, \gamma^\prime) . f(\vx^{\prime\prime}) = 0$. In other words, $\vx^{**}$ is a solution to \Cref{eqn:definitionCCA}.

\paragraph{Proof of $q(\vz) \in CCA_p \implies \vz \in co\Pi_2{}3SAT$} If $q(\vz) \in CCA_p$, then we know that $\exists \vx^* \in B_p \left (\vx^{(s)}, \varepsilon_{\rho_{p,n}} \right ) .\left ( h^\prime(\vx^*) \neq h(\vx^{(s)}) \land \forall \vx^{\prime\prime} \in B_p(\vx^*, \gamma^\prime) . h^\prime(\vx^{\prime\prime}) = h^\prime(\vx^*) \right )$. In other words, $\exists \vx^* \in B_p \left (\vx^{(s)}, \varepsilon_{\rho_{p,n}} \right ) . \left ( h^\prime(\vx^*) = 0 \land \forall \vx^{\prime\prime} \in B_p(\vx^*, \gamma^\prime) . h^\prime(\vx^{\prime\prime}) = 0 \right )$.

We will first prove by contradiction that $\vx^* \in \left ((\gamma^\prime, \frac{1}{2} - \gamma^\prime) \cup (\frac{1}{2} + \gamma^\prime, 1 - \gamma^\prime) \right )^n$.

First, suppose that $x^*_i \in (-\infty, 0) \cup (1, \infty)$ for some $i$. Then $h^\prime(x^*) = 0$ due to the fact that $inv_T(x^*) = 1$.

Second, suppose that $x^*_i \in [0, \gamma^\prime] \cup [1 - \gamma^\prime, 1]$ for some $i$. Then $\vx^{(w)}$, defined as follows:
\begin{equation}
    x^{(w)}_j = \begin{cases} 0 & i = j \land x^*_i \in [0, \gamma^\prime] \\
    1 & i = j \land x^*_i \in [1 - \gamma^\prime, 1] \\
    x^*_j & j \neq i\end{cases}
\end{equation}
is such that $\vx^{(w)} \in B_p(\vx^*, \gamma^\prime)$ but $h^\prime(\vx^{(w)}) = 1$.

Finally, suppose that $x^*_i \in [\frac{1}{2} - \gamma, \frac{1}{2} + \gamma]$ for some $i$. Then $\vx^{(w)}$, defined as follows:
\begin{equation}
    x^{(w)}_j = \begin{cases} \frac{1}{2} & i = j \\
    x^*_j & \text{otherwise} \end{cases}
\end{equation}
is such that $\vx^{(w)} \in B_p(\vx^*, \gamma^\prime)$ but $h^\prime(\vx^{(w)}) = 1$.

Therefore, $\vx^* \in \left ((\gamma^\prime, \frac{1}{2} - \gamma^\prime) \cup (\frac{1}{2} + \gamma^\prime, 1 - \gamma^\prime) \right )^n$.

As a consequence $\forall \vx^{\prime\prime} \in B_p(\vx^*, \gamma^\prime) . e^{(x)}(\vx^{\prime\prime}) = e^{(x)}(\vx^\prime)$.

From this, due to the fact that $\gamma^\prime > \frac{1}{8}$ and that $p > 0$, we can conclude that for all $\hat{\vy}$, there exists a $\vx^{\prime\prime} \in B_p(\vx^*, \gamma^\prime)$ such that:
\begin{align}
\begin{split}
    x^{\prime\prime}_i \in \left (0, \frac{1}{4} \right ) & \text{ for } x^*_i \in \left (\gamma^\prime, \frac{1}{2} - \gamma^\prime \right ), \hat{y}_i = 0 \\
    x^{\prime\prime}_i \in \left (\frac{1}{4}, \frac{1}{2} \right) & \text{ for } x^*_i \in \left (\gamma^\prime, \frac{1}{2} - \gamma^\prime \right ), \hat{y}_i = 1 \\
    x^{\prime\prime}_i \in \left (\frac{1}{2}, \frac{3}{4} \right ) & \text{ for } x^*_i \in \left (\frac{1}{2} + \gamma^\prime, 1 - \gamma^\prime \right ), \hat{y}_i = 0 \\
    x^{\prime\prime}_i \in \left (\frac{3}{4}, 1 \right) & \text{ for } x^*_i \in \left (\frac{1}{2} + \gamma^\prime, 1 - \gamma^\prime \right ), \hat{y}_i = 1 
\end{split}
\end{align}

In other words, for all $\hat{\vy}$ there exists a corresponding $\vx^{\prime\prime} \in B_p(\vx^*, \gamma^\prime)$ such that $e^{(y)}(\vx^{\prime\prime}) = \hat{\vy}$.

Therefore, since $h^\prime(\vx^{\prime\prime}) = 1$ iff $R(e^{(x)}(\vx^{\prime\prime}), e^{(y)}(\vx^{\prime\prime}))$ is true and since $\forall \vx^{\prime\prime} \in B_p(\vx^*, \gamma^\prime) . h^\prime(\vx^{\prime\prime}) = 0$, we can conclude that $\forall \hat{\vy}. \lnot R(e^{(x)}(\vx^*), \hat{\vy})$. In other words, $\vz \in co\Pi_2{}3SAT$.

\subsection{Proof of \Cref{cor:foolingPolynomial}}

Similarly to the proof of \Cref{cor:attackPolynomial}, it follows from the fact that ReLU classifiers are polynomial-time classifiers (w.r.t. the size of the tuple).


\clearpage
\section{Relation with Robust Transductive Learning}
\label{sec:transductiveLearning}

In this section, we outline the similarities between transductive approaches to robust learning (taking the work of \citet{chen2021towards} as an example) and CA: the former fixes the input and adapts the model at inference time, while the latter fixes the model and solves an optimization problem in the input space. In particular, the approach by Chen et al. involves adapting the model at test time to a set $U$ of user-provided (and potentially adversarially corrupted) inputs. For the sake of clarity, we rewrite the adaptation task as follows:
\begin{equation}
    \label{eqn:adaptation}
    \argmin_\vtheta \mathcal{L}_d(f_\vtheta, U)
\end{equation}

where $\mathcal{L}_d$ is an unsupervised adaptation loss, and define $\Gamma(U) = f_{\vtheta^*}$, where $\vtheta^*$ is the solution to \Cref{eqn:adaptation}.
Attacking this technique thus involves solving the following constrained optimization problem (adapted from Equation 6 of the original paper):
\begin{equation}
    \argmax_{U^\prime \in N(U)} \mathcal{L}_a(f_{\vtheta^*, U^\prime}) \text{ s.t. } \vtheta^* = \argmin_\vtheta \mathcal{L}_d(f_\vtheta, U^\prime)
\end{equation}

where $\mathcal{L}_a$ is the loss for the adversarial objective. Chen then provides an alternative formulation (Equation 8 of the original paper) that is more tractable from a formal point of view; however, our adapted equation is a good starting point for our comparison.
In particular, as in our work, attacking the approach by Chen et al. requires solving a problem that involves nested optimization, and therefore the same 
 ``core'' of complexity. With this formulation, the connections between Chen et al. and CA become clear:
\begin{compactitem}
    \item Both approaches use an optimization problem at inference time that is parameterized over the input (thus avoiding the second informal asymmetry mentioned in \Cref{sec:additionalSources});
    \item Attacks against both approaches lead to nested optimization problems.
\end{compactitem}

We therefore conjecture that it should be possible to extend the result from our \Cref{the:foolingCA} to the approach by Chen et. al.
However, there are some differences between the work of Chen et al. and ours that will likely need to be addressed in a formal proof:
\begin{compactitem}
    \item The former is designed with transductive learning in mind, while the latter is intended for ``regular'' classification (i.e. where the model is fixed);
    \item The former is meant to be used with arbitrarily large sets of inputs, while the latter only deals with one input at the time;
    \item The former uses two different losses ($\mathcal{L}_d$ and $\mathcal{L}_a$), which can potentially make theoretical analyses more complex;
    \item There are several possible ways to adapt a model to a given $U$, and a proof would likely have to consider a sufficiently ``interesting'' subset of such techniques.
\end{compactitem}

We hope that our theoretical findings will encourage research into such areas.
\clearpage
\section{Full Experimental Setup}
\label{app:fullExperimentalSetup}




All our code is written in Python + PyTorch \citep{pytorch}, with the exception of the MIPVerify interface, which is written in Julia. When possible, most experiments were run in parallel, in order to minimize execution times.

\paragraph{Models}

All models were trained using Adam \citep{kingma2014adam} and dataset augmentation. We performed a manual hyperparameter and architecture search to find a suitable compromise between accuracy and MIPVerify convergence. The process required approximately 4 months. When performing adversarial training, following \citep{madry2018towards} we used the final adversarial example found by the Projected Gradient Descent attack, instead of the closest. To maximize uniformity, we used for each configuration the same training and pruning hyperparameters (when applicable), which we report in \Cref{tab:hyperparameters}. We report the chosen architectures in \Cref{tab:architecturesMnist,tab:architecturesCifar10}, while \Cref{tab:accuracies} outlines their accuracies and parameter counts.

\paragraph{UG100}
The first 250 samples of the test set of each dataset were used for hyperparameter tuning and were thus not considered in our analysis. For our G100 dataset, we sampled uniformly across each ground truth label and removed the examples for which MIPVerify crashed. \Cref{tab:groundTruths} details the composition of the dataset by ground truth label.


\paragraph{Attacks} For the Basic Iterative Method (BIM), the Fast Gradient Sign Method (FGSM) and the Projected Gradient Descent (PGD) attack, we used the implementations provided by the AdverTorch library \citep{ding2019advertorch}. For the Brendel \& Bethge (B\&B) attack and the Deepfool (DF) attack, we used the implementations provided by the Foolbox Native library \citep{rauber2020foolbox}. The Carlini \& Wagner and the uniform noise attacks were instead implemented by the authors. We modified the attacks that did not return the closest adversarial example found (i.e. BIM, Carlini \& Wagner, Deepfool, FGSM and PGD) to do so. For the attacks that accept $\varepsilon$ as a parameter (i.e. BIM, FGSM, PGD and uniform noise), for each example we first performed an initial search with a decaying value of $\varepsilon$, followed by a binary search.
In order to pick the attack parameters, we first selected the strong set by performing an extensive manual search. The process took approximately 3 months. We then modified the strong set in order to obtain the balanced parameter set. We report the parameters of both sets (as well as the parameters of the binary and $\varepsilon$ decay searches) in \Cref{tab:heuristicParameters}.






\paragraph{MIPVerify} We ran MIPVerify using the Julia library MIPVerify.jl and Gurobi \citep{gurobi}. Since MIPVerify can be sped up by providing a distance upper bound, we used the same pool of adversarial examples utilized throughout the paper. For CIFAR10 we used the strong parameter set, while for MNIST we used the strong parameter set with some differences (reported in \Cref{tab:mnistMipParameterDifferences}). Since numerical issues might cause the distance upper bound computed by the heuristic attacks to be slightly different from the one computed by MIPVerify, we ran a series of \emph{exploratory runs}, each with a different correction factor (1.05, 1.25, 1.5, 2), and picked the first factor that caused MIPVerify to find a feasible (but not necessarily optimal) solution. If the solution was not optimal, we then performed a \emph{main run} with a higher computational budget. We provide the parameters of MIPVerify in \Cref{tab:mipVerifyParameters}.
We also report in \Cref{tab:tightness} the percentage of tight bounds for each combination.

\begin{table}[]
\caption{Training and pruning hyperparameters.}
\centering
\begin{tabular}{lcc}
\toprule
\multicolumn{1}{l}{\multirow{2}{*}{\textbf{Parameter Name}}} & \multicolumn{2}{c}{\textbf{Value}} \\ \cmidrule{2-3} 
\multicolumn{1}{l}{} & \multicolumn{1}{c}{\textbf{MNIST}} & \textbf{CIFAR10} \\ \midrule
\multicolumn{3}{c}{Common Hyperparameters} \\ \midrule
\multicolumn{1}{l}{Epochs} & \multicolumn{2}{c}{425} \\
\multicolumn{1}{l}{Learning Rate} & \multicolumn{2}{c}{1e-4} \\
\multicolumn{1}{l}{Batch Size} & \multicolumn{1}{c}{32} & \multicolumn{1}{c}{128} \\
\multicolumn{1}{l}{Adam $\beta$} & \multicolumn{2}{c}{(0.9, 0.999)} \\
\multicolumn{1}{l}{Flip \%} & \multicolumn{2}{c}{50\%} \\
\multicolumn{1}{l}{Translation Ratio} & \multicolumn{2}{c}{0.1} \\
\multicolumn{1}{l}{Rotation (deg.)} & \multicolumn{2}{c}{15°} \\ \midrule
\multicolumn{3}{c}{Adversarial Hyperparameters (Adversarial and ReLU only)} \\ \midrule
\multicolumn{1}{l}{Attack} & \multicolumn{2}{c}{PGD} \\
\multicolumn{1}{l}{Attack \#Iterations} & \multicolumn{2}{c}{200} \\
\multicolumn{1}{l}{Attack Learning Rate} & \multicolumn{2}{c}{0.1} \\
\multicolumn{1}{l}{Adversarial Ratio} & \multicolumn{2}{c}{1} \\
\multicolumn{1}{l}{$\varepsilon$} & \multicolumn{1}{c}{0.05} & 2/255 \\ \midrule
\multicolumn{3}{c}{ReLU Hyperparameters (ReLU only)} \\ \midrule
\multicolumn{1}{l}{L1 Regularization Coeff.} & \multicolumn{1}{c}{2e-5} & 1e-5 \\
\multicolumn{1}{l}{RS Loss Coeff.} & \multicolumn{1}{c}{1.2e-4} & 1e-3 \\
\multicolumn{1}{l}{Weight Pruning Threshold} & \multicolumn{2}{c}{1e-3} \\
\multicolumn{1}{l}{ReLU Pruning Threshold} & \multicolumn{2}{c}{90\%} \\
\bottomrule \end{tabular}
\label{tab:hyperparameters}
\end{table}

\begin{table}[]
\caption{MNIST Architectures.}
\centering
\subfloat[MNIST A]{
\begin{tabular}{|c|}
\hline
Input \\ \hline
Flatten \\ \hline
Linear (in = 784, out = 100) \\ \hline
ReLU \\ \hline
Linear (in = 100, out = 10) \\ \hline
Output \\
\hline \end{tabular}
}

\subfloat[MNIST B]{
\begin{tabular}{|c|}
\hline
Input \\ \hline
Conv2D (in = 1, out = 4, 5x5 kernel, stride = 3, padding = 0) \\ \hline
ReLU \\ \hline
Flatten \\ \hline
Linear (in = 256, out = 10) \\ \hline
Output \\
\hline \end{tabular}
}

\subfloat[MNIST C]{
\begin{tabular}{|c|}
\hline
Input \\ \hline
Conv2D (in = 1, out = 8, 5x5 kernel, stride = 4, padding = 0) \\ \hline
ReLU \\ \hline
Flatten \\ \hline
Linear (in = 288, out = 10) \\ \hline
Output \\
\hline \end{tabular}
}

\label{tab:architecturesMnist}

\end{table}

\begin{table}[]
\caption{CIFAR10 architectures.}
\centering
\subfloat[CIFAR10 A]{
\begin{tabular}{|c|}
\hline
Input \\ \hline
Conv2D (in = 3, out = 8, 3x3 kernel, stride = 2, padding = 0) \\ \hline
ReLU \\ \hline
Flatten \\ \hline
Linear (in = 1800, out = 10) \\ \hline
Output \\
\hline \end{tabular}
}

\subfloat[CIFAR10 B]{
\begin{tabular}{|c|}
\hline
Input \\ \hline
Conv2D (in = 3, out = 20, 5x5 kernel, stride = 4, padding = 0) \\ \hline
ReLU \\ \hline
Flatten \\ \hline
Linear (in = 980, out = 10) \\ \hline
Output \\
\hline \end{tabular}
}

\subfloat[CIFAR10 C]{
\begin{tabular}{|c|}
\hline
Input \\ \hline
Conv2D (in = 3, out = 8, 5x5 kernel, stride = 4, padding = 0) \\ \hline
ReLU \\ \hline
Conv2D (in = 8, out = 8, 3x3 kernel, stride = 2, padding = 0) \\ \hline
ReLU \\ \hline
Flatten \\ \hline
Linear (in = 72, out = 10) \\ \hline
Output \\
\hline \end{tabular}
}
\label{tab:architecturesCifar10}
\end{table}

\begin{table}[]
\caption{Parameter counts and accuracies of trained models.}
\centering
\begin{tabular}{llll}
\toprule
\textbf{Architecture} & \textbf{\#Parameters} & \textbf{Training} & \textbf{Accuracy} \\ \midrule
\multirow{3}{*}{MNIST A} & \multirow{3}{*}{79510} & Standard & 95.87\% \\
 & & Adversarial & 94.24\% \\
 & & ReLU & 93.57\% \\ \midrule
\multirow{3}{*}{MNIST B} & \multirow{3}{*}{2674} & Standard & 89.63\% \\
 & & Adversarial & 84.54\% \\
 & & ReLU & 83.69\% \\ \midrule
\multirow{3}{*}{MNIST C} & \multirow{3}{*}{3098} & Standard & 90.71\% \\
 & & Adversarial & 87.35\% \\
 & & ReLU & 85.67\% \\ \midrule
\multirow{3}{*}{CIFAR10 A} & \multirow{3}{*}{18234} & Standard & 53.98\% \\ 
 & & Adversarial & 50.77\% \\
 & & ReLU & 32.85\% \\ \midrule
\multirow{3}{*}{CIFAR10 B} & \multirow{3}{*}{11330} & Standard & 55.81\% \\
 & & Adversarial & 51.35\% \\
 & & ReLU & 37.33\% \\ \midrule
\multirow{3}{*}{CIFAR10 C} & \multirow{3}{*}{1922} & Standard & 47.85\% \\
 & & Adversarial & 45.19\% \\
 & & ReLU & 32.27\% \\
\bottomrule \end{tabular}
\label{tab:accuracies}
\end{table}

\begin{table}[]
\caption{Ground truth labels of the UG100 dataset.}
\centering
\hfill
\subfloat[MNIST]{
\begin{tabular}{lll}
\midrule
\multicolumn{1}{c}{\textbf{Ground Truth}} & \multicolumn{1}{c}{\textbf{Count}} & \multicolumn{1}{c}{\textbf{\%}} \\ \midrule
0 & 219 & 9.77\% \\
1 & 228 & 10.17\% \\
2 & 225 & 10.04\% \\
3 & 225 & 10.04\% \\
4 & 225 & 10.04\% \\
5 & 220 & 9.82\% \\
6 & 227 & 10.13\% \\
7 & 221 & 9.86\% \\
8 & 225 & 10.04\% \\
9 & 226 & 10.08\% \\
\bottomrule \end{tabular}
}
\hfill
\subfloat[CIFAR10]{
\begin{tabular}{lll}
\midrule
\multicolumn{1}{c}{\textbf{Ground Truth}} & \multicolumn{1}{c}{\textbf{Count}} & \multicolumn{1}{c}{\textbf{\%}} \\ \midrule
Airplane & 228 & 10.05\% \\
Automobile & 227 & 10.00\% \\
Bird & 228 & 10.05\% \\
Cat & 228 & 10.05\% \\
Deer & 226 & 9.96\% \\
Dog & 227 & 10.00\% \\
Frog & 227 & 10.00\% \\
Horse & 227 & 10.00\% \\
Ship & 225 & 9.92\% \\
Truck & 226 & 9.96\% \\
\bottomrule \end{tabular}
}
\hfill
\label{tab:groundTruths}
\end{table}

\begin{table}[]
\caption{Parameters of heuristic attacks.}
\centering
\begin{tabular}{llclll}
\toprule
\multicolumn{1}{c}{\multirow{2}{*}{\textbf{Attack}}} & \multicolumn{1}{c}{\multirow{2}{*}{\textbf{Parameter Name}}} & \multicolumn{2}{c}{\textbf{MNIST}} & \multicolumn{2}{c}{\textbf{CIFAR10}} \\ \cmidrule{3-6} 
\multicolumn{1}{c}{} & \multicolumn{1}{c}{} & \multicolumn{1}{c}{\textbf{Strong}} & \multicolumn{1}{c}{\textbf{Balanced}} & \multicolumn{1}{c}{\textbf{Strong}} & \multicolumn{1}{c}{\textbf{Balanced}} \\ \midrule
\multirow{5}{*}{BIM} & Initial Search Factor & \multicolumn{4}{c}{0.75} \\
 & Initial Search Steps & \multicolumn{4}{c}{30} \\
 & Initial Search Factor & \multicolumn{4}{c}{0.75} \\
 & Binary Search Steps & \multicolumn{4}{c}{20} \\
 & \#Iterations & \multicolumn{1}{l}{2k} & \multicolumn{1}{l}{200} & \multicolumn{1}{l}{5k} & 200 \\ 
 & Learning Rate & \multicolumn{1}{l}{1e-3} & \multicolumn{1}{l}{1e-2} & \multicolumn{1}{l}{1e-5} & 1e-3 \\
\midrule
\multirow{7}{*}{Brendel \& Bethge} & Initial Attack & \multicolumn{4}{c}{Blended Noise} \\
 & Overshoot & \multicolumn{4}{c}{1.1} \\ 
 & LR Decay & \multicolumn{4}{c}{0.75} \\
 & LR Decay Every n Steps & \multicolumn{4}{c}{50} \\
 & \#Iterations & \multicolumn{1}{l}{5k} & \multicolumn{1}{l}{200} & \multicolumn{1}{l}{5k} & 200 \\ 
 & Learning Rate & \multicolumn{1}{l}{1e-3} & \multicolumn{1}{l}{1e-3} & \multicolumn{1}{l}{1e-5} & 1e-3 \\
 & Momentum & \multicolumn{4}{c}{0.8} \\
 & Initial Directions & \multicolumn{4}{c}{1000} \\
 & Init Steps & \multicolumn{4}{c}{1000} \\
\midrule
\multirow{14}{*}{Carlini \& Wagner} & Minimum $\tau$ & \multicolumn{4}{c}{1e-5} \\
 & Initial $\tau$ & \multicolumn{4}{c}{1} \\
 & $\tau$ Factor & \multicolumn{1}{l}{0.95} & \multicolumn{1}{l}{0.9} & \multicolumn{1}{l}{0.99} & 0.9 \\ 
 & Initial Const & \multicolumn{4}{c}{1e-5} \\
 & Const Factor & \multicolumn{4}{c}{2} \\
 & Maximum Const & \multicolumn{4}{c}{20} \\
 & Reduce Const & \multicolumn{4}{c}{False} \\
 & Warm Start & \multicolumn{4}{c}{True} \\
 & Abort Early & \multicolumn{4}{c}{True} \\
 & Learning Rate & \multicolumn{1}{l}{1e-2} & \multicolumn{1}{l}{1e-2} & \multicolumn{1}{l}{1e-5} & 1e-4 \\
 & Max Iterations & \multicolumn{1}{l}{1k} & \multicolumn{1}{l}{100} & \multicolumn{1}{l}{5k} & 100 \\
 & $\tau$ Check Every n Steps & \multicolumn{4}{c}{1} \\
 & Const Check Every n Steps & \multicolumn{4}{c}{5} \\
\midrule
\multirow{4}{*}{Deepfool} & \#Iterations & \multicolumn{4}{c}{5k} \\
 & Candidates & \multicolumn{4}{c}{10} \\
 & Overshoot & \multicolumn{4}{c}{1e-5} \\
\midrule
\multirow{3}{*}{FGSM} & Initial Search Factor & \multicolumn{4}{c}{0.75} \\
 & Initial Search Steps & \multicolumn{4}{c}{30} \\
 & Initial Search Factor & \multicolumn{4}{c}{0.75} \\
 & Binary Search Steps & \multicolumn{4}{c}{20} \\
\midrule
\multirow{6}{*}{PGD} & Initial Search Factor & \multicolumn{4}{c}{0.75} \\
 & Initial Search Steps & \multicolumn{4}{c}{30} \\
 & Initial Search Factor & \multicolumn{4}{c}{0.75} \\
 & Binary Search Steps & \multicolumn{4}{c}{20} \\
 & \#Iterations & \multicolumn{1}{l}{5k} & \multicolumn{1}{l}{200} & \multicolumn{1}{l}{5k} & 200 \\
 & Learning Rate & \multicolumn{1}{l}{1e-4} & \multicolumn{1}{l}{1e-3} & \multicolumn{1}{l}{1e-4} & 1e-3 \\
\midrule
\multirow{4}{*}{Uniform Noise} & Initial Search Factor & \multicolumn{4}{c}{0.75} \\
 & Initial Search Steps & \multicolumn{4}{c}{30} \\
 & Initial Search Factor & \multicolumn{4}{c}{0.75} \\
 & Binary Search Steps & \multicolumn{4}{c}{20} \\
 & Runs & \multicolumn{1}{l}{8k} & \multicolumn{1}{l}{200} & \multicolumn{1}{l}{8k} & 200 \\
\bottomrule \end{tabular}
\label{tab:heuristicParameters}
\end{table}

\begin{table}[]
\caption{Parameter set used to initialize MIPVerify for MNIST. All other parameters are identical to the strong MNIST attack parameter set.}
\centering
\begin{tabular}{lll}
\toprule
\textbf{Attack Name} & \textbf{Parameter Name} & \textbf{Value} \\ \midrule
\multirow{2}{*}{BIM} & \#Iterations & 5k \\
 & Learning Rate & 1e-5 \\ \midrule
Brendel \& Bethge & Learning Rate & 1e-3 \\ \midrule
\multirow{3}{*}{Carlini \& Wagner} & Tau Factor & 0.99 \\
 & Learning Rate & 1e-3 \\
 & \#Iterations & 5k \\
\bottomrule \end{tabular}

\label{tab:mnistMipParameterDifferences}
\end{table}

\begin{table}[]
\caption{Parameters of MIPVerify.}
\centering
\begin{tabular}{lcl}
\toprule
\multicolumn{1}{c}{\multirow{2}{*}{\textbf{Parameter Name}}} & \multicolumn{2}{c}{\textbf{Value}} \\ \cmidrule{2-3} 
\multicolumn{1}{c}{} & \multicolumn{1}{l}{\textbf{Exploration}} & \textbf{Main} \\ \midrule
Absolute Tolerance & \multicolumn{2}{c}{1e-5} \\
Relative Tolerance & \multicolumn{2}{c}{1e-10} \\
Threads & \multicolumn{2}{c}{1} \\
Timeout (s) & \multicolumn{1}{l}{120} & 7200 \\
Tightening Absolute Tolerance & \multicolumn{2}{c}{1e-4} \\
Tightening Relative Tolerance & \multicolumn{2}{c}{1e-10} \\
Tightening Timeout (s) & \multicolumn{1}{l}{20} & 240 \\
Tightening Threads & \multicolumn{2}{c}{1} \\
\bottomrule \end{tabular}
\label{tab:mipVerifyParameters}
\end{table}

\begin{table}[]
\caption{MIPVerify bound tightness statistics.}
\centering
\begin{tabular}{lll}
\toprule
\textbf{Architecture} & \textbf{Training} & \textbf{\% Tight} \\ \midrule
\multirow{3}{*}{MNIST A} & Standard & 95.40\% \\
 & Adversarial & 99.60\% \\
 & ReLU & 82.46\% \\ \midrule
\multirow{3}{*}{MNIST B} & Standard & 74.61\% \\
 & Adversarial & 85.68\% \\
 & ReLU & 75.55\% \\ \midrule
\multirow{3}{*}{MNIST C} & Standard & 86.21\% \\
 & Adversarial & 97.28\% \\ 
 & ReLU & 95.63\% \\ \midrule
\multirow{3}{*}{CIFAR10 A} & Standard & 81.18\% \\
 & Adversarial & 82.50\% \\
 & ReLU & 92.73\% \\ \midrule
\multirow{3}{*}{CIFAR10 B} & Standard & 56.32\% \\
 & Adversarial & 58.88\% \\ 
 & ReLU & 81.67\% \\ \midrule
\multirow{3}{*}{CIFAR10 C} & Standard & 100.00\% \\
 & Adversarial & 100.00\% \\
 & ReLU & 100.00\% \\
\bottomrule \end{tabular}
\label{tab:tightness}
\end{table}
\clearpage
\section{Quantile-Based Calibration}
\label{app:quantileCalibration}

The error correction model in CA can be empirically calibrated so as to control the chance of false positives (i.e. inputs wrongly reported as not robust) and false negatives (i.e. non-robust inputs reported as being robust).

Given the strong correlation that we observed between the distance of heuristic adversarial examples and the true decision boundary distance, using a linear model $b_\alpha$ seems a reasonable choice. Under this assumption, the correction model depends only on the distance between the original example and the adversarial one, i.e. on $\|\vx, a(\vx)\|$. This property allows us to rewrite the main check performed by CA as:
\begin{equation}
    ||\vx - a(\vx))||_p - b(\vx) = \alpha_1 ||\vx - a(\vx)||_p + \alpha_0 \leq \varepsilon
\end{equation}
where $a(\vx)$ is the adversarial example found by the attack $a$ for the input $\vx$.
The parameters $\alpha_1, \alpha_0$ can then be obtained via quantile regression \citep{koenker1978regression} by using the true decision boundary distance (i.e. $d_p^*(x)$) as a target.

The approach provides a simple, interpretable mechanism to control how conservative the detection check should be: with a small quantile, CA will tend to underestimate the decision boundary distance, leading to fewer missed detections, but more false alarms; using a high quantile will lead to the opposite behavior.

We test this type of buffer using 5-fold cross-validation on each configuration. Specifically, we calibrate the model using 1\%, 50\% and 99\% as quantiles. \Cref{tab:calibrationResultsMnistStrong,tab:calibrationResultsMnistBalanced,tab:calibrationResultsCifar10Strong,tab:calibrationResultsCifar10Balanced} provide a comparison between the expected quantile and the average true quantile of each configuration on the validation folds. Additionally, we plot in \Cref{fig:f1MnistA,fig:f1MnistB,fig:f1MnistC,fig:f1Cifar10A,fig:f1Cifar10B,fig:f1Cifar10C} the mean $F_1$ score in relation to the choice of $\varepsilon$.

\begin{table}[]
\caption{Expected vs true quantile for MNIST strong with 5-fold cross validation.}
\centering
\begin{tabular}{llll}
\toprule
\textbf{Architecture} & \textbf{Training} & \textbf{Expected Quantile} & \textbf{True Quantile} \\ \midrule
\multirow{9}{*}{A} & \multirow{3}{*}{Standard} & 1.00\% & 0.99\textpm1.02\% \\
 &  & 50.00\% & 49.93\textpm2.35\% \\
 &  & 99.00\% & 95.60\textpm3.77\% \\ \cmidrule{2-4} 
 & \multirow{3}{*}{Adversarial} & 1.00\% & 1.11\textpm0.53\% \\
 &  & 50.00\% & 50.25\textpm1.58\% \\
 &  & 99.00\% & 89.84\textpm6.42\% \\ \cmidrule{2-4} 
 & \multirow{3}{*}{ReLU} & 1.00\% & 1.11\textpm0.45\% \\
 &  & 50.00\% & 50.02\textpm1.72\% \\
 &  & 99.00\% & 91.95\textpm5.64\% \\ \midrule
\multirow{9}{*}{B} & \multirow{3}{*}{Standard} & 1.00\% & 1.07\textpm0.48\% \\
 &  & 50.00\% & 49.80\textpm0.76\% \\
 &  & 99.00\% & 97.76\textpm0.71\% \\ \cmidrule{2-4} 
 & \multirow{3}{*}{Adversarial} & 1.00\% & 1.22\textpm1.01\% \\
 &  & 50.00\% & 49.88\textpm4.63\% \\
 &  & 99.00\% & 98.10\textpm0.36\% \\ \cmidrule{2-4} 
 & \multirow{3}{*}{ReLU} & 1.00\% & 1.04\textpm0.77\% \\
 &  & 50.00\% & 49.98\textpm3.17\% \\
 &  & 99.00\% & 97.69\textpm1.41\% \\ \midrule
\multirow{9}{*}{C} & \multirow{3}{*}{Standard} & 1.00\% & 1.07\textpm0.37\% \\
 &  & 50.00\% & 50.17\textpm1.64\% \\
 &  & 99.00\% & 98.73\textpm0.42\% \\ \cmidrule{2-4} 
 & \multirow{3}{*}{Adversarial} & 1.00\% & 1.05\textpm0.29\% \\
 &  & 50.00\% & 49.87\textpm3.58\% \\
 &  & 99.00\% & 99.00\textpm0.47\% \\ \cmidrule{2-4} 
 & \multirow{3}{*}{ReLU} & 1.00\% & 1.06\textpm0.67\% \\
 &  & 50.00\% & 50.02\textpm1.85\% \\
 &  & 99.00\% & 93.99\textpm3.51\% \\ \bottomrule
\end{tabular}
\label{tab:calibrationResultsMnistStrong}
\end{table}

\begin{table}[]
\caption{Expected vs true quantile for MNIST balanced with 5-fold cross validation.}
\centering
\begin{tabular}{llll}
\toprule
\textbf{Architecture} & \textbf{Training} & \textbf{Expected Quantile} & \textbf{True Quantile} \\ \midrule
\multirow{9}{*}{A} & \multirow{3}{*}{Standard} & 1.00\% & 1.30\textpm0.79\% \\
 &  & 50.00\% & 49.98\textpm3.10\% \\
 &  & 99.00\% & 93.99\textpm2.59\% \\ \cmidrule{2-4} 
 & \multirow{3}{*}{Adversarial} & 1.00\% & 0.97\textpm0.40\% \\
 &  & 50.00\% & 50.12\textpm1.14\% \\
 &  & 99.00\% & 90.44\textpm1.90\% \\ \cmidrule{2-4} 
 & \multirow{3}{*}{ReLU} & 1.00\% & 1.02\textpm0.31\% \\
 &  & 50.00\% & 50.02\textpm1.05\% \\
 &  & 99.00\% & 95.10\textpm2.82\% \\ \midrule
\multirow{9}{*}{B} & \multirow{3}{*}{Standard} & 1.00\% & 1.03\textpm0.36\% \\
 &  & 50.00\% & 49.98\textpm0.70\% \\
 &  & 99.00\% & 98.88\textpm0.45\% \\ \cmidrule{2-4} 
 & \multirow{3}{*}{Adversarial} & 1.00\% & 1.17\textpm0.97\% \\
 &  & 50.00\% & 50.17\textpm4.54\% \\
 &  & 99.00\% & 98.69\textpm0.59\% \\ \cmidrule{2-4} 
 & \multirow{3}{*}{ReLU} & 1.00\% & 1.04\textpm0.49\% \\
 &  & 50.00\% & 50.34\textpm2.49\% \\
 &  & 99.00\% & 98.73\textpm0.53\% \\ \midrule
\multirow{9}{*}{C} & \multirow{3}{*}{Standard} & 1.00\% & 1.07\textpm0.33\% \\
 &  & 50.00\% & 49.98\textpm0.91\% \\
 &  & 99.00\% & 98.88\textpm0.55\% \\ \cmidrule{2-4} 
 & \multirow{3}{*}{Adversarial} & 1.00\% & 1.10\textpm0.37\% \\
 &  & 50.00\% & 50.12\textpm4.15\% \\
 &  & 99.00\% & 99.00\textpm0.35\% \\ \cmidrule{2-4} 
 & \multirow{3}{*}{ReLU} & 1.00\% & 1.06\textpm0.67\% \\
 &  & 50.00\% & 50.12\textpm2.67\% \\
 &  & 99.00\% & 98.62\textpm0.50\% \\
\bottomrule
\end{tabular}
\label{tab:calibrationResultsMnistBalanced}
\end{table}

\begin{table}[]
\caption{Expected vs true quantile for CIFAR10 strong with 5-fold cross validation.}
\centering
\begin{tabular}{llll}
\toprule
\textbf{Architecture} & \textbf{Training} & \textbf{Expected Quantile} & \textbf{True Quantile} \\ \midrule
\multirow{9}{*}{A} & \multirow{3}{*}{Standard} & 1.00\% & 1.09\textpm0.86\% \\
 &  & 50.00\% & 50.09\textpm1.84\% \\
 &  & 99.00\% & 98.82\textpm0.63\% \\ \cmidrule{2-4} 
 & \multirow{3}{*}{Adversarial} & 1.00\% & 1.05\textpm0.23\% \\
 &  & 50.00\% & 49.86\textpm3.59\% \\
 &  & 99.00\% & 98.90\textpm0.62\% \\ \cmidrule{2-4} 
 & \multirow{3}{*}{ReLU} & 1.00\% & 0.97\textpm0.41\% \\
 &  & 50.00\% & 49.93\textpm3.42\% \\
 &  & 99.00\% & 97.66\textpm1.35\% \\ \midrule
\multirow{9}{*}{B} & \multirow{3}{*}{Standard} & 1.00\% & 0.98\textpm0.18\% \\
 &  & 50.00\% & 49.91\textpm1.18\% \\
 &  & 99.00\% & 98.84\textpm0.56\% \\ \cmidrule{2-4} 
 & \multirow{3}{*}{Adversarial} & 1.00\% & 0.91\textpm0.48\% \\
 &  & 50.00\% & 50.00\textpm3.58\% \\
 &  & 99.00\% & 98.69\textpm0.72\% \\ \cmidrule{2-4} 
 & \multirow{3}{*}{ReLU} & 1.00\% & 1.10\textpm0.72\% \\
 &  & 50.00\% & 49.98\textpm2.21\% \\
 &  & 99.00\% & 98.85\textpm0.61\% \\ \midrule
\multirow{9}{*}{C} & \multirow{3}{*}{Standard} & 1.00\% & 0.93\textpm0.60\% \\
 &  & 50.00\% & 50.00\textpm1.86\% \\
 &  & 99.00\% & 98.71\textpm0.71\% \\ \cmidrule{2-4} 
 & \multirow{3}{*}{Adversarial} & 1.00\% & 1.09\textpm0.17\% \\
 &  & 50.00\% & 50.14\textpm2.63\% \\
 &  & 99.00\% & 98.27\textpm0.81\% \\ \cmidrule{2-4} 
 & \multirow{3}{*}{ReLU} & 1.00\% & 1.01\textpm0.62\% \\
 &  & 50.00\% & 50.02\textpm2.09\% \\
 &  & 99.00\% & 96.17\textpm2.40\% \\
\bottomrule
\end{tabular}
\label{tab:calibrationResultsCifar10Strong}
\end{table}

\begin{table}[]
\caption{Expected vs true quantile for CIFAR10 balanced with 5-fold cross validation.}
\centering
\begin{tabular}{llll}
\toprule
\textbf{Architecture} & \textbf{Training} & \textbf{Expected Quantile} & \textbf{True Quantile} \\ \midrule
\multirow{9}{*}{A} & \multirow{3}{*}{Standard} & 1.00\% & 0.95\textpm0.61\% \\
 &  & 50.00\% & 50.32\textpm2.38\% \\
 &  & 99.00\% & 98.87\textpm0.59\% \\ \cmidrule{2-4} 
 & \multirow{3}{*}{Adversarial} & 1.00\% & 1.05\textpm0.23\% \\
 &  & 50.00\% & 50.23\textpm2.65\% \\
 &  & 99.00\% & 98.81\textpm0.96\% \\ \cmidrule{2-4} 
 & \multirow{3}{*}{ReLU} & 1.00\% & 4.14\textpm5.32\% \\
 &  & 50.00\% & 50.37\textpm1.02\% \\
 &  & 99.00\% & 94.62\textpm2.87\% \\ \midrule
\multirow{9}{*}{B} & \multirow{3}{*}{Standard} & 1.00\% & 1.07\textpm0.46\% \\
 &  & 50.00\% & 49.91\textpm2.78\% \\
 &  & 99.00\% & 98.93\textpm0.73\% \\ \cmidrule{2-4} 
 & \multirow{3}{*}{Adversarial} & 1.00\% & 1.13\textpm0.57\% \\
 &  & 50.00\% & 50.18\textpm2.05\% \\
 &  & 99.00\% & 98.82\textpm0.71\% \\ \cmidrule{2-4} 
 & \multirow{3}{*}{ReLU} & 1.00\% & 1.23\textpm0.38\% \\
 &  & 50.00\% & 50.11\textpm0.38\% \\
 &  & 99.00\% & 98.77\textpm0.51\% \\ \midrule
\multirow{9}{*}{C} & \multirow{3}{*}{Standard} & 1.00\% & 0.98\textpm0.50\% \\
 &  & 50.00\% & 50.09\textpm2.21\% \\
 &  & 99.00\% & 98.85\textpm0.43\% \\ \cmidrule{2-4} 
 & \multirow{3}{*}{Adversarial} & 1.00\% & 1.09\textpm0.26\% \\
 &  & 50.00\% & 49.96\textpm2.72\% \\
 &  & 99.00\% & 98.86\textpm0.32\% \\ \cmidrule{2-4} 
 & \multirow{3}{*}{ReLU} & 1.00\% & 1.01\textpm0.36\% \\
 &  & 50.00\% & 49.93\textpm1.60\% \\
 &  & 99.00\% & 97.93\textpm0.63\% \\
\bottomrule
\end{tabular}
\label{tab:calibrationResultsCifar10Balanced}
\end{table}

\newlength{\calibrationFigureSize}
\setlength{\calibrationFigureSize}{0.32\textwidth}

\begin{figure}
\centering
\subfloat[MNIST A Standard Strong]{
    \begin{adjustbox}{width=\calibrationFigureSize}
    \includegraphics{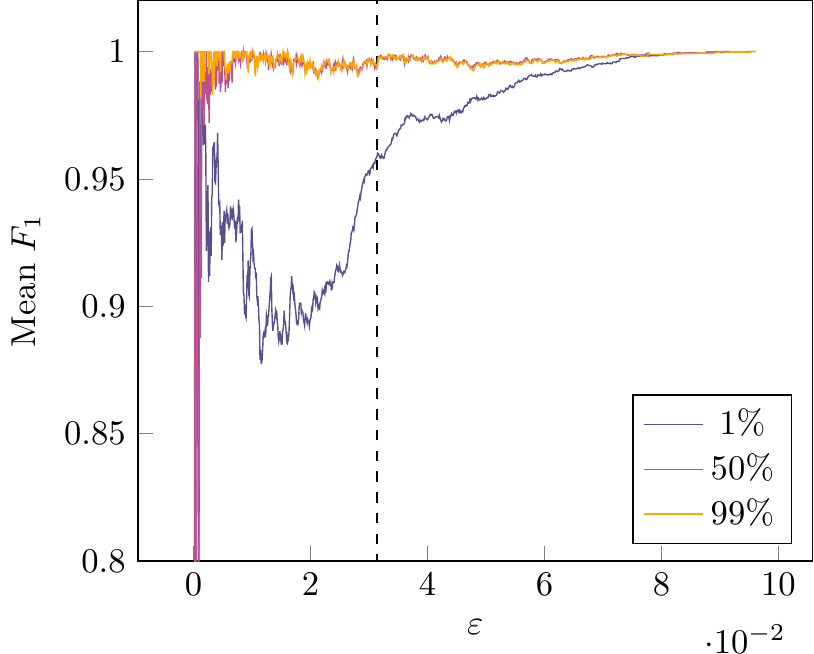}
    \end{adjustbox}
}
\hfill
\subfloat[MNIST A Standard Balanced]{
    \begin{adjustbox}{width=\calibrationFigureSize}
    \includegraphics{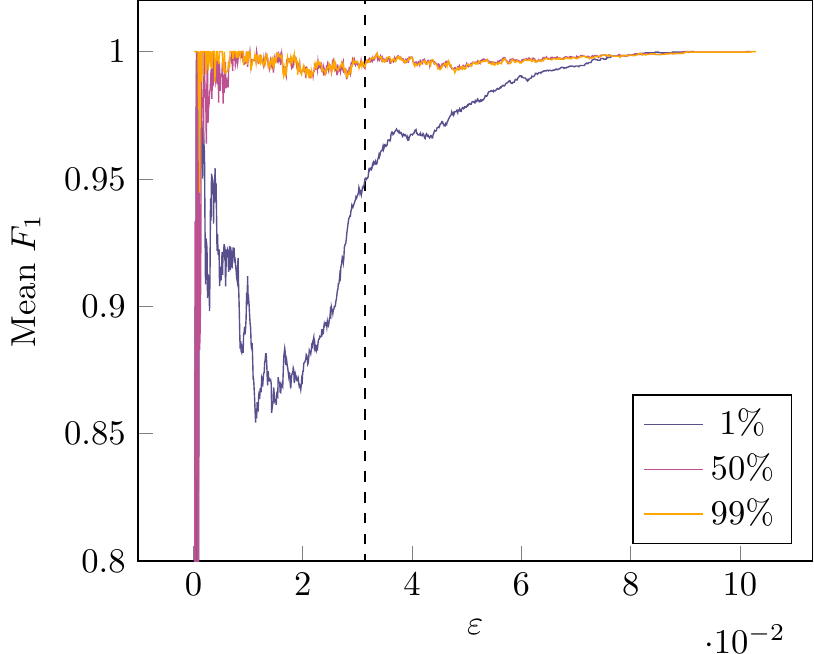}
    \end{adjustbox}
}
\hfill
\subfloat[MNIST A Adversarial Strong]{
    \begin{adjustbox}{width=\calibrationFigureSize}
    \includegraphics{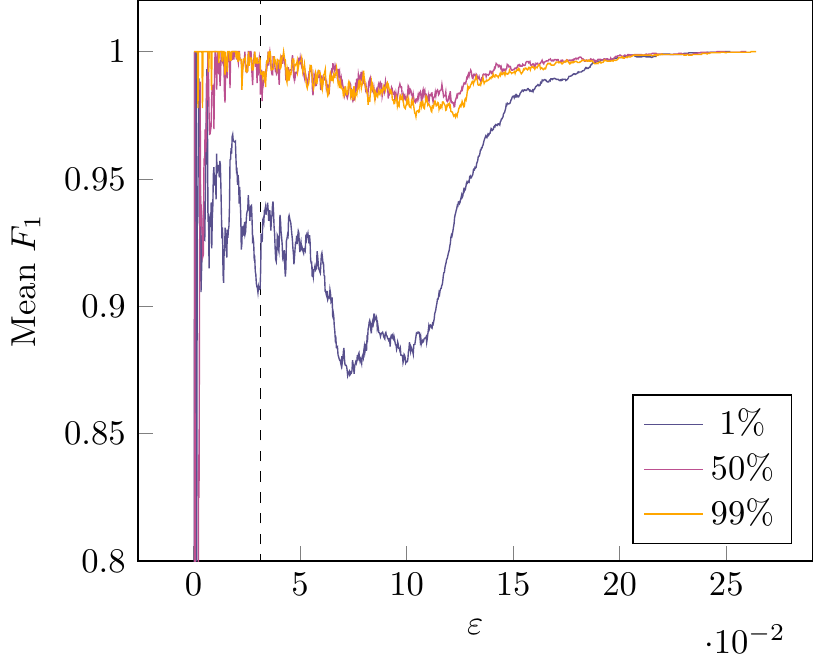}
    \end{adjustbox}
}

\subfloat[MNIST A Adversarial Balanced]{
    \begin{adjustbox}{width=\calibrationFigureSize}
    \includegraphics{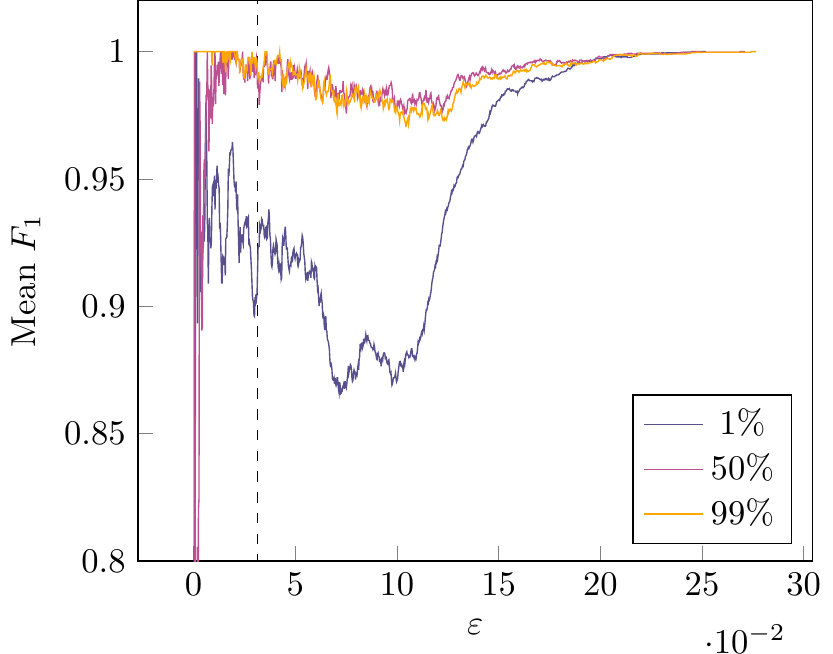}
    \end{adjustbox}
}
\hfill
\subfloat[MNIST A ReLU Strong]{
    \begin{adjustbox}{width=\calibrationFigureSize}
    \includegraphics{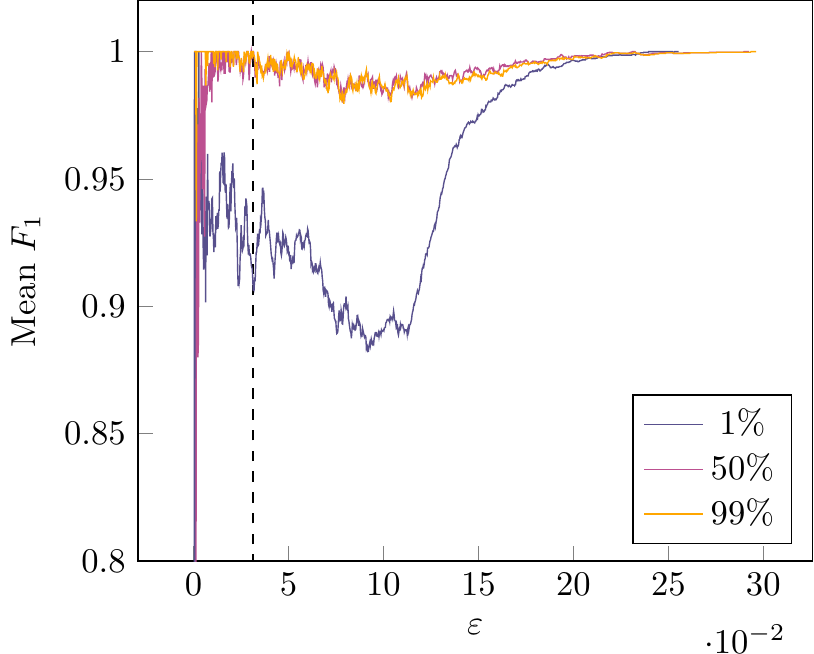}
    \end{adjustbox}
}
\hfill
\subfloat[MNIST A ReLU Balanced]{
    \begin{adjustbox}{width=\calibrationFigureSize}
    \includegraphics{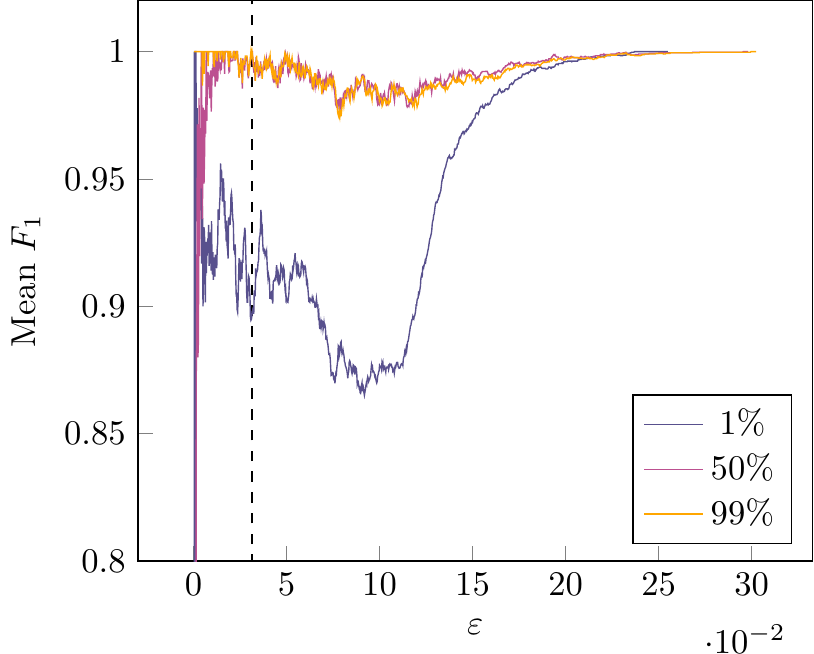}
    \end{adjustbox}
}

\caption{$F_1$ scores in relation to $\varepsilon$ for MNIST A for each considered percentile. For ease of visualization, we set the graph cutoff at $F_1 = 0.8$. We also mark 8/255 (a common choice for $\varepsilon$) with a dotted line.} 
   
\label{fig:f1MnistA}
\end{figure}

\begin{figure}
\centering
\subfloat[MNIST B Standard Strong]{
    \begin{adjustbox}{width=\calibrationFigureSize}
    \includegraphics{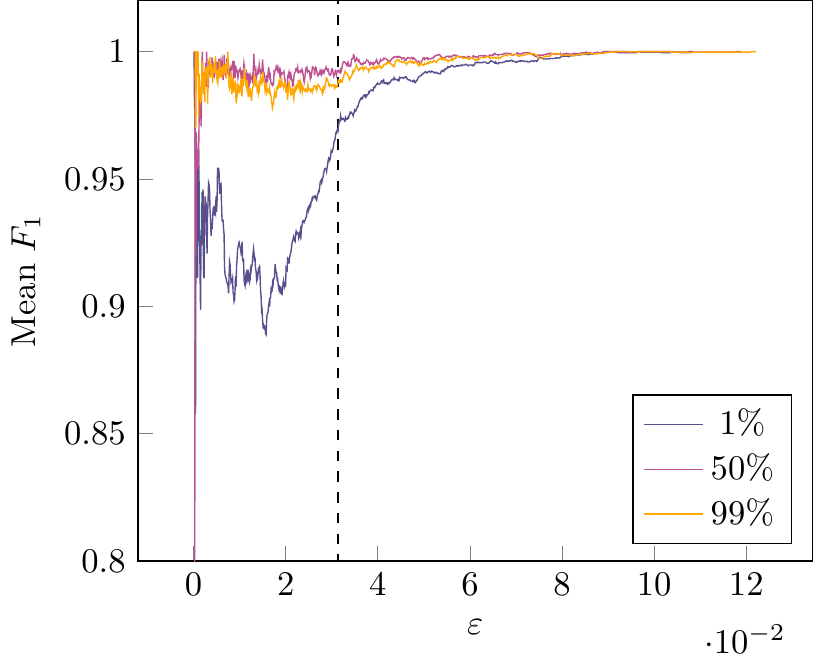}
    \end{adjustbox}
}
\hfill
\subfloat[MNIST B Standard Balanced]{
    \begin{adjustbox}{width=\calibrationFigureSize}
    \includegraphics{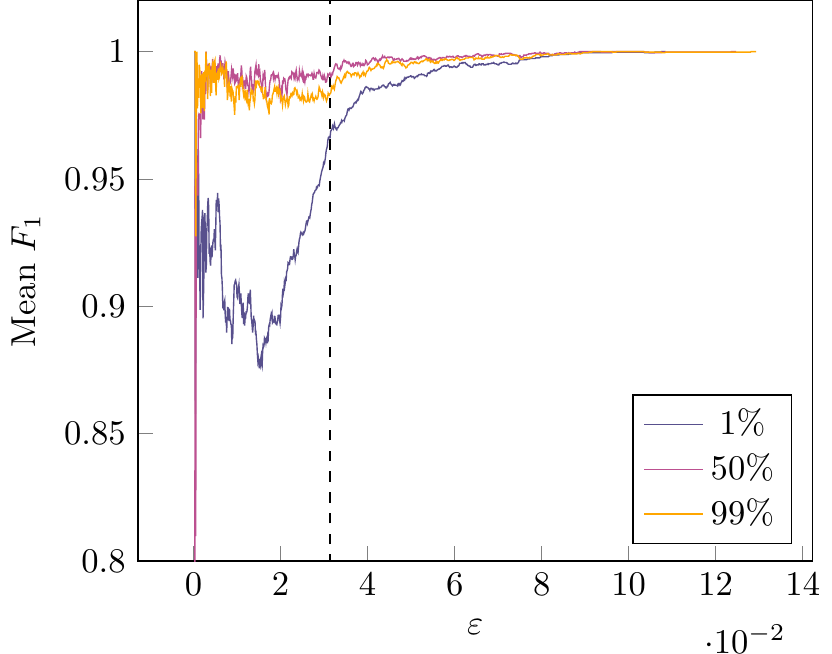}
    \end{adjustbox}
}
\hfill
\subfloat[MNIST B Adversarial Strong]{
    \begin{adjustbox}{width=\calibrationFigureSize}
    \includegraphics{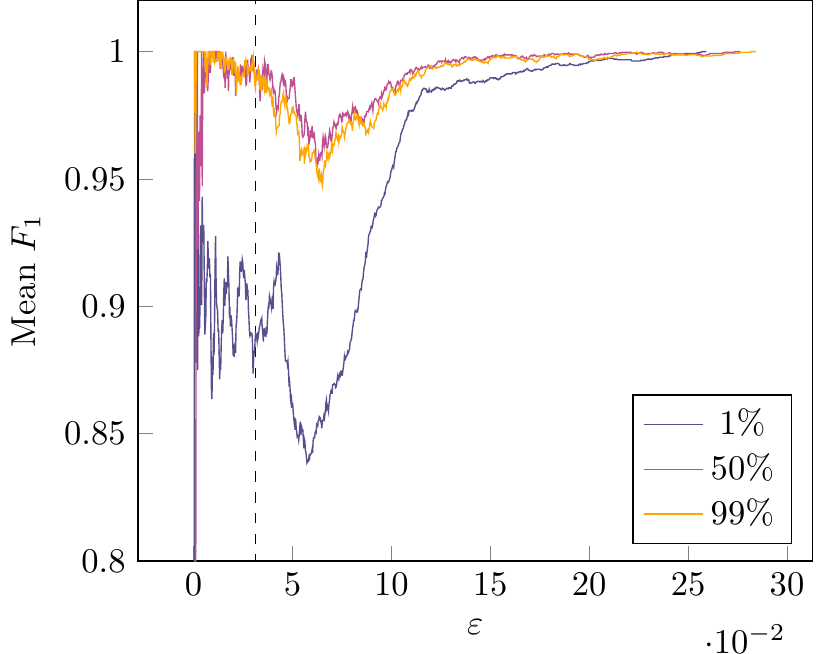}
    \end{adjustbox}
}

\subfloat[MNIST B Adversarial Balanced]{
    \begin{adjustbox}{width=\calibrationFigureSize}
    \includegraphics{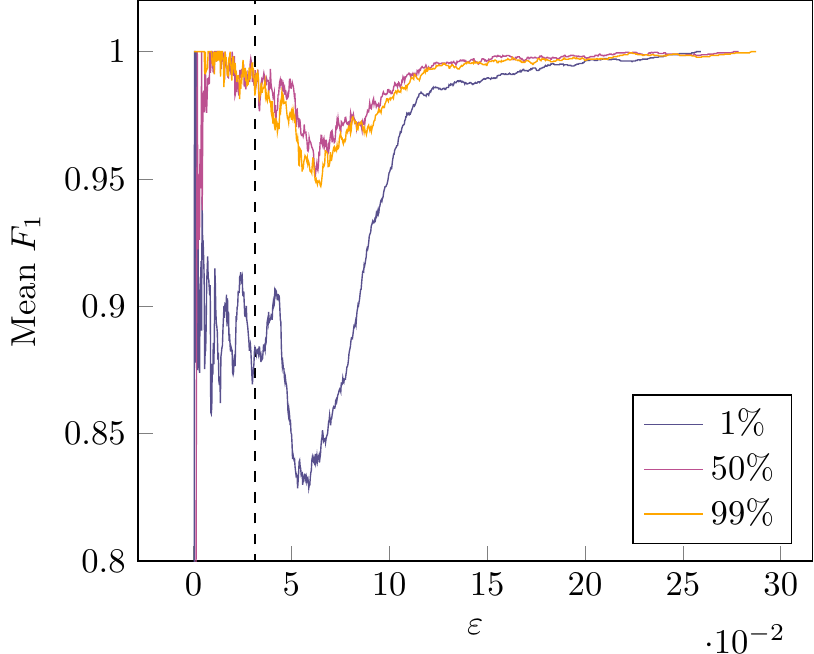}
    \end{adjustbox}
}
\hfill
\subfloat[MNIST B ReLU Strong]{
    \begin{adjustbox}{width=\calibrationFigureSize}
    \includegraphics{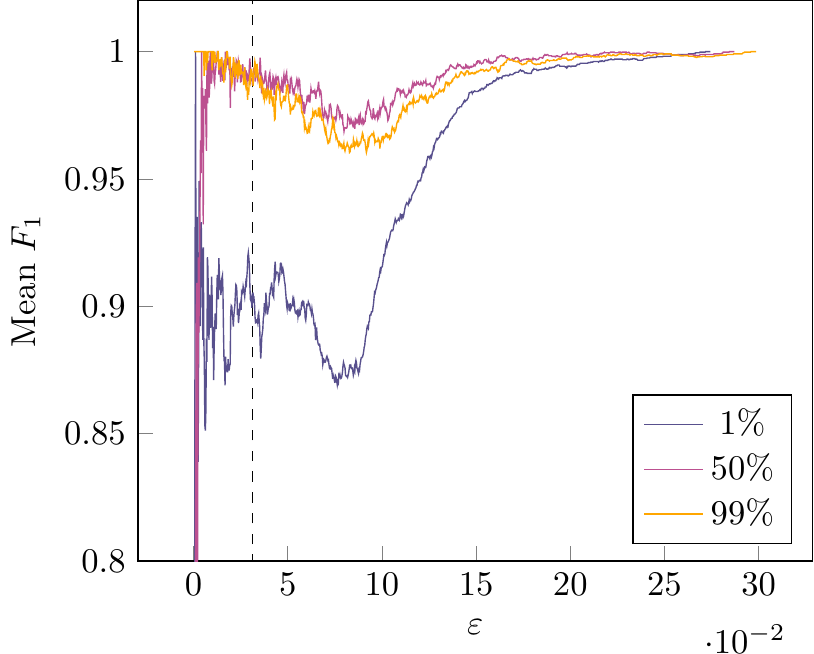}
    \end{adjustbox}
}
\hfill
\subfloat[MNIST B ReLU Balanced]{
    \begin{adjustbox}{width=\calibrationFigureSize}
    \includegraphics{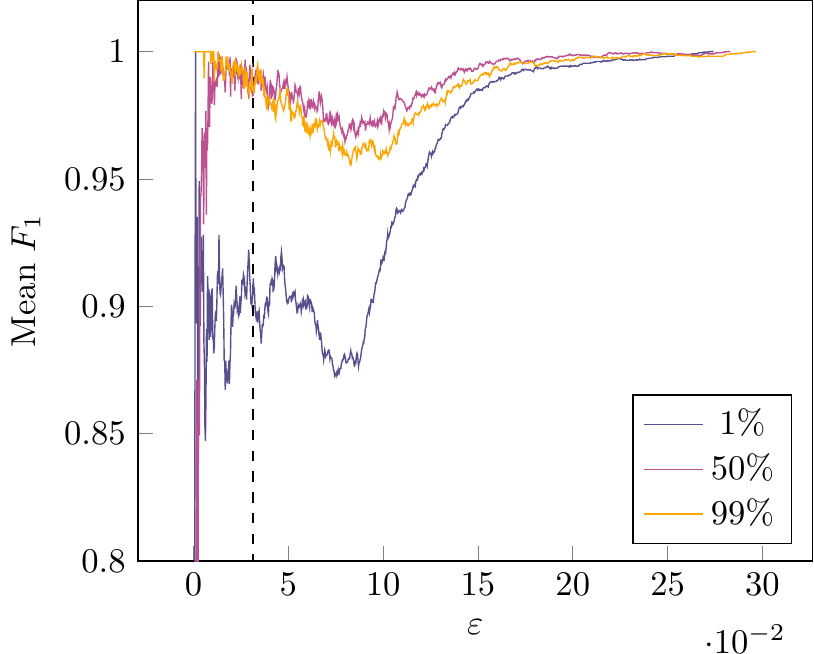}
    \end{adjustbox}
}

\caption{$F_1$ scores in relation to $\varepsilon$ for MNIST B for each considered percentile. For ease of visualization, we set the graph cutoff at $F_1 = 0.8$. We also mark 8/255 (a common choice for $\varepsilon$) with a dotted line.} 
   
\label{fig:f1MnistB}
\end{figure}

\begin{figure}
\centering
\subfloat[MNIST C Standard Strong]{
    \begin{adjustbox}{width=\calibrationFigureSize}
    \includegraphics{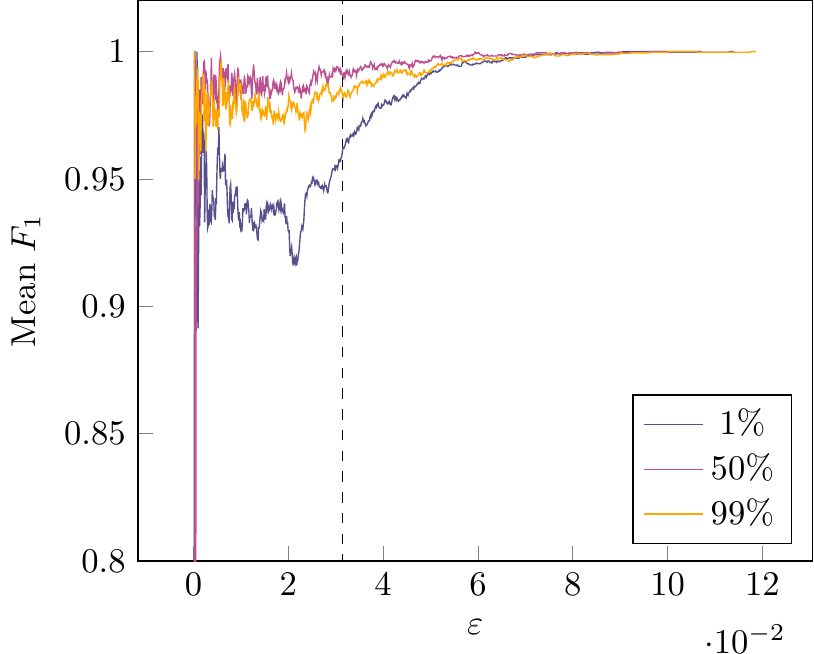}
    \end{adjustbox}
}
\hfill
\subfloat[MNIST C Standard Balanced]{
    \begin{adjustbox}{width=\calibrationFigureSize}
    \includegraphics{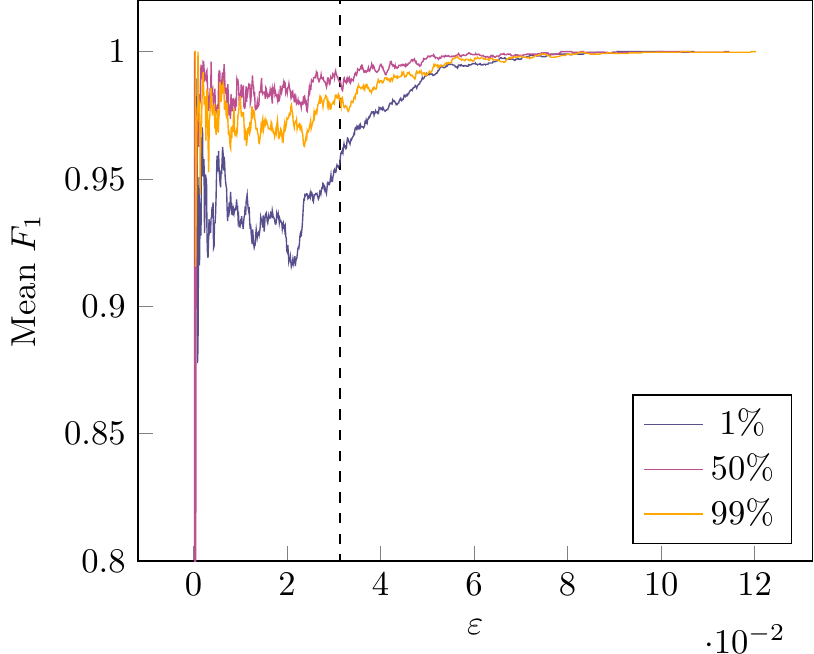}
    \end{adjustbox}
}
\hfill
\subfloat[MNIST C Adversarial Strong]{
    \begin{adjustbox}{width=\calibrationFigureSize}
    \includegraphics{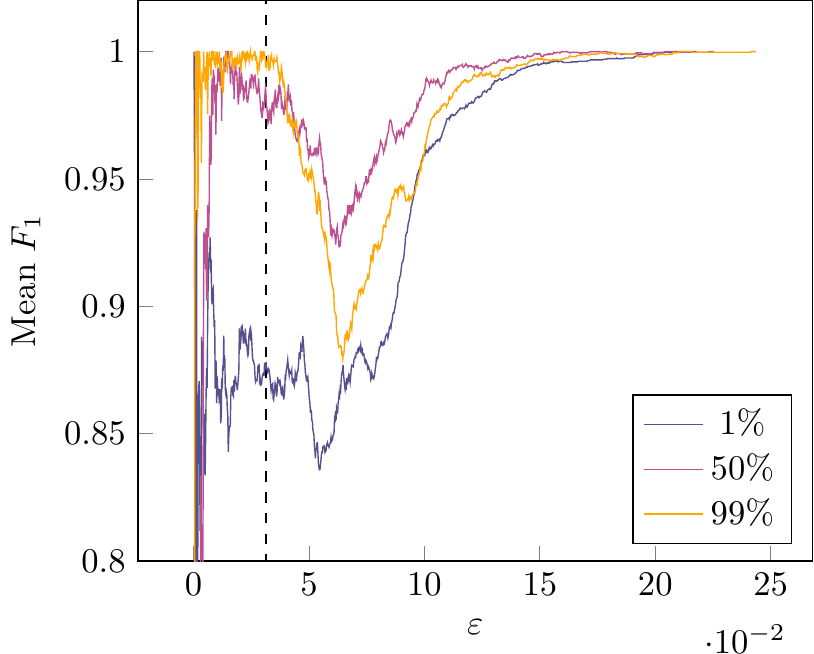}
    \end{adjustbox}
}

\subfloat[MNIST C Adversarial Balanced]{
    \begin{adjustbox}{width=\calibrationFigureSize}
    \includegraphics{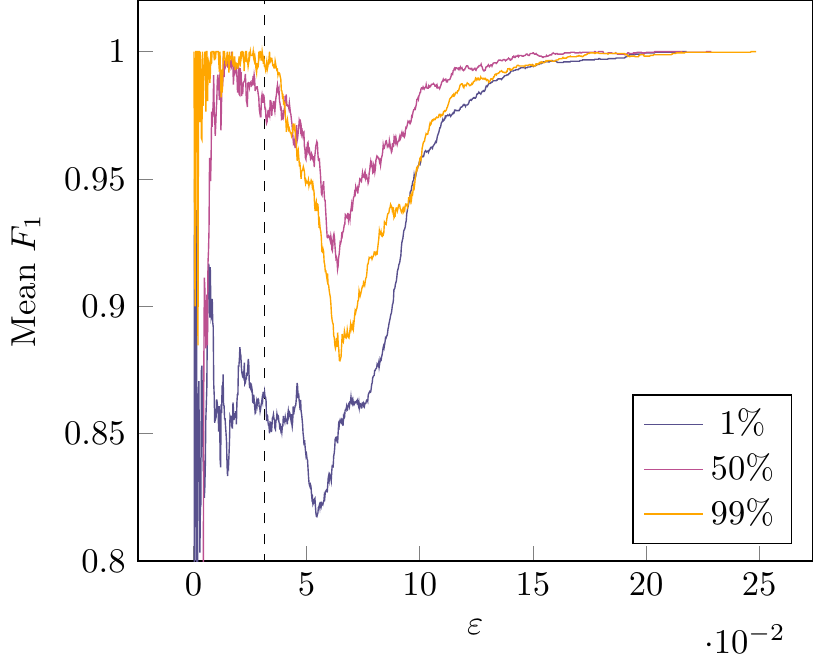}
    \end{adjustbox}
}
\hfill
\subfloat[MNIST C ReLU Strong]{
    \begin{adjustbox}{width=\calibrationFigureSize}
    \includegraphics{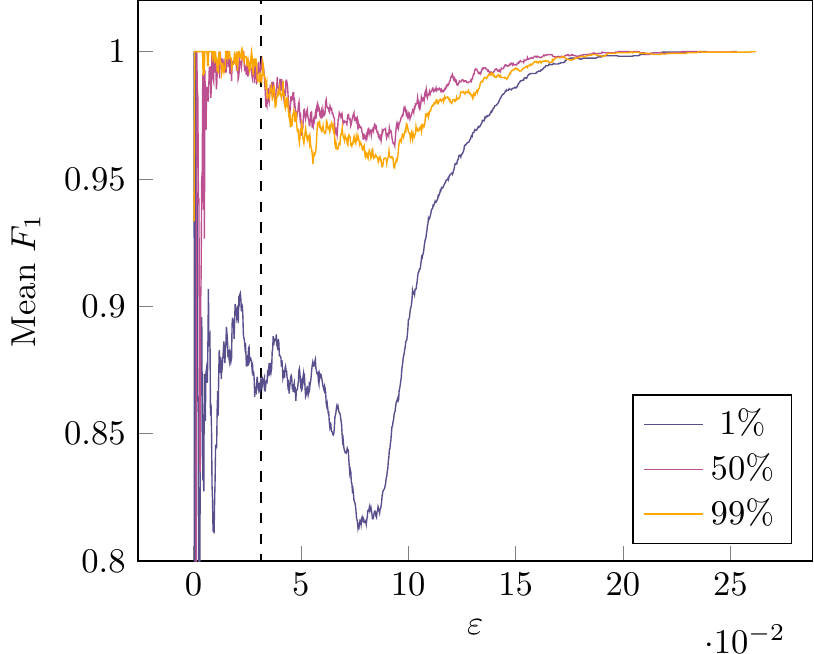}
    \end{adjustbox}
}
\hfill
\subfloat[MNIST C ReLU Balanced]{
    \begin{adjustbox}{width=\calibrationFigureSize}
    \includegraphics{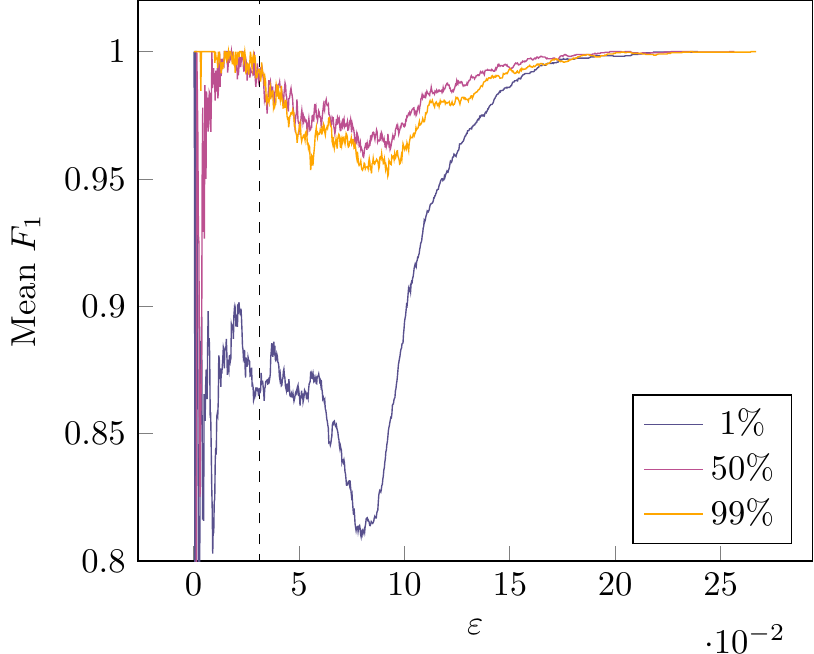}
    \end{adjustbox}
}

\caption{$F_1$ scores in relation to $\varepsilon$ for MNIST C for each considered percentile. For ease of visualization, we set the graph cutoff at $F_1 = 0.8$. We also mark 8/255 (a common choice for $\varepsilon$) with a dotted line.} 
   
\label{fig:f1MnistC}
\end{figure}

\begin{figure}
\centering
\subfloat[CIFAR10 A Standard Strong]{
    \begin{adjustbox}{width=\calibrationFigureSize}
    \includegraphics{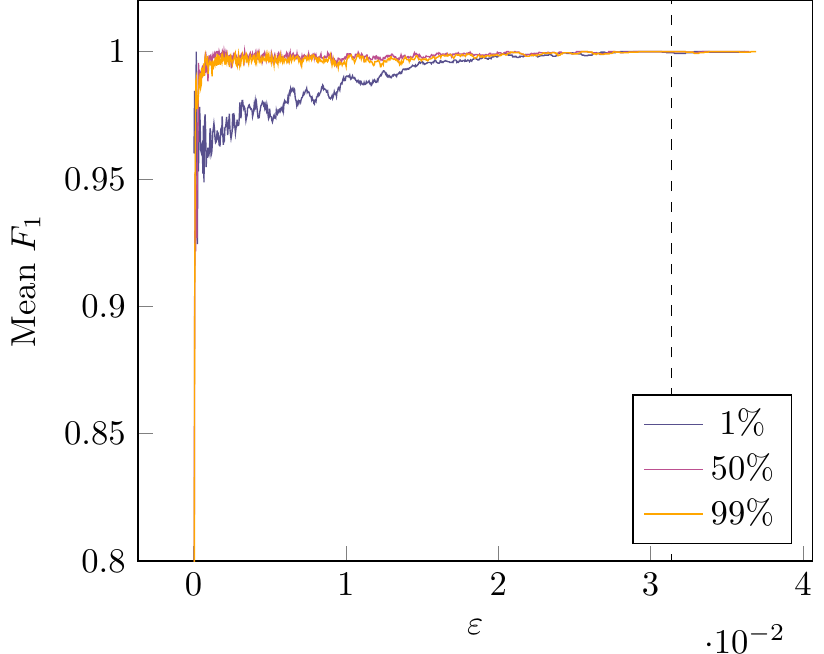}
    \end{adjustbox}
}
\hfill
\subfloat[CIFAR10 A Standard Balanced]{
    \begin{adjustbox}{width=\calibrationFigureSize}
    \includegraphics{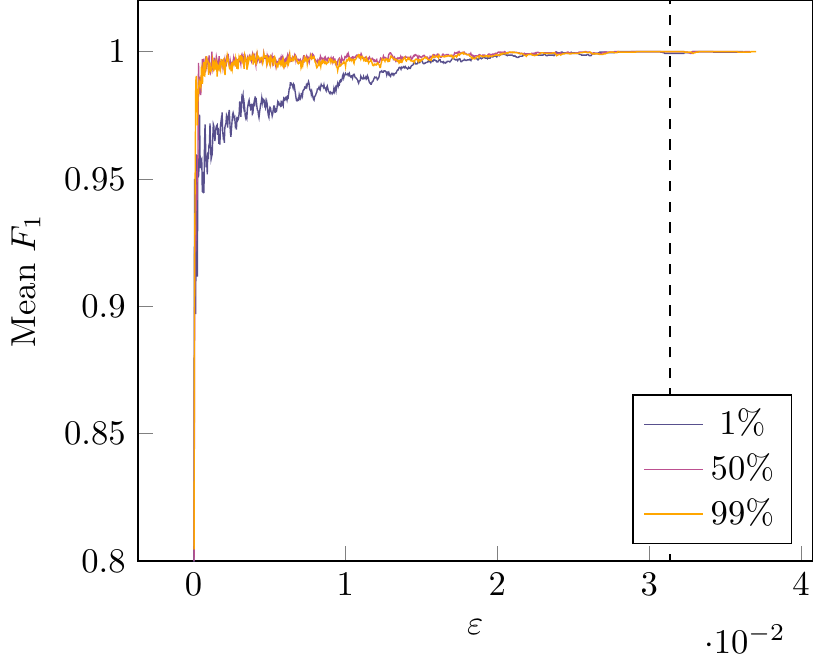}
    \end{adjustbox}
}
\hfill
\subfloat[CIFAR10 A Adversarial Strong]{
    \begin{adjustbox}{width=\calibrationFigureSize}
    \includegraphics{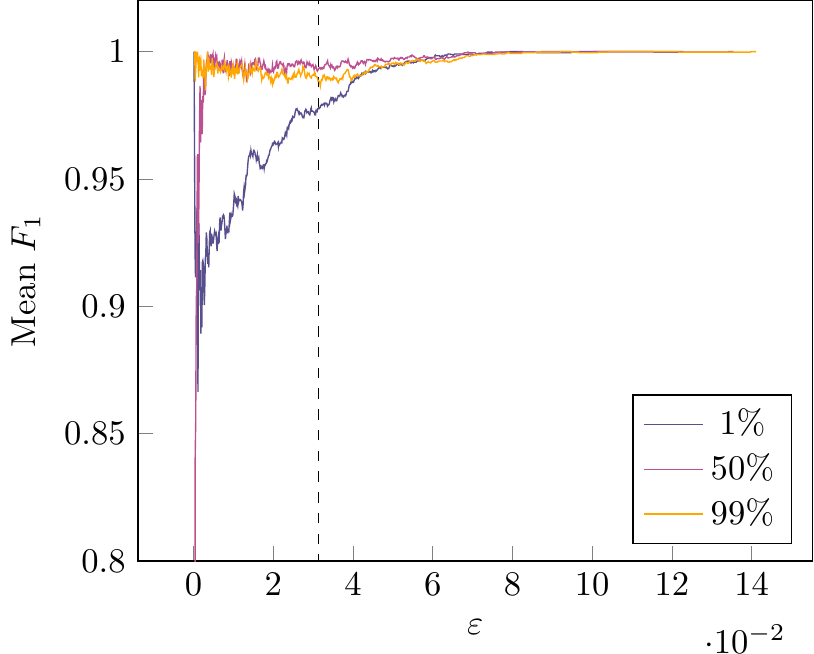}
    \end{adjustbox}
}

\subfloat[CIFAR10 A Adversarial Balanced]{
    \begin{adjustbox}{width=\calibrationFigureSize}
    \includegraphics{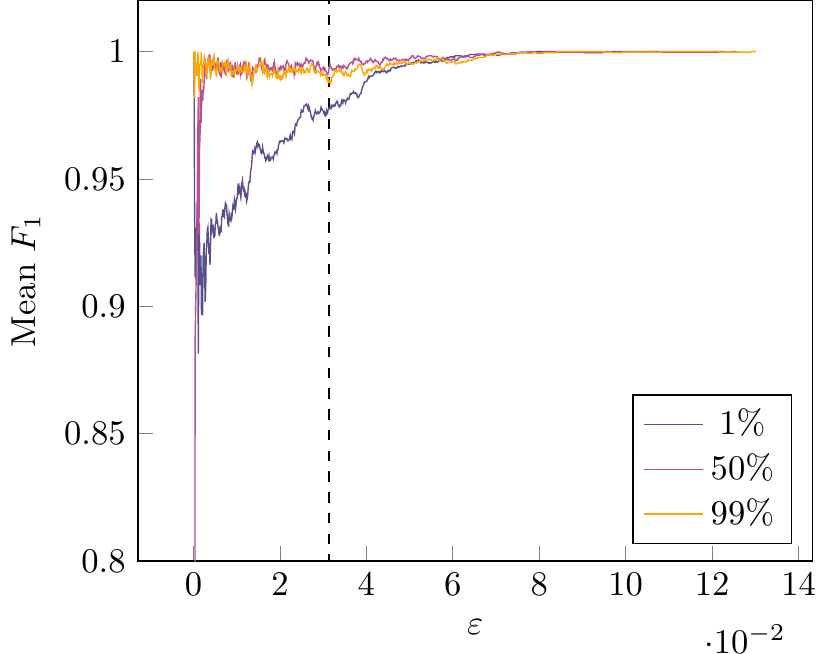}
    \end{adjustbox}
}
\hfill
\subfloat[CIFAR10 A ReLU Strong]{
    \begin{adjustbox}{width=\calibrationFigureSize}
    \includegraphics{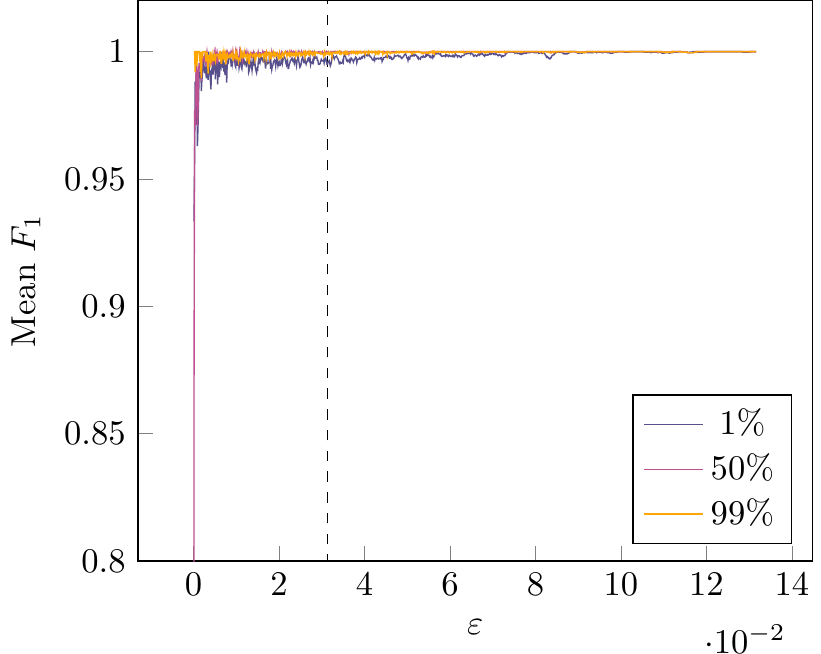}
    \end{adjustbox}
}
\hfill
\subfloat[CIFAR10 A ReLU Balanced]{
    \begin{adjustbox}{width=\calibrationFigureSize}
    \includegraphics{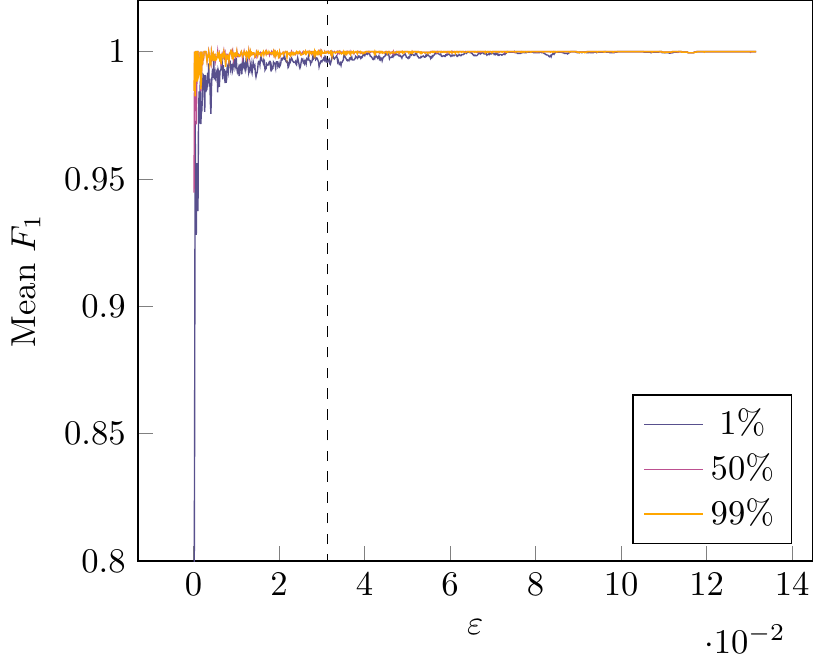}
    \end{adjustbox}
}

\caption{$F_1$ scores in relation to $\varepsilon$ for CIFAR10 A for each considered percentile. For ease of visualization, we set the graph cutoff at $F_1 = 0.8$. We also mark 8/255 (a common choice for $\varepsilon$) with a dotted line.} 
   
\label{fig:f1Cifar10A}
\end{figure}

\begin{figure}
\centering
\subfloat[CIFAR10 B Standard Strong]{
    \begin{adjustbox}{width=\calibrationFigureSize}
    \includegraphics{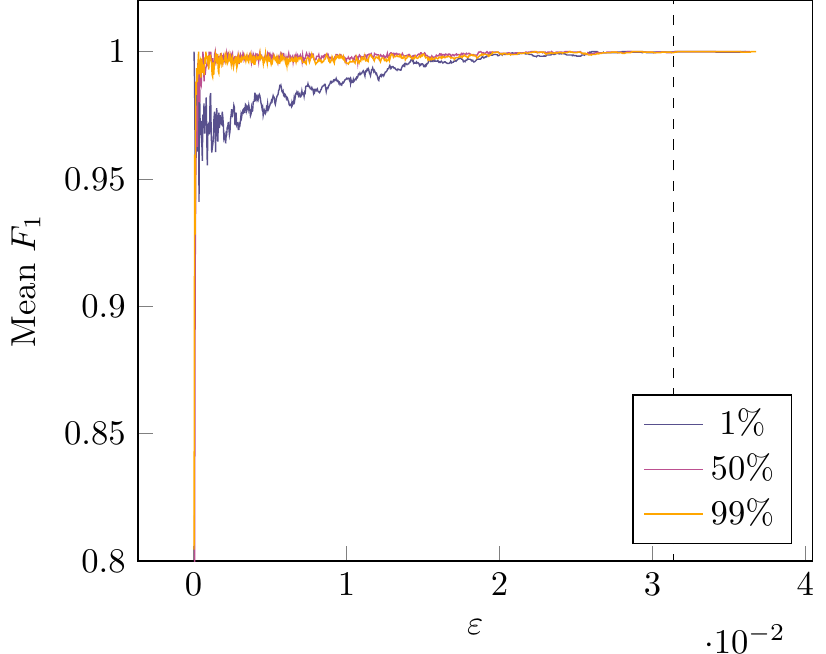}
    \end{adjustbox}
}
\hfill
\subfloat[CIFAR10 B Standard Balanced]{
    \begin{adjustbox}{width=\calibrationFigureSize}
    \includegraphics{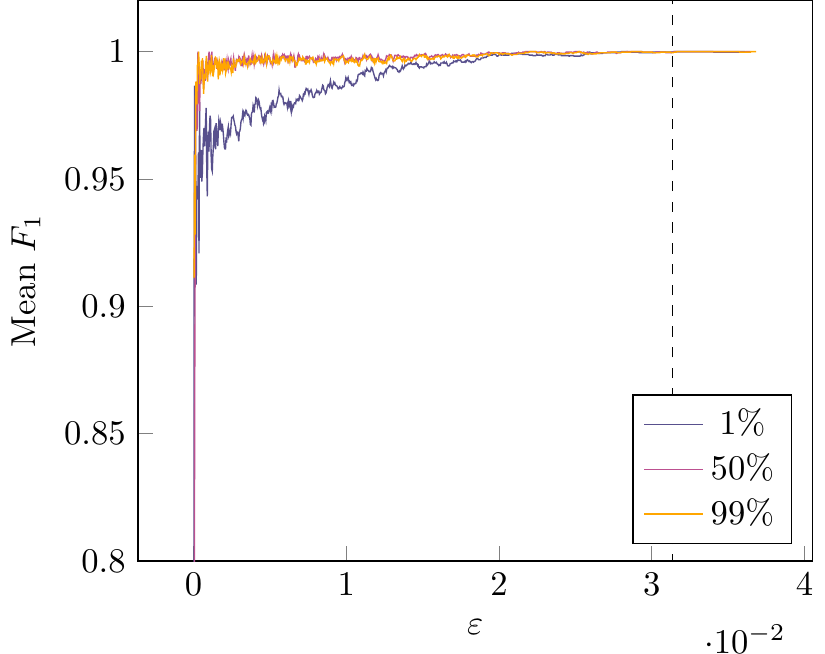}
    \end{adjustbox}
}
\hfill
\subfloat[CIFAR10 B Adversarial Strong]{
    \begin{adjustbox}{width=\calibrationFigureSize}
    \includegraphics{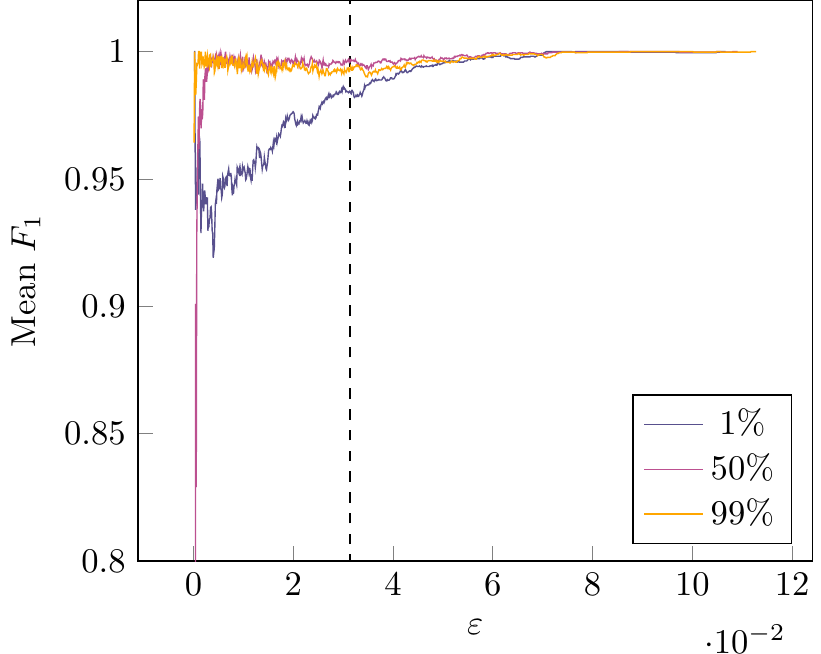}
    \end{adjustbox}
}

\subfloat[CIFAR10 B Adversarial Balanced]{
    \begin{adjustbox}{width=\calibrationFigureSize}
    \includegraphics{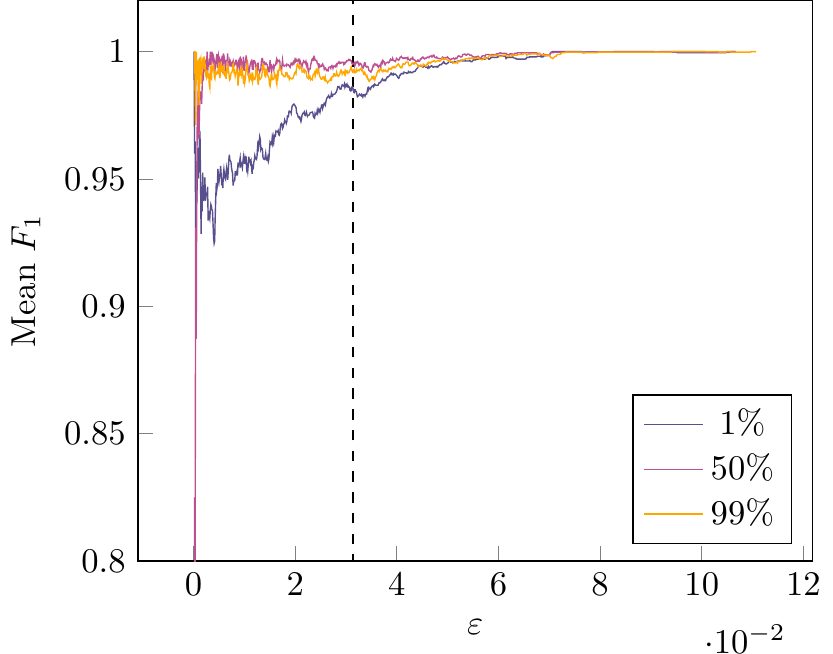}
    \end{adjustbox}
}
\hfill
\subfloat[CIFAR10 B ReLU Strong]{
    \begin{adjustbox}{width=\calibrationFigureSize}
    \includegraphics{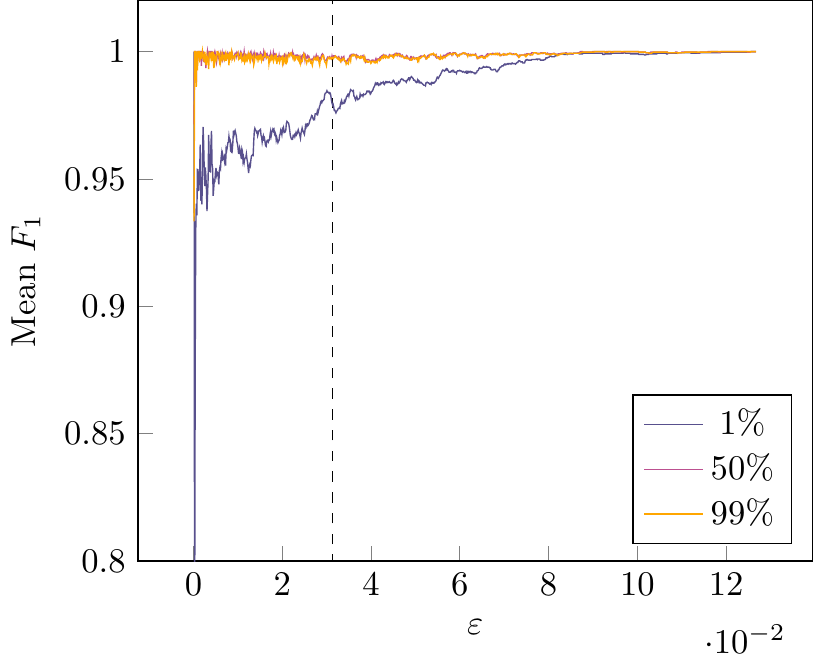}
    \end{adjustbox}
}
\hfill
\subfloat[CIFAR10 B ReLU Balanced]{
    \begin{adjustbox}{width=\calibrationFigureSize}
    \includegraphics{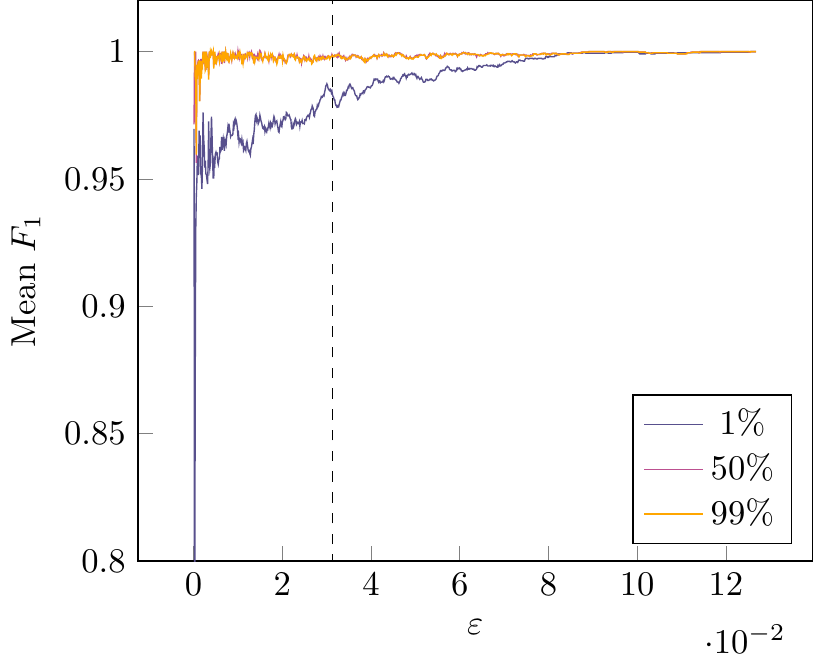}
    \end{adjustbox}
}

\caption{$F_1$ scores in relation to $\varepsilon$ for CIFAR10 B for each considered percentile. For ease of visualization, we set the graph cutoff at $F_1 = 0.8$. We also mark 8/255 (a common choice for $\varepsilon$) with a dotted line.} 
   
\label{fig:f1Cifar10B}
\end{figure}

\begin{figure}
\centering
\subfloat[CIFAR10 C Standard Strong]{
    \begin{adjustbox}{width=\calibrationFigureSize}
    \includegraphics{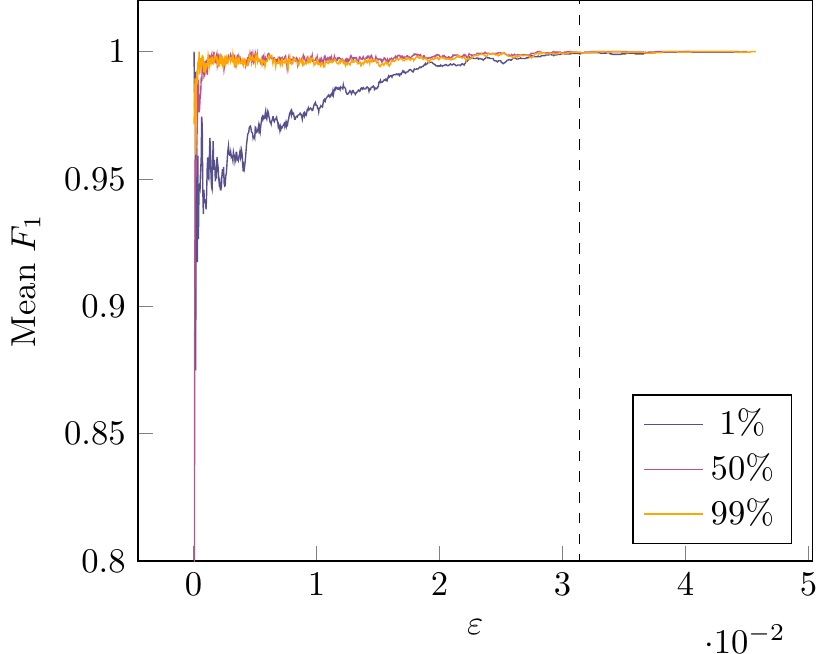}
    \end{adjustbox}
}
\hfill
\subfloat[CIFAR10 C Standard Balanced]{
    \begin{adjustbox}{width=\calibrationFigureSize}
    \includegraphics{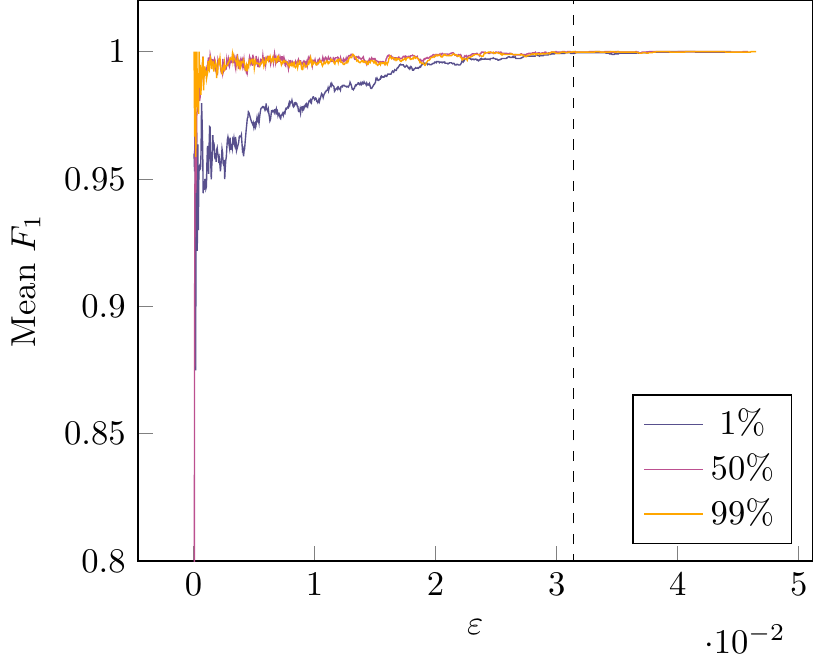}
    \end{adjustbox}
}
\hfill
\subfloat[CIFAR10 C Adversarial Strong]{
    \begin{adjustbox}{width=\calibrationFigureSize}
    \includegraphics{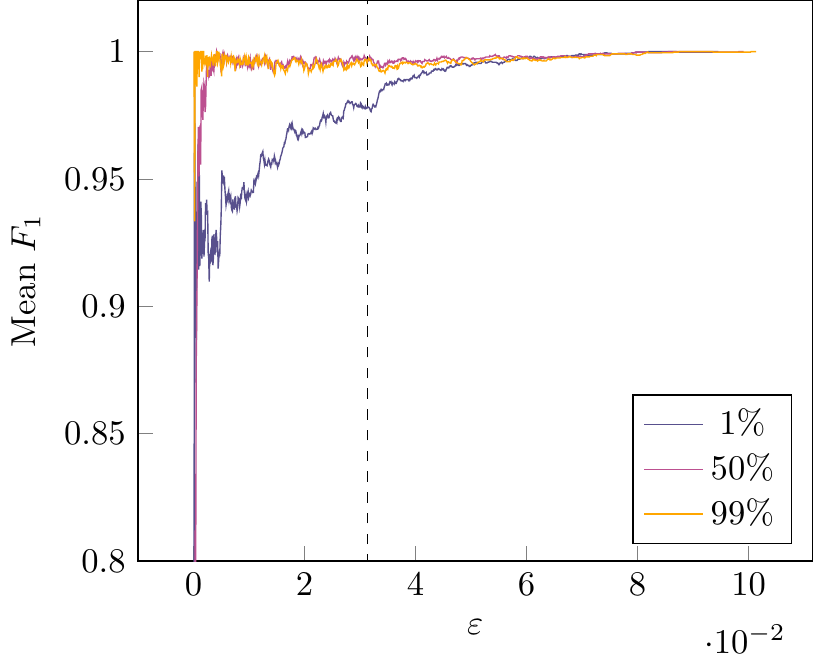}
    \end{adjustbox}
}

\subfloat[CIFAR10 C Adversarial Balanced]{
    \begin{adjustbox}{width=\calibrationFigureSize}
    \includegraphics{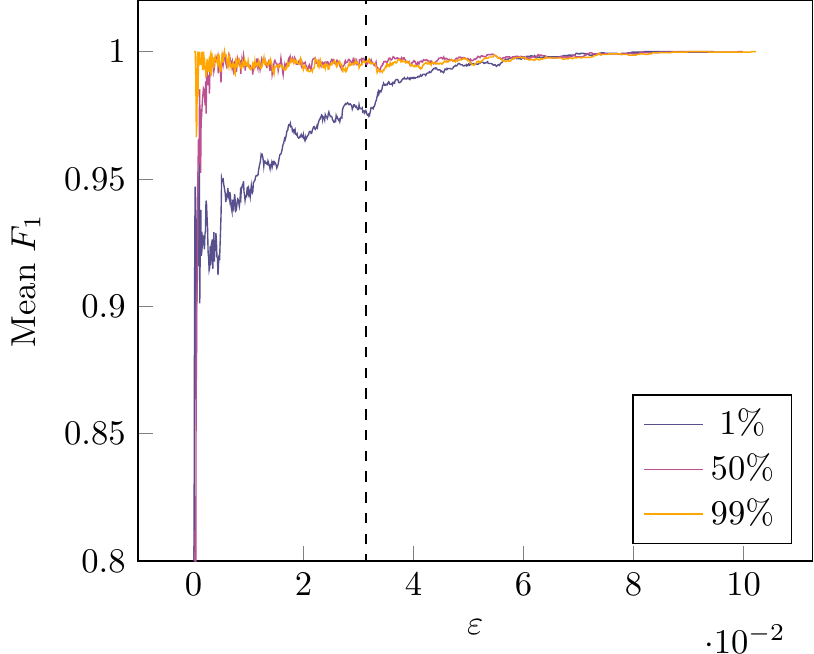}
    \end{adjustbox}
}
\hfill
\subfloat[CIFAR10 C ReLU Strong]{
    \begin{adjustbox}{width=\calibrationFigureSize}
    \includegraphics{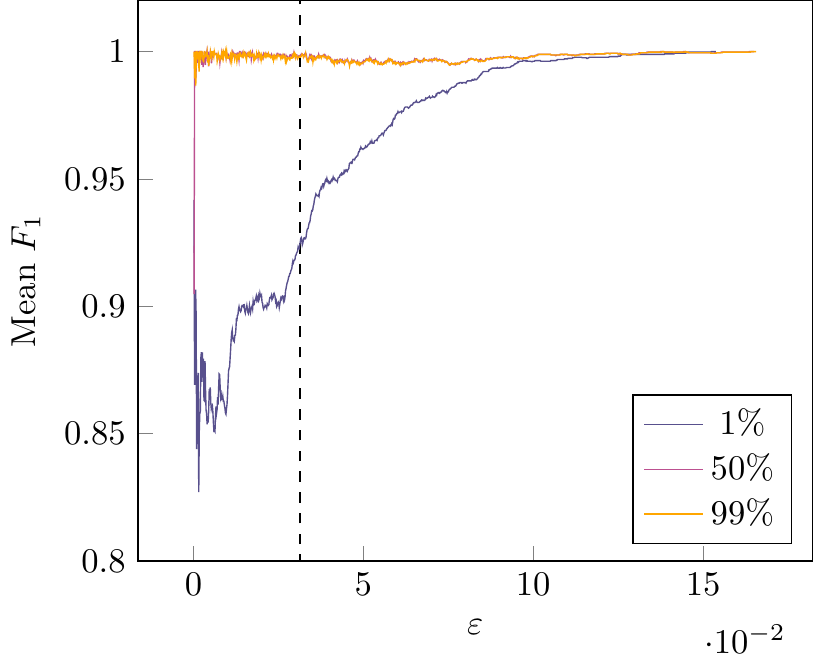}
    \end{adjustbox}
}
\hfill
\subfloat[CIFAR10 C ReLU Balanced]{
    \begin{adjustbox}{width=\calibrationFigureSize}
    \includegraphics{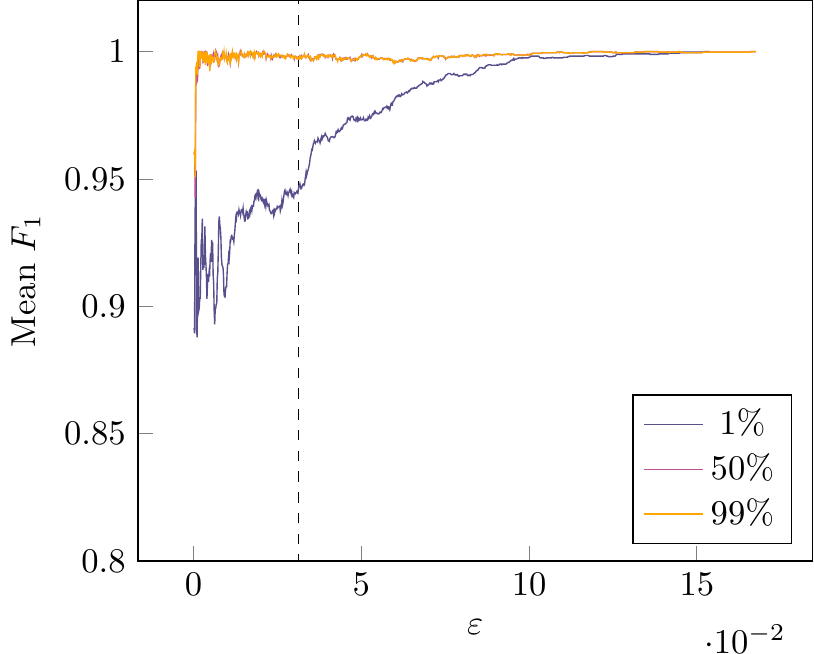}
    \end{adjustbox}
}

\caption{$F_1$ scores in relation to $\varepsilon$ for CIFAR10 C for each considered percentile. For ease of visualization, we set the graph cutoff at $F_1 = 0.8$. We also mark 8/255 (a common choice for $\varepsilon$) with a dotted line.} 
   
\label{fig:f1Cifar10C}
\end{figure}

\clearpage
\section{Additional Results}
\label{app:extraResults}

\Cref{tab:performanceMnistStrong,tab:performanceMnistBalanced,tab:performanceCifar10Strong,tab:performanceCifar10Balanced} detail the performance of the various attack sets on every combination, while \Cref{fig:scatterMnistA,fig:scatterMnistB,fig:scatterMnistC,fig:scatterCifar10A,fig:scatterCifar10B,fig:scatterCifar10C} showcase the relation between the true and estimated decision boundary distances.

\begin{table}[]
\caption{Performance of the strong attack set on MNIST.}
\centering
\begin{tabular}{llllll}
\midrule
\multicolumn{1}{c}{\textbf{Architecture}} & \multicolumn{1}{c}{\textbf{Training}} & \multicolumn{1}{c}{\textbf{Success Rate}} & \multicolumn{1}{c}{\textbf{Difference}} & \multicolumn{1}{c}{\textbf{\% Below 1/255}} & \multicolumn{1}{c}{$\mathbf{R^2}$} \\ \midrule
\multirow{3}{*}{MNIST A} & Standard & 100.00\% & 1.51\% & 98.16\% & 0.996 \\ 
 & Adversarial & 100.00\% & 2.48\% & 81.43\% & 0.994 \\ 
 & ReLU & 100.00\% & 2.14\% & 84.33\% & 0.995 \\ \midrule
\multirow{3}{*}{MNIST B} & Standard & 100.00\% & 3.38\% & 97.36\% & 0.995 \\ 
 & Adversarial & 100.00\% & 4.34\% & 75.09\% & 0.991 \\ 
 & ReLU & 100.00\% & 4.80\% & 68.02\% & 0.992 \\ \midrule
\multirow{3}{*}{MNIST C} & Standard & 100.00\% & 4.52\% & 96.92\% & 0.996 \\ 
 & Adversarial & 100.00\% & 8.76\% & 48.78\% & 0.981 \\ 
 & ReLU & 100.00\% & 4.84\% & 68.24\% & 0.988 \\
\bottomrule \end{tabular}
\label{tab:performanceMnistStrong}
\end{table}

\begin{table}[]
\caption{Performance of the balanced attack set on MNIST.}
\centering
\begin{tabular}{llllll}
\midrule
\multicolumn{1}{c}{\textbf{Architecture}} & \multicolumn{1}{c}{\textbf{Training}} & \multicolumn{1}{c}{\textbf{Success Rate}} & \multicolumn{1}{c}{\textbf{Difference}} & \multicolumn{1}{c}{\textbf{\% Below 1/255}} & \multicolumn{1}{c}{$\mathbf{R^2}$} \\ \midrule
\multirow{3}{*}{MNIST A} & Standard & 100.00\% & 1.68\% & 97.94\% & 0.995 \\ 
 & Adversarial & 100.00\% & 2.87\% & 77.64\% & 0.993 \\ 
 & ReLU & 100.00\% & 2.55\% & 80.86\% & 0.993 \\ \midrule
\multirow{3}{*}{MNIST B} & Standard & 100.00\% & 4.09\% & 96.55\% & 0.995 \\ 
 & Adversarial & 100.00\% & 4.90\% & 72.60\% & 0.988 \\ 
 & ReLU & 100.00\% & 5.53\% & 62.96\% & 0.989 \\ \midrule
\multirow{3}{*}{MNIST C} & Standard & 100.00\% & 5.43\% & 96.04\% & 0.995 \\ 
 & Adversarial & 100.00\% & 9.50\% & 48.43\% & 0.977 \\ 
 & ReLU & 100.00\% & 5.28\% & 66.96\% & 0.986 \\
\bottomrule \end{tabular}
\label{tab:performanceMnistBalanced}
\end{table}

\begin{table}[]
\caption{Performance of the strong attack set on CIFAR10.}
\centering
\begin{tabular}{llllll}
\midrule
\multicolumn{1}{c}{\textbf{Architecture}} & \multicolumn{1}{c}{\textbf{Training}} & \multicolumn{1}{c}{\textbf{Success Rate}} & \multicolumn{1}{c}{\textbf{Difference}} & \multicolumn{1}{c}{\textbf{\% Below 1/255}} & \multicolumn{1}{c}{$\mathbf{R^2}$} \\ \midrule
\multirow{3}{*}{CIFAR10 A} & Standard & 100.00\% & 1.62\% & 100.00\% & 0.999 \\ 
 & Adversarial & 100.00\% & 4.42\% & 95.88\% & 0.995 \\ 
 & ReLU & 100.00\% & 0.26\% & 100.00\% & 1.000 \\ \midrule
\multirow{3}{*}{CIFAR10 B} & Standard & 100.00\% & 1.44\% & 100.00\% & 0.999 \\ 
 & Adversarial & 100.00\% & 3.17\% & 97.69\% & 0.997 \\ 
 & ReLU & 100.00\% & 1.38\% & 98.81\% & 0.999 \\ \midrule
\multirow{3}{*}{CIFAR10 C} & Standard & 100.00\% & 2.11\% & 100.00\% & 0.999 \\ 
 & Adversarial & 100.00\% & 3.10\% & 97.14\% & 0.996 \\ 
 & ReLU & 100.00\% & 2.35\% & 96.12\% & 0.990 \\
\bottomrule \end{tabular}
\label{tab:performanceCifar10Strong}
\end{table}

\begin{table}[]
\caption{Performance of the balanced attack set on CIFAR10.}
\centering
\begin{tabular}{llllll}
\midrule
\multicolumn{1}{c}{\textbf{Architecture}} & \multicolumn{1}{c}{\textbf{Training}} & \multicolumn{1}{c}{\textbf{Success Rate}} & \multicolumn{1}{c}{\textbf{Difference}} & \multicolumn{1}{c}{\textbf{\% Below 1/255}} & \multicolumn{1}{c}{$\mathbf{R^2}$} \\ \midrule
\multirow{3}{*}{CIFAR10 A} & Standard & 100.00\% & 1.71\% & 100.00\% & 0.999 \\ 
 & Adversarial & 100.00\% & 4.18\% & 96.57\% & 0.995 \\ 
 & ReLU & 100.00\% & 0.18\% & 100.00\% & 1.000 \\ \midrule
\multirow{3}{*}{CIFAR10 B} & Standard & 100.00\% & 1.53\% & 100.00\% & 0.999 \\ 
 & Adversarial & 100.00\% & 2.92\% & 98.46\% & 0.996 \\ 
 & ReLU & 100.00\% & 1.19\% & 98.94\% & 0.999 \\ \midrule
\multirow{3}{*}{CIFAR10 C} & Standard & 100.00\% & 2.06\% & 100.00\% & 0.999 \\ 
 & Adversarial & 100.00\% & 3.12\% & 97.28\% & 0.996 \\ 
 & ReLU & 100.00\% & 1.45\% & 97.44\% & 0.995 \\
\bottomrule \end{tabular}
\label{tab:performanceCifar10Balanced}
\end{table}

\newlength{\resultFigureSize}
\setlength{\resultFigureSize}{0.275\textwidth}

\begin{figure}
\centering
\subfloat[MNIST A Standard Strong]{
    \begin{adjustbox}{width=\resultFigureSize}
    \includegraphics{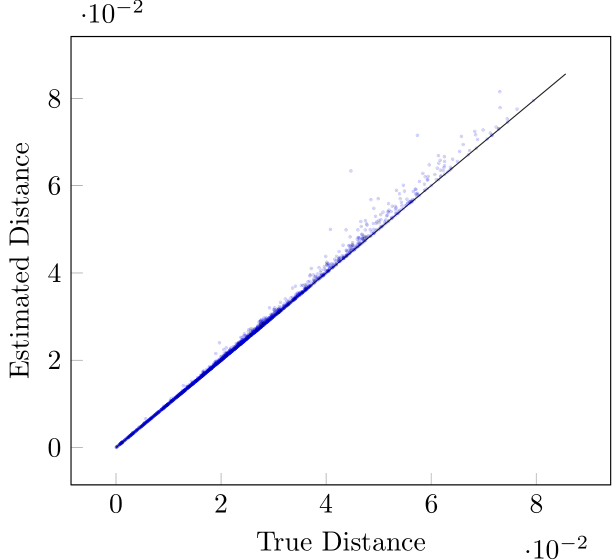}
    \end{adjustbox}
}
\hfill
\subfloat[MNIST A Standard Balanced]{
    \begin{adjustbox}{width=\resultFigureSize}
    \includegraphics{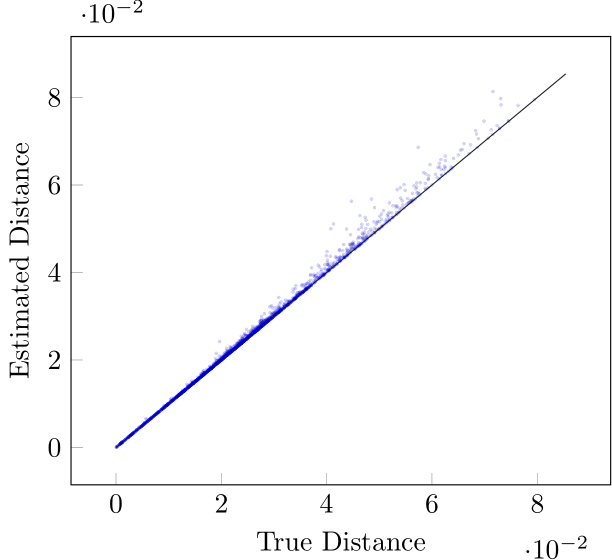}
    \end{adjustbox}
}
\hfill
\subfloat[MNIST A Adversarial Strong]{
    \begin{adjustbox}{width=\resultFigureSize}
    \includegraphics{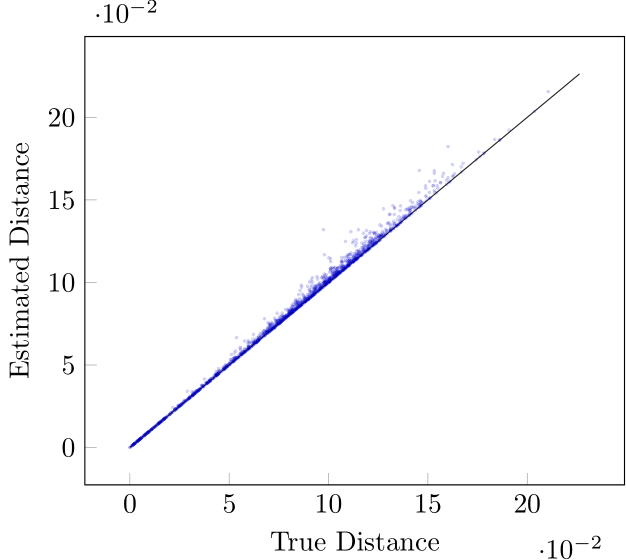}
    \end{adjustbox}
}

\subfloat[MNIST A Adversarial Balanced]{
    \begin{adjustbox}{width=\resultFigureSize}
    \includegraphics{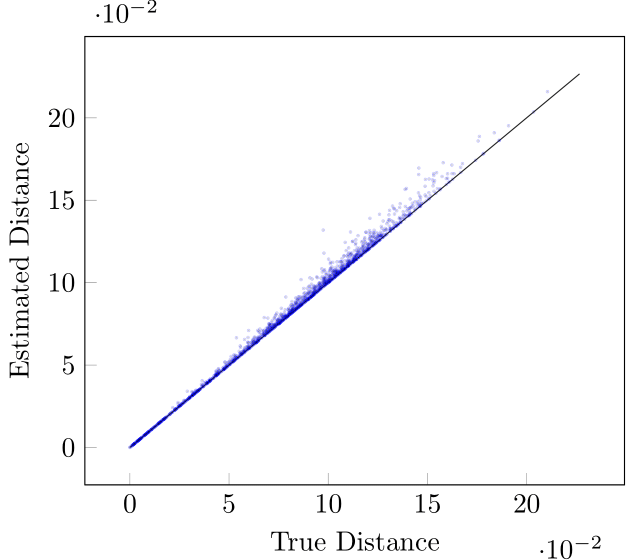}
    \end{adjustbox}
}
\hfill
\subfloat[MNIST A ReLU Strong]{
    \begin{adjustbox}{width=\resultFigureSize}
    \includegraphics{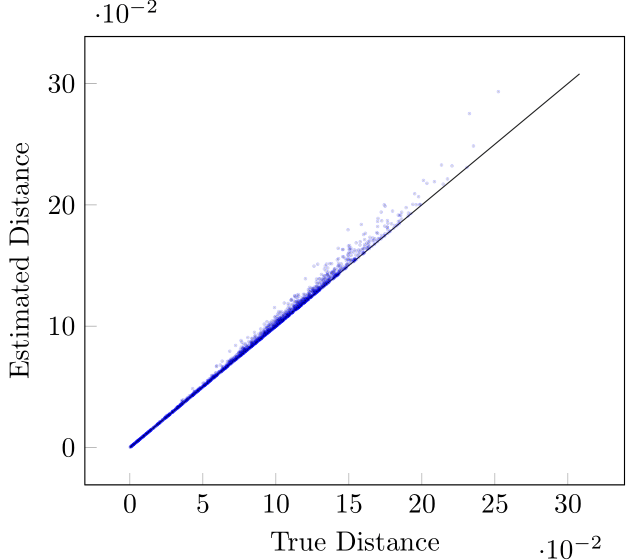}
    \end{adjustbox}
}
\hfill
\subfloat[MNIST A ReLU Balanced]{
    \begin{adjustbox}{width=\resultFigureSize}
    \includegraphics{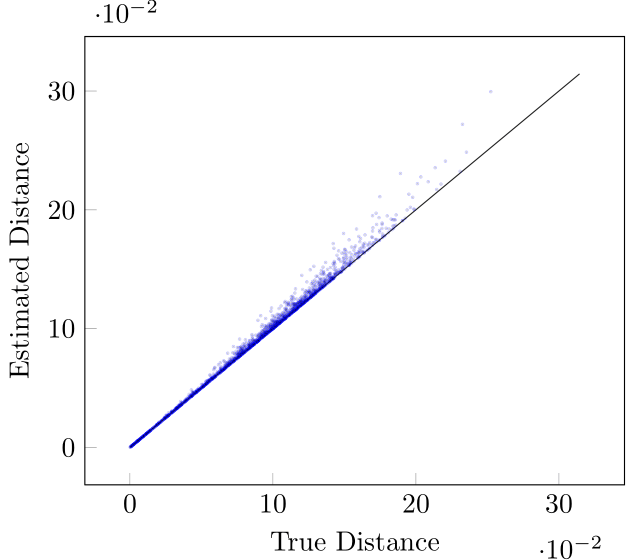}
    \end{adjustbox}
}

\caption{Decision boundary distances found by the attack pools compared to those found by MIPVerify on MNIST A. The black line represents the theoretical optimum. Note that no samples are below the black line.} 
   
\label{fig:scatterMnistA}
\end{figure}

\begin{figure}
\centering
\subfloat[MNIST B Standard Strong]{
    \begin{adjustbox}{width=\resultFigureSize}
    \includegraphics{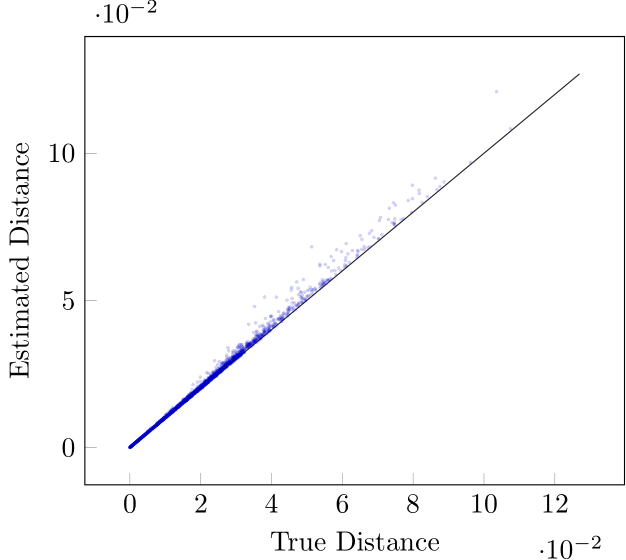}
    \end{adjustbox}
}
\hfill
\subfloat[MNIST B Standard Balanced]{
    \begin{adjustbox}{width=\resultFigureSize}
    \includegraphics{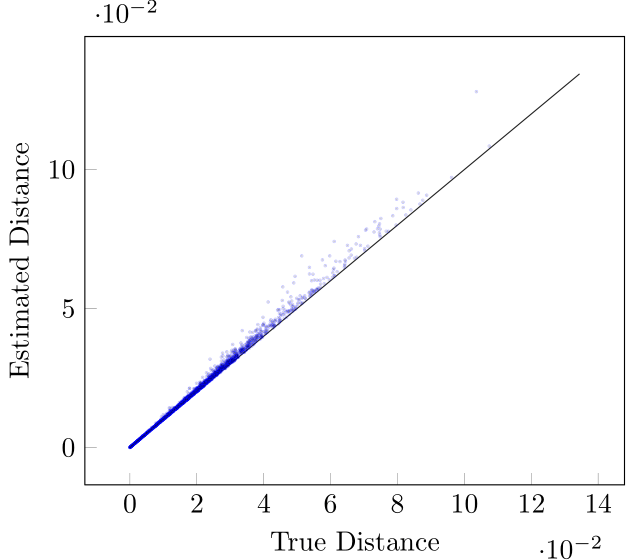}
    \end{adjustbox}
}
\hfill
\subfloat[MNIST B Adversarial Strong]{
    \begin{adjustbox}{width=\resultFigureSize}
    \includegraphics{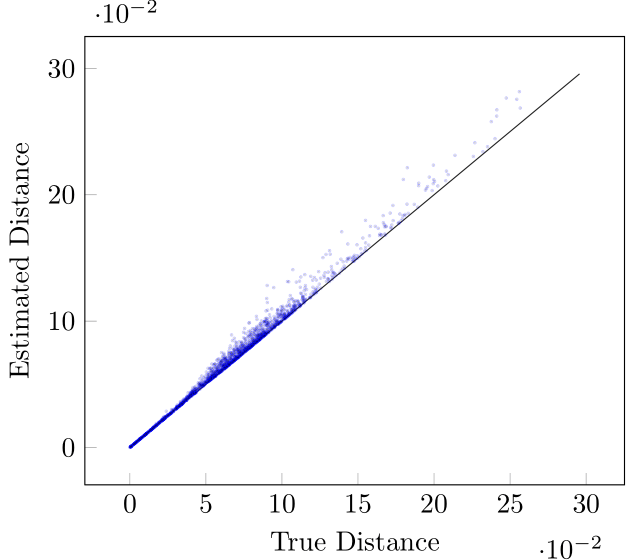}
    \end{adjustbox}
}

\subfloat[MNIST B Adversarial Balanced]{
    \begin{adjustbox}{width=\resultFigureSize}
    \includegraphics{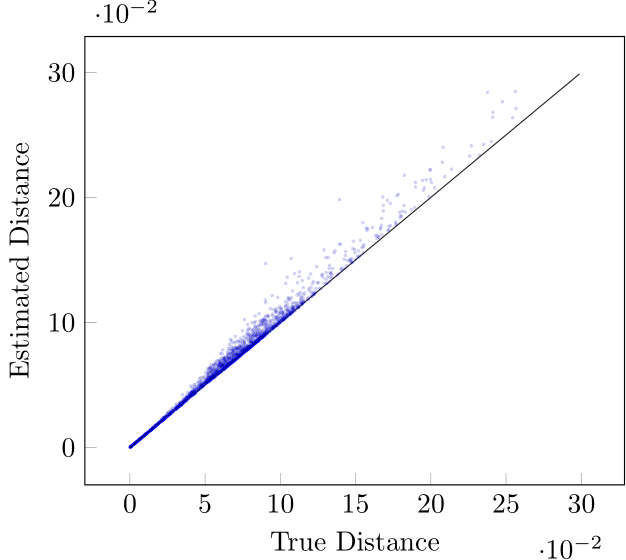}
    \end{adjustbox}
}
\hfill
\subfloat[MNIST B ReLU Strong]{
    \begin{adjustbox}{width=\resultFigureSize}
    \includegraphics{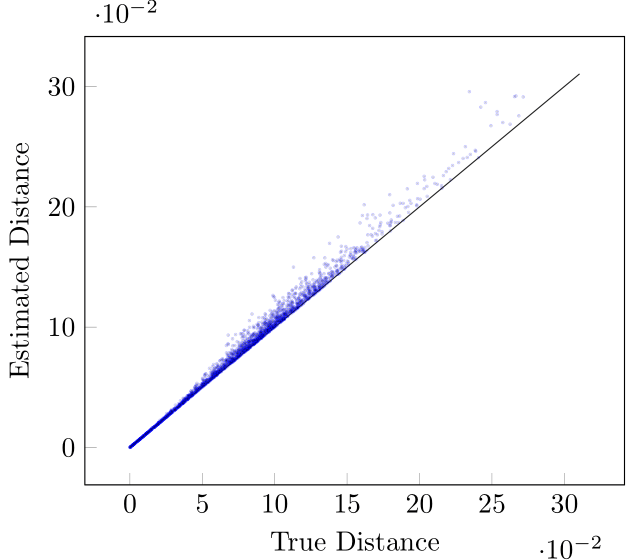}
    \end{adjustbox}
}
\hfill
\subfloat[MNIST B ReLU Balanced]{
    \begin{adjustbox}{width=\resultFigureSize}
    \includegraphics{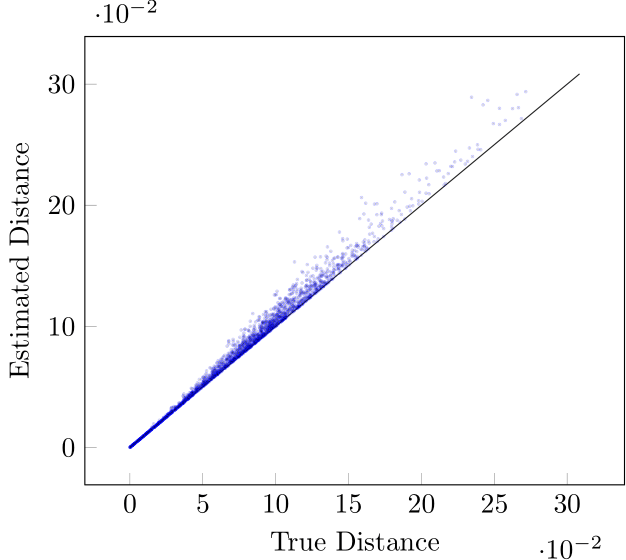}
    \end{adjustbox}
}

\caption{Decision boundary distances found by the attack pools compared to those found by MIPVerify on MNIST B. The black line represents the theoretical optimum. Note that no samples are below the black line.} 
   
\label{fig:scatterMnistB}
\end{figure}

\begin{figure}
\centering
\subfloat[MNIST C Standard Strong]{
    \begin{adjustbox}{width=\resultFigureSize}
    \includegraphics{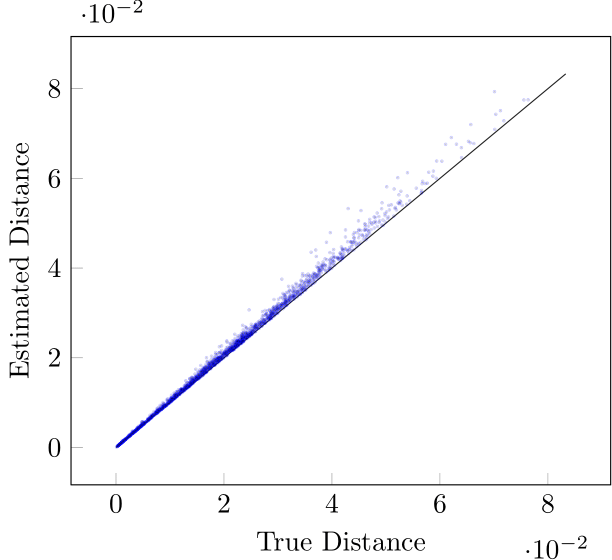}
    \end{adjustbox}
}
\hfill
\subfloat[MNIST C Standard Balanced]{
    \begin{adjustbox}{width=\resultFigureSize}
    \includegraphics{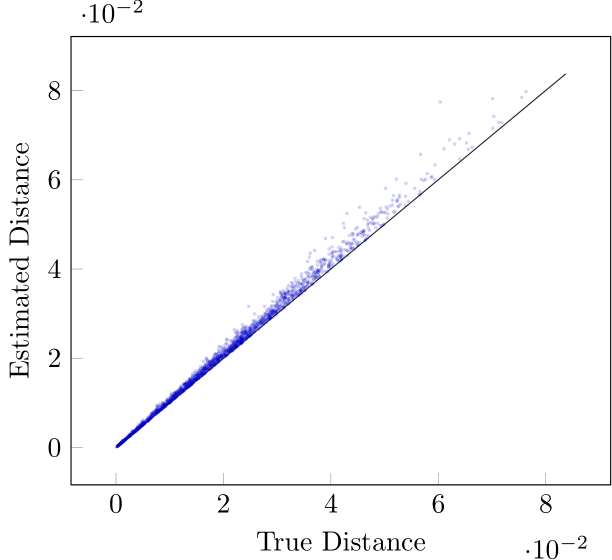}
    \end{adjustbox}
}
\hfill
\subfloat[MNIST C Adversarial Strong]{
    \begin{adjustbox}{width=\resultFigureSize}
    \includegraphics{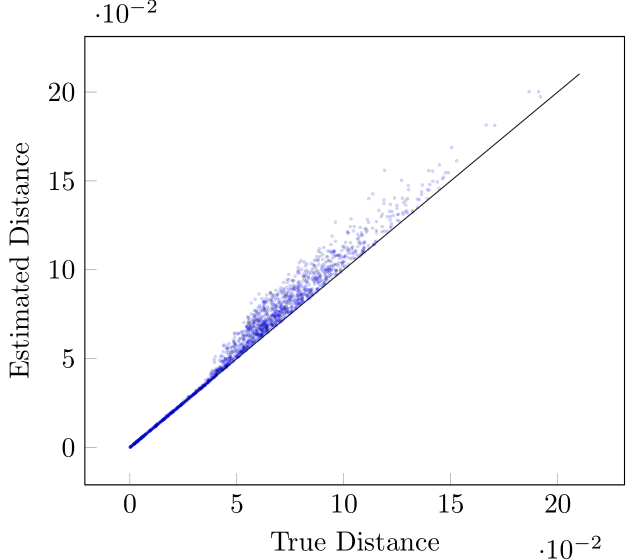}
    \end{adjustbox}
}

\subfloat[MNIST C Adversarial Balanced]{
    \begin{adjustbox}{width=\resultFigureSize}
    \includegraphics{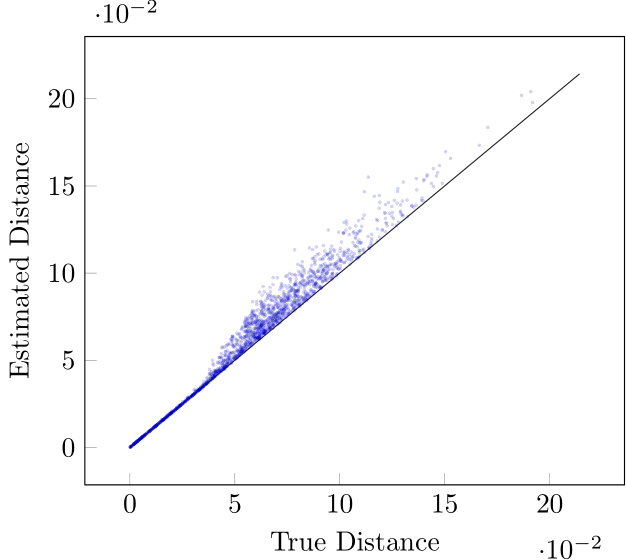}
    \end{adjustbox}
}
\hfill
\subfloat[MNIST C ReLU Strong]{
    \begin{adjustbox}{width=\resultFigureSize}
    \includegraphics{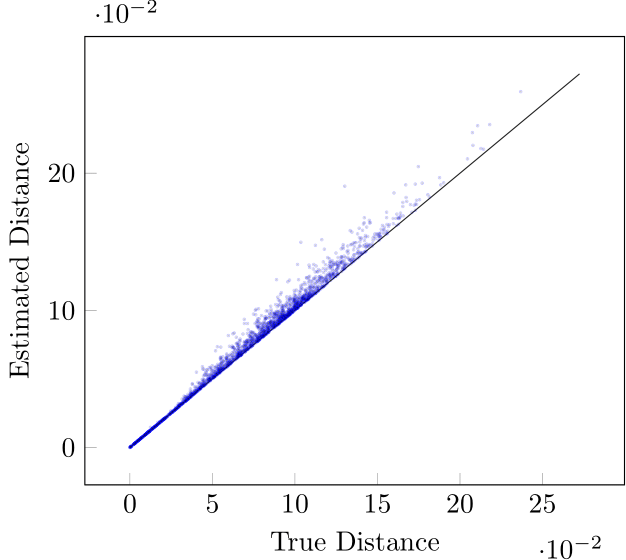}
    \end{adjustbox}
}
\hfill
\subfloat[MNIST C ReLU Balanced]{
    \begin{adjustbox}{width=\resultFigureSize}
    \includegraphics{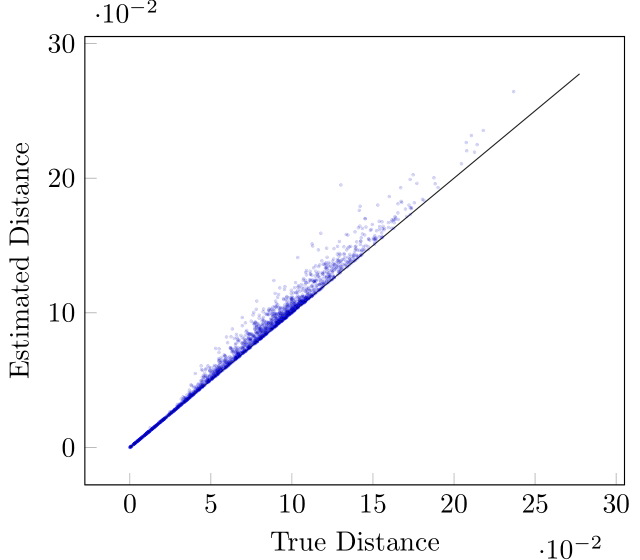}
    \end{adjustbox}
}

\caption{Decision boundary distances found by the attack pools compared to those found by MIPVerify on MNIST C. The black line represents the theoretical optimum. Note that no samples are below the black line.} 
   
\label{fig:scatterMnistC}
\end{figure}

\begin{figure}
\centering
\subfloat[CIFAR10 A Standard Strong]{
    \begin{adjustbox}{width=\resultFigureSize}
    \includegraphics{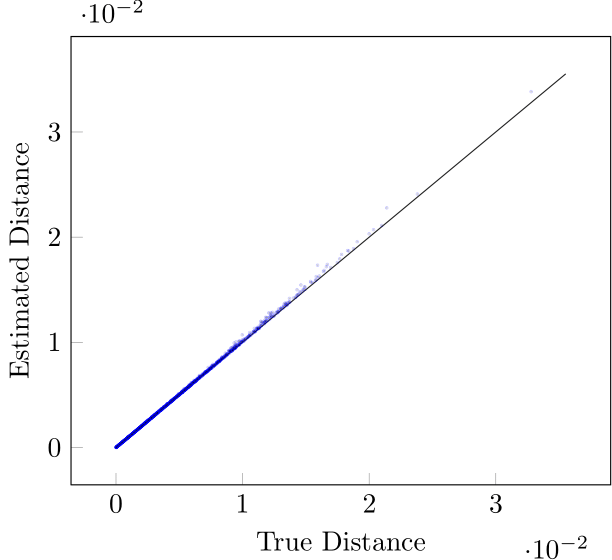}
    \end{adjustbox}
}
\hfill
\subfloat[CIFAR10 A Standard Balanced]{
    \begin{adjustbox}{width=\resultFigureSize}
    \includegraphics{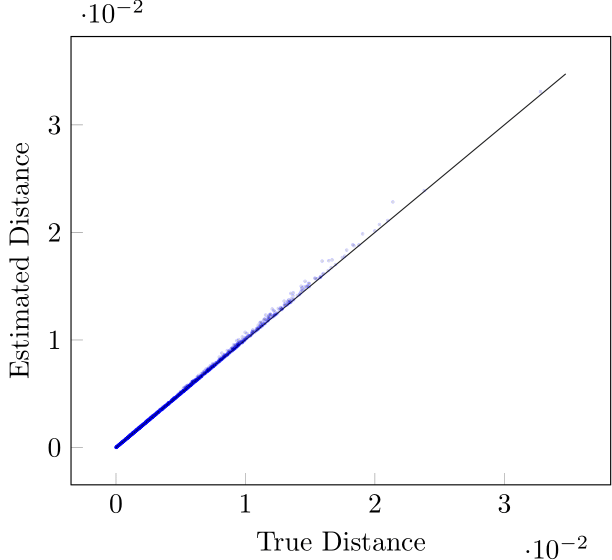}
    \end{adjustbox}
}
\hfill
\subfloat[CIFAR10 A Adversarial Strong]{
    \begin{adjustbox}{width=\resultFigureSize}
    \includegraphics{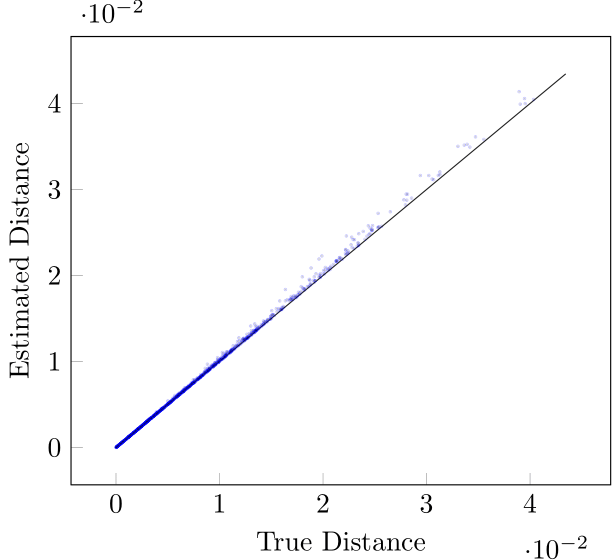}
    \end{adjustbox}
}

\subfloat[CIFAR10 A Adversarial Balanced]{
    \begin{adjustbox}{width=\resultFigureSize}
    \includegraphics{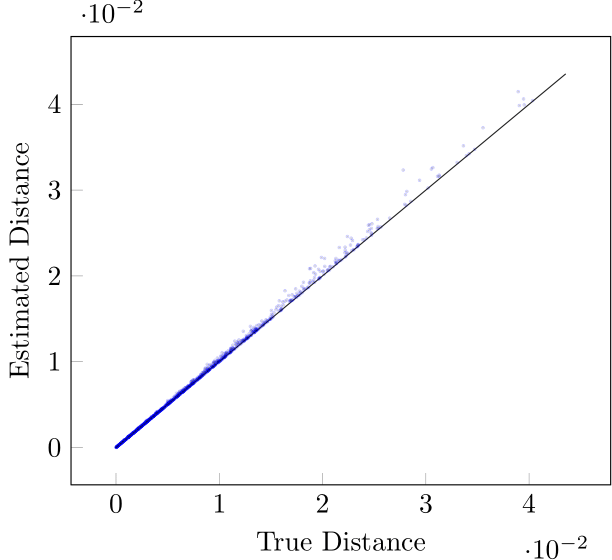}
    \end{adjustbox}
}
\hfill
\subfloat[CIFAR10 A ReLU Strong]{
    \begin{adjustbox}{width=\resultFigureSize}
    \includegraphics{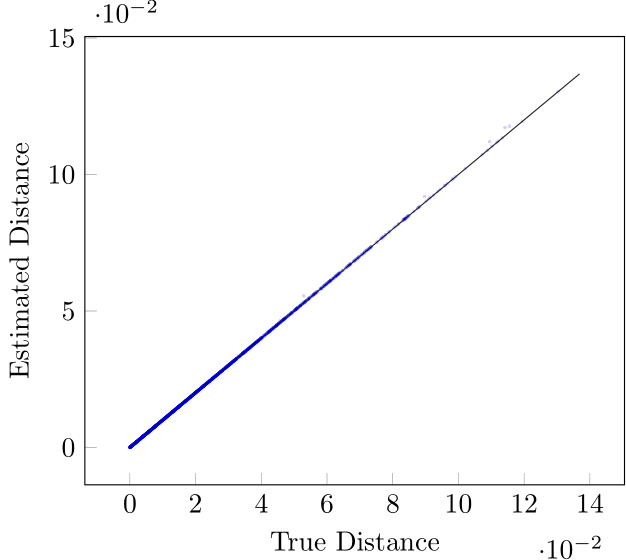}
    \end{adjustbox}
}
\hfill
\subfloat[CIFAR10 A ReLU Balanced]{
    \begin{adjustbox}{width=\resultFigureSize}
    \includegraphics{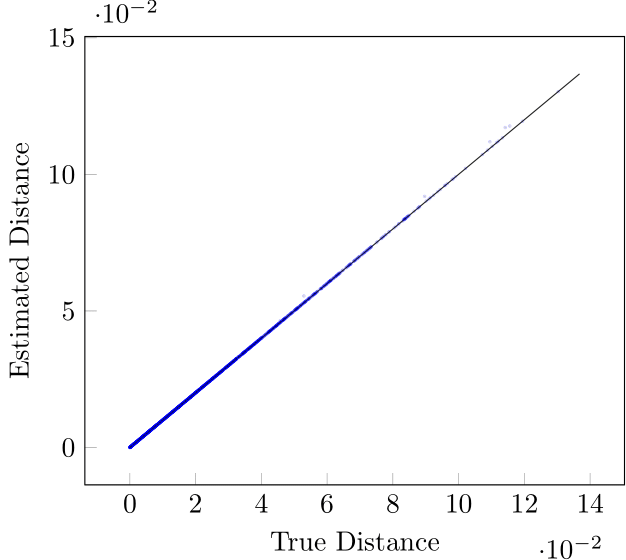}
    \end{adjustbox}
}

\caption{Decision boundary distances found by the attack pools compared to those found by MIPVerify on CIFAR10 A. The black line represents the theoretical optimum. Note that no samples are below the black line.} 
   
\label{fig:scatterCifar10A}
\end{figure}

\begin{figure}
\centering
\subfloat[CIFAR10 B Standard Strong]{
    \begin{adjustbox}{width=\resultFigureSize}
    \includegraphics{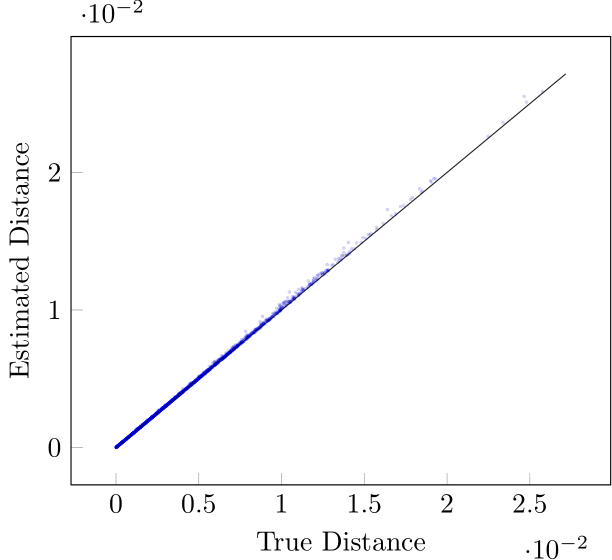}
    \end{adjustbox}
}
\hfill
\subfloat[CIFAR10 B Standard Balanced]{
    \begin{adjustbox}{width=\resultFigureSize}
    \includegraphics{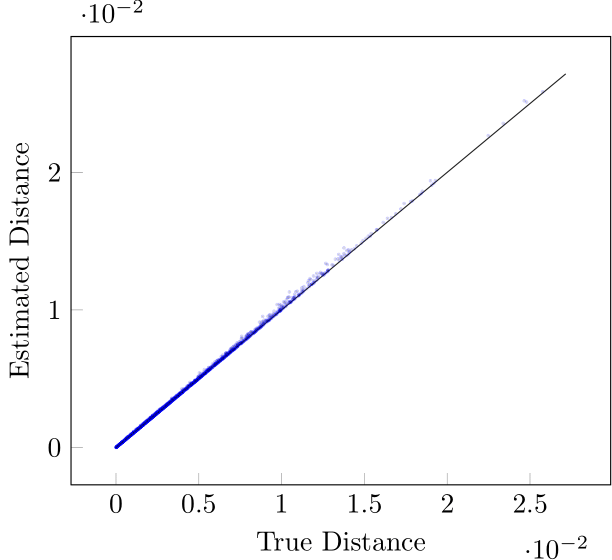}
    \end{adjustbox}
}
\hfill
\subfloat[CIFAR10 B Adversarial Strong]{
    \begin{adjustbox}{width=\resultFigureSize}
    \includegraphics{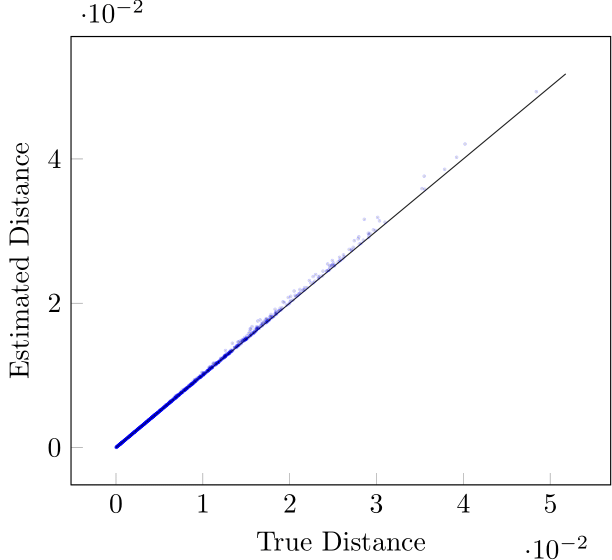}
    \end{adjustbox}
}

\subfloat[CIFAR10 B Adversarial Balanced]{
    \begin{adjustbox}{width=\resultFigureSize}
    \includegraphics{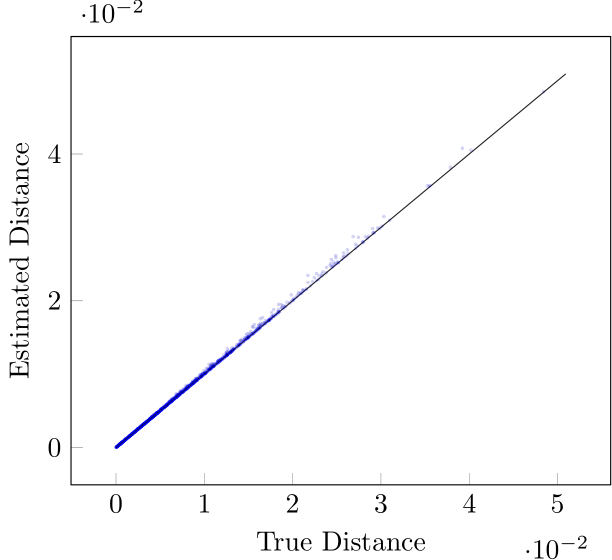}
    \end{adjustbox}
}
\hfill
\subfloat[CIFAR10 B ReLU Strong]{
    \begin{adjustbox}{width=\resultFigureSize}
    \includegraphics{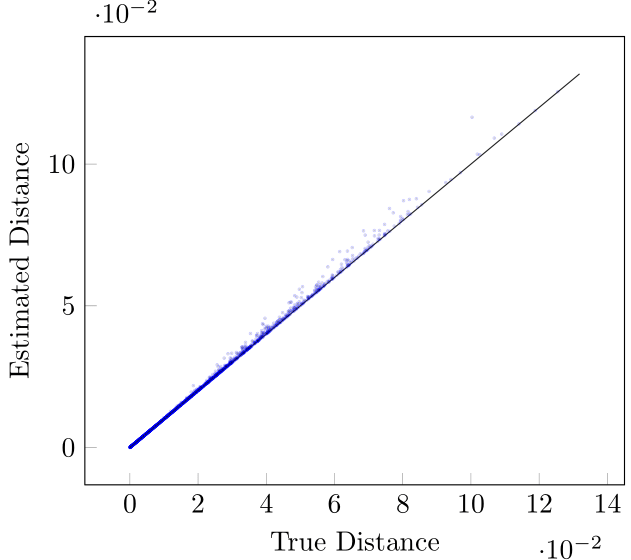}
    \end{adjustbox}
}
\hfill
\subfloat[CIFAR10 B ReLU Balanced]{
    \begin{adjustbox}{width=\resultFigureSize}
    \includegraphics{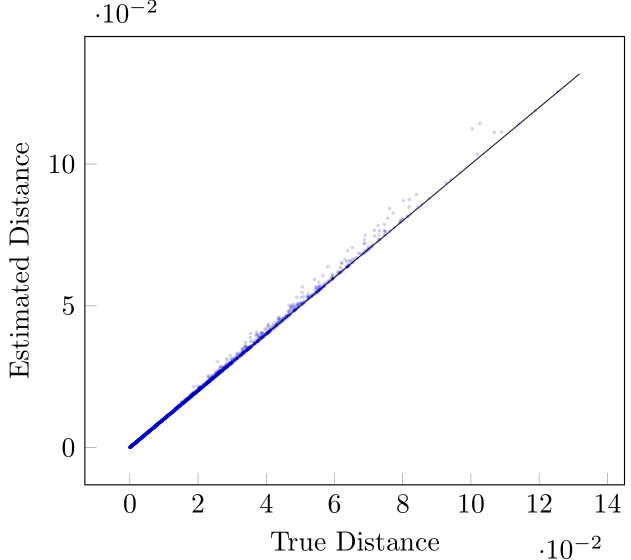}
    \end{adjustbox}
}

\caption{Decision boundary distances found by the attack pools compared to those found by MIPVerify on CIFAR10 B. The black line represents the theoretical optimum. Note that no samples are below the black line.} 
   
\label{fig:scatterCifar10B}
\end{figure}

\begin{figure}
\centering
\subfloat[CIFAR10 C Standard Strong]{
    \begin{adjustbox}{width=\resultFigureSize}
    \includegraphics{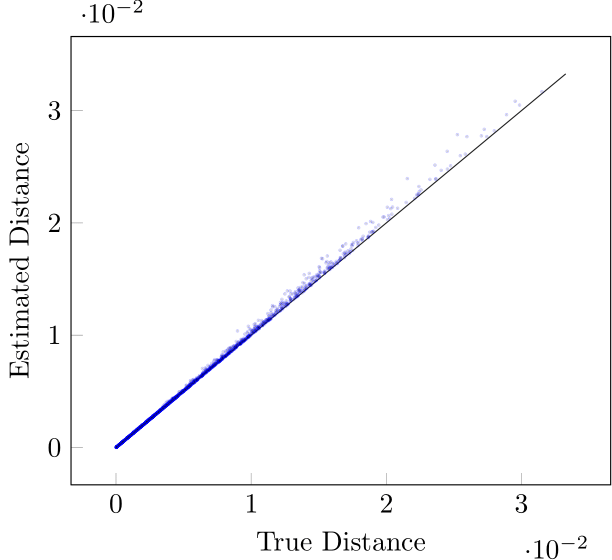}
    \end{adjustbox}
}
\hfill
\subfloat[CIFAR10 C Standard Balanced]{
    \begin{adjustbox}{width=\resultFigureSize}
    \includegraphics{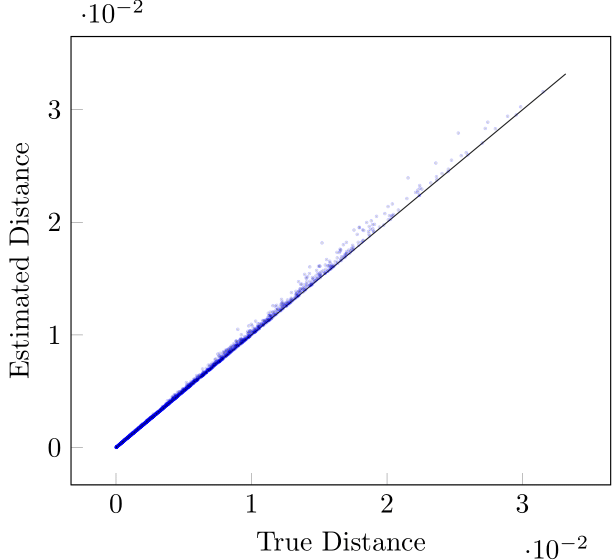}
    \end{adjustbox}
}
\hfill
\subfloat[CIFAR10 C Adversarial Strong]{
    \begin{adjustbox}{width=\resultFigureSize}
    \includegraphics{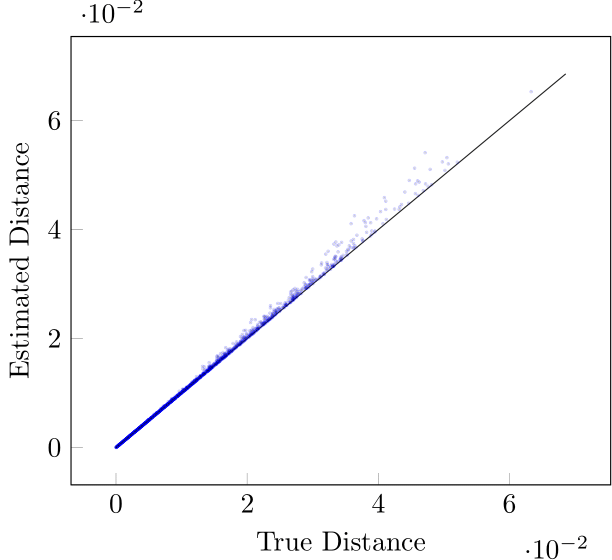}
    \end{adjustbox}
}

\subfloat[CIFAR10 C Adversarial Balanced]{
    \begin{adjustbox}{width=\resultFigureSize}
    \includegraphics{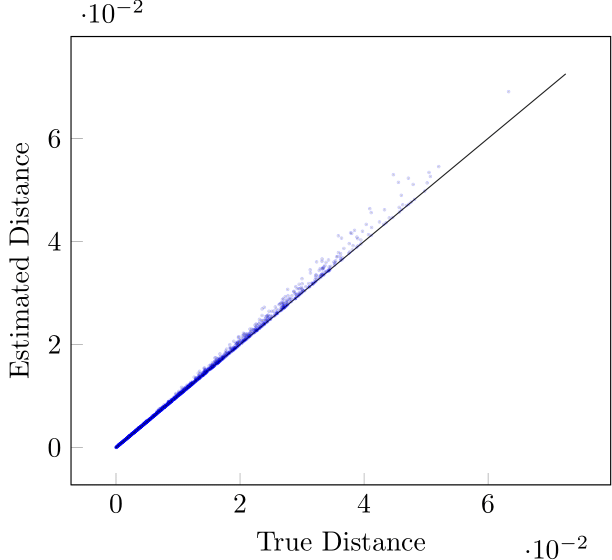}
    \end{adjustbox}
}
\hfill
\subfloat[CIFAR10 C ReLU Strong]{
    \begin{adjustbox}{width=\resultFigureSize}
    \includegraphics{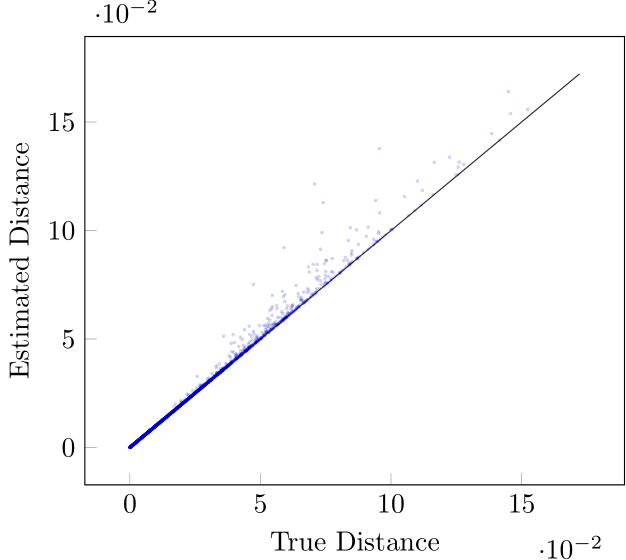}
    \end{adjustbox}
}
\hfill
\subfloat[CIFAR10 C ReLU Balanced]{
    \begin{adjustbox}{width=\resultFigureSize}
    \includegraphics{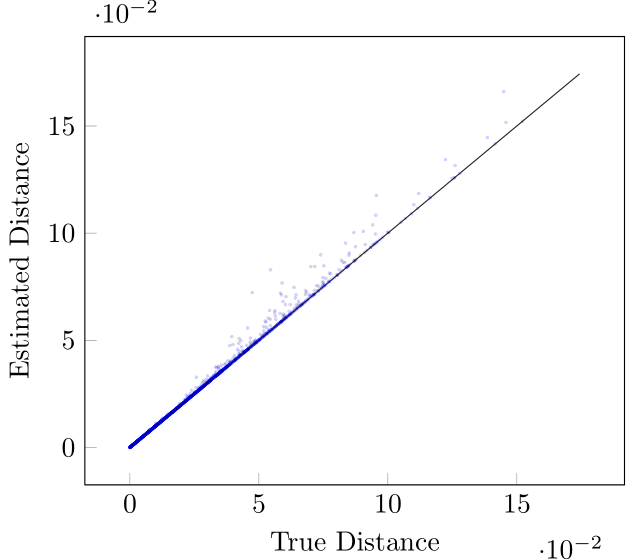}
    \end{adjustbox}
}

\caption{Decision boundary distances found by the attack pools compared to those found by MIPVerify on CIFAR10 C. The black line represents the theoretical optimum. Note that no samples are below the black line.} 
   
\label{fig:scatterCifar10C}
\end{figure}
\clearpage
\section{Ablation Study}
\label{app:ablation}

We outline the best attack pools by size in \Cref{tab:bestPoolsMnistStrong,tab:bestPoolsMnistBalanced,tab:bestPoolsCifar10Strong,tab:bestPoolsCifar10Balanced}. Additionally, we report the performance of pools composed of individual attacks in \Cref{tab:individualAttacksMnistStrong,tab:individualAttacksMnistBalanced,tab:individualAttacksCifar10Strong,tab:individualAttacksCifar10Balanced}. Finally, we detail the performance of dropping a specific attack in \Cref{tab:droppedAttacksMnistStrong,tab:droppedAttacksMnistBalanced,tab:droppedAttacksCifar10Strong,tab:droppedAttacksCifar10Balanced}.

\begin{table}[]
\caption{Best pools of a given size by success rate and $R^2$ for MNIST strong.}
\centering
\begin{tabular}{llllll}
\toprule
$\bm{n}$ & \textbf{Attacks}          & \textbf{Success Rate} & \textbf{Difference} & \textbf{$\mathbf{<}$ 1/255} & $\mathbf{R^2}$  \\ \midrule
1             & PGD                       & 100.00\textpm0.00\%   & 10.98\textpm4.41\%  & 51.83\textpm27.78\%     & 0.975\textpm0.010 \\
2             & C\&W, PGD                 & 100.00\textpm0.00\%   & 7.99\textpm3.31\%   & 60.68\textpm25.43\%     & 0.986\textpm0.005 \\
3             & B\&B, C\&W, PGD           & 100.00\textpm0.00\%   & 4.71\textpm1.97\%   & 77.97\textpm15.52\%     & 0.989\textpm0.004 \\
4             & B\&B, C\&W, DF, PGD & 100.00\textpm0.00\%   & 4.36\textpm2.03\%   & 79.02\textpm15.62\%     & 0.991\textpm0.005 \\
5             & No FGSM, Uniform  & 100.00\textpm0.00\%   & 4.09\textpm2.02\%   & 79.81\textpm15.70\%     & 0.992\textpm0.005 \\
6             & No Uniform        & 100.00\textpm0.00\%   & 4.09\textpm2.02\%   & 79.81\textpm15.70\%     & 0.992\textpm0.005 \\
7             & All                       & 100.00\textpm0.00\%   & 4.09\textpm2.02\%   & 79.81\textpm15.70\%     & 0.992\textpm0.005 \\ \bottomrule
\end{tabular}
\label{tab:bestPoolsMnistStrong}
\end{table}

\begin{table}[]
\caption{Best pools of a given size by success rate and $R^2$ for MNIST balanced.}
\centering
\begin{tabular}{llllll}
\toprule
$\bm{n}$ & \textbf{Attacks} & \textbf{Success Rate} & \textbf{Difference} & \textbf{$\mathbf{<}$ 1/255} & $\mathbf{R^2}$ \\ \midrule
1 & BIM & 100.00\textpm0.00\% & 11.72\textpm4.18\% & 50.92\textpm26.43\% & 0.965\textpm0.010 \\
2 & BIM, B\&B & 100.00\textpm0.00\% & 6.11\textpm2.28\% & 73.23\textpm15.90\% & 0.980\textpm0.007 \\
3 & BIM, B\&B, C\&W & 100.00\textpm0.00\% & 5.29\textpm2.06\% & 75.72\textpm16.10\% & 0.986\textpm0.005 \\
4 & BIM, B\&B, C\&W, DF & 100.00\textpm0.00\% & 4.85\textpm2.10\% & 77.33\textpm15.85\% & 0.989\textpm0.005 \\
5 & No FGSM, Uniform & 100.00\textpm0.00\% & 4.65\textpm2.16\% & 77.78\textpm16.08\% & 0.990\textpm0.006 \\
6 & No Uniform & 100.00\textpm0.00\% & 4.65\textpm2.16\% & 77.78\textpm16.08\% & 0.990\textpm0.006 \\
7 & All & 100.00\textpm0.00\% & 4.65\textpm2.16\% & 77.78\textpm16.08\% & 0.990\textpm0.006 \\ \bottomrule
\end{tabular}
\label{tab:bestPoolsMnistBalanced}
\end{table}

\begin{table}[]
\caption{Best pools of a given size by success rate and $R^2$ for CIFAR10 strong.}
\centering
\begin{tabular}{llllll}
\toprule
$\bm{n}$ & \textbf{Attacks} & \textbf{Success Rate} & \textbf{Difference} & \textbf{$\mathbf{<}$ 1/255} & $\mathbf{R^2}$ \\ \midrule
1 & DF & 100.00\textpm0.00\% & 6.11\textpm3.49\% & 95.06\textpm4.81\% & 0.989\textpm0.011 \\
2 & DF, PGD & 100.00\textpm0.00\% & 4.71\textpm2.37\% & 96.32\textpm3.56\% & 0.995\textpm0.007 \\
3 & C\&W, DF, PGD & 100.00\textpm0.00\% & 2.54\textpm1.30\% & 98.17\textpm2.00\% & 0.996\textpm0.006 \\
4 & B\&B, C\&W, DF, PGD & 100.00\textpm0.00\% & 2.21\textpm1.16\% & 98.40\textpm1.63\% & 0.997\textpm0.003 \\
5 & No FGSM, Uniform & 100.00\textpm0.00\% & 2.21\textpm1.16\% & 98.40\textpm1.63\% & 0.997\textpm0.003 \\
6 & No Uniform & 100.00\textpm0.00\% & 2.21\textpm1.16\% & 98.40\textpm1.63\% & 0.997\textpm0.003 \\
7 & All & 100.00\textpm0.00\% & 2.21\textpm1.16\% & 98.40\textpm1.63\% & 0.997\textpm0.003 \\ \bottomrule
\end{tabular}
\label{tab:bestPoolsCifar10Strong}
\end{table}

\begin{table}[]
\caption{Best pools of a given size by success rate and $R^2$ for CIFAR10 balanced.}
\centering
\begin{tabular}{llllll}
\toprule
$\bm{n}$ & \textbf{Attacks} & \textbf{Success Rate} & \textbf{Difference} & \textbf{$\mathbf{<}$ 1/255} & $\mathbf{R^2}$ \\ \midrule
1 & DF & 100.00\textpm0.00\% & 6.11\textpm3.49\% & 95.06\textpm4.81\% & 0.989\textpm0.011 \\
2 & B\&B, DF & 100.00\textpm0.00\% & 2.52\textpm1.51\% & 98.23\textpm1.81\% & 0.995\textpm0.004 \\
3 & BIM, B\&B, DF & 100.00\textpm0.00\% & 2.21\textpm1.25\% & 98.53\textpm1.52\% & 0.997\textpm0.002 \\
4 & BIM, B\&B, C\&W, DF & 100.00\textpm0.00\% & 2.06\textpm1.16\% & 98.73\textpm1.32\% & 0.998\textpm0.002 \\
5 & No FGSM, Uniform & 100.00\textpm0.00\% & 2.04\textpm1.13\% & 98.74\textpm1.29\% & 0.998\textpm0.002 \\
6 & No FGSM & 100.00\textpm0.00\% & 2.04\textpm1.13\% & 98.74\textpm1.29\% & 0.998\textpm0.002 \\
7 & All & 100.00\textpm0.00\% & 2.04\textpm1.13\% & 98.74\textpm1.29\% & 0.998\textpm0.002 \\ \bottomrule
\end{tabular}

\label{tab:bestPoolsCifar10Balanced}
\end{table}

\begin{table}[]
\caption{Performance of individual attacks for MNIST strong.}
\centering
\begin{tabular}{lllll}
\toprule
\textbf{Attack} & \textbf{Success Rate} & \textbf{Difference} & \textbf{$\mathbf{<}$ 1/255} & $\mathbf{R^2}$ \\ \midrule
BIM & 100.00\textpm0.00\% & 10.90\textpm4.42\% & 53.57\textpm28.07\% & 0.966\textpm0.012 \\
B\&B & 99.99\textpm0.01\% & 18.50\textpm7.09\% & 58.78\textpm9.91\% & 0.812\textpm0.044 \\
C\&W & 100.00\textpm0.00\% & 17.52\textpm2.74\% & 48.02\textpm21.28\% & 0.910\textpm0.024 \\
Deepfool & 100.00\textpm0.00\% & 21.59\textpm7.73\% & 44.15\textpm20.02\% & 0.923\textpm0.027 \\
FGSM & 99.72\textpm0.51\% & 44.43\textpm15.76\% & 28.20\textpm17.30\% & 0.761\textpm0.132 \\
PGD & 100.00\textpm0.00\% & 10.98\textpm4.41\% & 51.83\textpm27.78\% & 0.975\textpm0.010 \\
Uniform & 99.52\textpm0.91\% & 414.47\textpm140.54\% & 0.82\textpm0.55\% & 0.623\textpm0.138 \\ \bottomrule
\end{tabular}
\label{tab:individualAttacksMnistStrong}
\end{table}

\begin{table}[]
\caption{Performance of individual attacks for MNIST balanced.}
\centering
\begin{tabular}{lllll}
\toprule
\textbf{Attack} & \textbf{Success Rate} & \textbf{Difference} & \textbf{$\mathbf{<}$ 1/255} & $\mathbf{R^2}$ \\ \midrule
BIM & 100.00\textpm0.00\% & 11.72\textpm4.18\% & 50.92\textpm26.43\% & 0.965\textpm0.010 \\
B\&B & 99.99\textpm0.03\% & 18.65\textpm7.29\% & 58.43\textpm9.61\% & 0.812\textpm0.039 \\
C\&W & 100.00\textpm0.00\% & 22.55\textpm3.83\% & 38.95\textpm22.49\% & 0.904\textpm0.025 \\
Deepfool & 100.00\textpm0.00\% & 21.59\textpm7.73\% & 44.15\textpm20.02\% & 0.923\textpm0.027 \\
FGSM & 99.72\textpm0.51\% & 44.43\textpm15.76\% & 28.20\textpm17.30\% & 0.761\textpm0.132 \\
PGD & 100.00\textpm0.00\% & 16.23\textpm6.59\% & 48.08\textpm28.88\% & 0.905\textpm0.070 \\
Uniform & 98.66\textpm1.90\% & 521.61\textpm181.40\% & 0.57\textpm0.38\% & 0.484\textpm0.122 \\ \bottomrule
\end{tabular}
\label{tab:individualAttacksMnistBalanced}
\end{table}

\begin{table}[]
\caption{Performance of individual attacks for CIFAR10 strong.}
\centering
\begin{tabular}{lllll}
\toprule
\textbf{Attack} & \textbf{Success Rate} & \textbf{Difference} & \textbf{$\mathbf{<}$ 1/255} & $\mathbf{R^2}$ \\ \midrule
BIM & 91.96\textpm7.40\% & 19.97\textpm5.95\% & 80.32\textpm12.97\% & 0.934\textpm0.041 \\
B\&B & 100.00\textpm0.00\% & 508.66\textpm196.37\% & 42.74\textpm7.85\% & 0.174\textpm0.074 \\
C\&W & 99.98\textpm0.02\% & 10.67\textpm3.64\% & 90.09\textpm5.51\% & 0.926\textpm0.030 \\
Deepfool & 100.00\textpm0.00\% & 6.11\textpm3.49\% & 95.06\textpm4.81\% & 0.989\textpm0.011 \\
FGSM & 100.00\textpm0.00\% & 31.80\textpm11.12\% & 69.20\textpm17.72\% & 0.847\textpm0.123 \\
PGD & 100.00\textpm0.00\% & 19.36\textpm5.99\% & 77.23\textpm15.89\% & 0.952\textpm0.027 \\
Uniform & 99.99\textpm0.02\% & 1206.79\textpm277.68\% & 2.48\textpm0.88\% & 0.910\textpm0.044 \\ \bottomrule
\end{tabular}
\label{tab:individualAttacksCifar10Strong}
\end{table}

\begin{table}[]
\caption{Performance of individual attacks for CIFAR10 balanced.}
\centering
\begin{tabular}{lllll}
\toprule
\textbf{Attack} & \textbf{Success Rate} & \textbf{Difference} & \textbf{$\mathbf{<}$ 1/255} & $\mathbf{R^2}$ \\ \midrule
BIM & 100.00\textpm0.00\% & 19.23\textpm5.92\% & 77.33\textpm15.89\% & 0.954\textpm0.025 \\
B\&B & 100.00\textpm0.00\% & 50.64\textpm52.17\% & 81.20\textpm10.68\% & 0.615\textpm0.349 \\
C\&W & 99.89\textpm0.09\% & 17.44\textpm4.01\% & 84.82\textpm8.51\% & 0.923\textpm0.026 \\
Deepfool & 100.00\textpm0.00\% & 6.11\textpm3.49\% & 95.06\textpm4.81\% & 0.989\textpm0.011 \\
FGSM & 100.00\textpm0.00\% & 31.80\textpm11.12\% & 69.20\textpm17.72\% & 0.847\textpm0.123 \\
PGD & 100.00\textpm0.00\% & 20.18\textpm6.56\% & 76.97\textpm16.07\% & 0.947\textpm0.031 \\
Uniform & 99.85\textpm0.26\% & 1617.74\textpm390.50\% & 1.80\textpm0.67\% & 0.853\textpm0.068 \\ \bottomrule
\end{tabular}
\label{tab:individualAttacksCifar10Balanced}
\end{table}

\begin{table}[]
\caption{Performance of pools without a specific attack for MNIST strong.}
\centering
\begin{tabular}{lllll}
\toprule
\textbf{Dropped Attack} & \textbf{Success Rate} & \textbf{Difference} & \textbf{$\mathbf{<}$ 1/255} & $\mathbf{R^2}$ \\ \midrule
None & 100.00\textpm0.00\% & 4.09\textpm2.02\% & 79.81\textpm15.70\% & 0.992\textpm0.005 \\
BIM & 100.00\textpm0.00\% & 4.35\textpm2.03\% & 79.02\textpm15.62\% & 0.991\textpm0.005 \\
B\&B & 100.00\textpm0.00\% & 6.76\textpm3.31\% & 64.46\textpm25.01\% & 0.990\textpm0.005 \\
C\&W & 100.00\textpm0.00\% & 4.65\textpm2.20\% & 77.70\textpm16.02\% & 0.989\textpm0.006 \\
Deepfool & 100.00\textpm0.00\% & 4.33\textpm1.97\% & 79.04\textpm15.75\% & 0.990\textpm0.004 \\
FGSM & 100.00\textpm0.00\% & 4.09\textpm2.02\% & 79.81\textpm15.70\% & 0.992\textpm0.005 \\
PGD & 100.00\textpm0.00\% & 4.26\textpm1.99\% & 79.36\textpm15.59\% & 0.991\textpm0.004 \\
Uniform & 100.00\textpm0.00\% & 4.09\textpm2.02\% & 79.81\textpm15.70\% & 0.992\textpm0.005 \\ \bottomrule
\end{tabular}
\label{tab:droppedAttacksMnistStrong}
\end{table}

\begin{table}[]
\caption{Performance of pools without a specific attack for MNIST balanced.}
\centering
\begin{tabular}{lllll}
\toprule
\textbf{Dropped Attack} & \textbf{Success Rate} & \textbf{Difference} & \textbf{$\mathbf{<}$ 1/255} & $\mathbf{R^2}$ \\ \midrule
None & 100.00\textpm0.00\% & 4.65\textpm2.16\% & 77.78\textpm16.08\% & 0.990\textpm0.006 \\
BIM & 100.00\textpm0.00\% & 5.13\textpm2.27\% & 76.14\textpm15.98\% & 0.988\textpm0.007 \\
B\&B & 100.00\textpm0.00\% & 7.93\textpm3.69\% & 60.79\textpm25.99\% & 0.987\textpm0.006 \\
C\&W & 100.00\textpm0.00\% & 4.93\textpm2.22\% & 77.05\textpm15.96\% & 0.988\textpm0.006 \\
Deepfool & 100.00\textpm0.00\% & 5.03\textpm2.14\% & 76.34\textpm16.36\% & 0.988\textpm0.005 \\
FGSM & 100.00\textpm0.00\% & 4.65\textpm2.16\% & 77.78\textpm16.08\% & 0.990\textpm0.006 \\
PGD & 100.00\textpm0.00\% & 4.85\textpm2.10\% & 77.33\textpm15.85\% & 0.989\textpm0.005 \\
Uniform & 100.00\textpm0.00\% & 4.65\textpm2.16\% & 77.78\textpm16.08\% & 0.990\textpm0.006 \\ \bottomrule
\end{tabular}
\label{tab:droppedAttacksMnistBalanced}
\end{table}

\begin{table}[]
\caption{Performance of pools without a specific attack for CIFAR10 strong.}
\centering
\begin{tabular}{lllll}
\toprule
\textbf{Dropped Attack} & \textbf{Success Rate} & \textbf{Difference} & \textbf{$\mathbf{<}$ 1/255} & $\mathbf{R^2}$ \\ \midrule
None & 100.00\textpm0.00\% & 2.21\textpm1.16\% & 98.40\textpm1.63\% & 0.997\textpm0.003 \\
BIM & 100.00\textpm0.00\% & 2.21\textpm1.16\% & 98.40\textpm1.63\% & 0.997\textpm0.003 \\
B\&B & 100.00\textpm0.00\% & 2.54\textpm1.30\% & 98.17\textpm2.00\% & 0.996\textpm0.006 \\
C\&W & 100.00\textpm0.00\% & 3.83\textpm2.06\% & 96.84\textpm3.12\% & 0.996\textpm0.004 \\
Deepfool & 100.00\textpm0.00\% & 4.02\textpm1.19\% & 95.65\textpm3.10\% & 0.992\textpm0.005 \\
FGSM & 100.00\textpm0.00\% & 2.21\textpm1.16\% & 98.40\textpm1.63\% & 0.997\textpm0.003 \\
PGD & 100.00\textpm0.00\% & 2.50\textpm1.48\% & 98.11\textpm1.93\% & 0.995\textpm0.005 \\
Uniform & 100.00\textpm0.00\% & 2.21\textpm1.16\% & 98.40\textpm1.63\% & 0.997\textpm0.003 \\ \bottomrule
\end{tabular}
\label{tab:droppedAttacksCifar10Strong}
\end{table}

\begin{table}[]
\caption{Performance of pools without a specific attack for CIFAR10 balanced.}
\centering
\begin{tabular}{lllll}
\toprule
\textbf{Dropped Attack} & \textbf{Success Rate} & \textbf{Difference} & \textbf{$\mathbf{<}$ 1/255} & $\mathbf{R^2}$ \\ \midrule
None & 100.00\textpm0.00\% & 2.04\textpm1.13\% & 98.74\textpm1.29\% & 0.998\textpm0.002 \\
BIM & 100.00\textpm0.00\% & 2.07\textpm1.15\% & 98.72\textpm1.31\% & 0.998\textpm0.002 \\
B\&B & 100.00\textpm0.00\% & 4.08\textpm1.95\% & 97.26\textpm2.70\% & 0.996\textpm0.006 \\
C\&W & 100.00\textpm0.00\% & 2.18\textpm1.22\% & 98.54\textpm1.50\% & 0.997\textpm0.002 \\
Deepfool & 100.00\textpm0.00\% & 4.00\textpm0.99\% & 95.89\textpm3.13\% & 0.993\textpm0.003 \\
FGSM & 100.00\textpm0.00\% & 2.04\textpm1.13\% & 98.74\textpm1.29\% & 0.998\textpm0.002 \\
PGD & 100.00\textpm0.00\% & 2.06\textpm1.16\% & 98.73\textpm1.32\% & 0.998\textpm0.002 \\
Uniform & 100.00\textpm0.00\% & 2.04\textpm1.13\% & 98.74\textpm1.29\% & 0.998\textpm0.002 \\ \bottomrule
\end{tabular}
\label{tab:droppedAttacksCifar10Balanced}
\end{table}
\clearpage
\section{Fast Parameter Set Tests}
\label{app:convergence}

We list the chosen parameter sets for Fast-100, Fast-1k and Fast-10k in \Cref{tab:convergenceParameters}. We plot the difference between the distance of the closest adversarial examples and the true decision boundary distance in \Cref{fig:convergenceAStandard,fig:convergenceAAdversarial,fig:convergenceAReLU,fig:convergenceBStandard,fig:convergenceBAdversarial,fig:convergenceBReLU,fig:convergenceCStandard,fig:convergenceCAdversarial,fig:convergenceCReLU}, while we plot the $R^2$ values in \Cref{fig:convergenceR2AStandard,fig:convergenceR2AAdversarial,fig:convergenceR2AReLU,fig:convergenceR2BStandard,fig:convergenceR2BAdversarial,fig:convergenceR2BReLU,fig:convergenceR2CStandard,fig:convergenceR2CAdversarial,fig:convergenceR2CReLU}. We do not study the Brendel \& Bethge and the Carlini \& Wagner attacks due to the fact that the number of model calls varies depending on how many inputs are attacked at the same time. Note that, for attacks that do not have the a 100\% success rate, the mean adversarial example distance can increase with the number of steps as new adversarial examples (for inputs for which there were previously no successful adversarial examples) are added.

\definecolor{plot1}{HTML}{648FFF}
\definecolor{plot2}{HTML}{785EF0}
\definecolor{plot3}{HTML}{DC267F}
\definecolor{plot4}{HTML}{FE6100}
\definecolor{plot5}{HTML}{FFB000}

\begin{table}[]
\centering
\caption{Parameters for the Fast-100, Fast-1k and Fast-10k sets.}
\begin{tabular}{@{}llclllll@{}}
\toprule
\multicolumn{1}{c}{\multirow{2}{*}{\textbf{Attack}}} & \multicolumn{1}{c}{\multirow{2}{*}{\textbf{Parameter Name}}} & \multicolumn{3}{c}{\textbf{MNIST}} & \multicolumn{3}{c}{\textbf{CIFAR10}} \\ \cmidrule(l){3-8} 
\multicolumn{1}{c}{} & \multicolumn{1}{c}{} & \textbf{100} & \multicolumn{1}{c}{1k} & \textbf{10k} & \multicolumn{1}{c}{\textbf{100}} & \multicolumn{1}{c}{\textbf{1k}} & 10k \\ \midrule
\multirow{6}{*}{BIM} & Initial Search Factor & \multicolumn{6}{c}{N/A} \\
 & Initial Search Steps & \multicolumn{6}{c}{N/A} \\
 & Binary Search Steps & \multicolumn{1}{l}{10} & 20 & 20 & 10 & 20 & 20 \\
 & Starting $\varepsilon$ & \multicolumn{3}{c}{0.5} & \multicolumn{3}{c}{0.1} \\
 & \#Iterations & \multicolumn{1}{l}{10} & 50 & 500 & 10 & 50 & 500 \\
 & Learning Rate & \multicolumn{1}{l}{0.1} & 0.01 & 1e-3 & 0.01 & 1e-3 & 1e-3 \\ \midrule
\multirow{4}{*}{Deepfool} & \#Iterations & \multicolumn{1}{l}{100} & 500 & 500 & \multicolumn{3}{c}{500} \\
 & Candidates & \multicolumn{6}{c}{10} \\
 & Overshoot & \multicolumn{1}{l}{0.1} & 1e-5 & 1e-5 & \multicolumn{3}{c}{1e-4} \\
 & Loss & \multicolumn{6}{c}{Logits} \\ \midrule
\multirow{4}{*}{FGSM} & Initial Search Factor & \multicolumn{3}{c}{0.75} & \multicolumn{3}{c}{0.5} \\
 & Initial Search Steps & \multicolumn{3}{c}{30} & \multicolumn{3}{c}{10} \\
 & Binary Search Steps & \multicolumn{6}{c}{20} \\
 & Starting $\varepsilon$ & \multicolumn{3}{c}{1} & \multicolumn{3}{c}{0.1} \\ \midrule
\multirow{7}{*}{PGD} & Initial Search Factor & \multicolumn{1}{l}{N/A} & 0.5 & 0.5 & N/A & 0.5 & 0.75 \\
 & Initial Search Steps & \multicolumn{1}{l}{N/A} & 10 & 10 & N/A & 10 & 30 \\
 & Binary Search Steps & \multicolumn{3}{c}{10} & 10 & 10 & 20 \\
 & Starting $\varepsilon$ & \multicolumn{3}{c}{0.1} & 0.1 & 0.1 & 1 \\
 & \#Iterations & \multicolumn{1}{l}{10} & 50 & 500 & 10 & 50 & 200 \\
 & Learning Rate & \multicolumn{1}{l}{0.1} & 0.01 & 1e-3 & 0.01 & 1e-3 & 1e-3 \\
 & Random Initialization & \multicolumn{6}{c}{True} \\ \midrule
\multirow{5}{*}{Uniform Noise} & Initial Search Factor & \multicolumn{3}{c}{0.75} & 0.75 & 0.75 & 0.25 \\
 & Initial Search Steps & \multicolumn{3}{c}{30} & 30 & 30 & 5 \\
 & Binary Search Steps & \multicolumn{3}{c}{20} & 20 & 20 & 15 \\
 & Starting $\varepsilon$ & \multicolumn{3}{c}{1} & 1 & 1 & 0.5 \\
 & Runs & \multicolumn{1}{l}{200} & 500 & 200 & 10 & 50 & 500 \\ \bottomrule
\end{tabular}
\label{tab:convergenceParameters}
\end{table}

\newlength{\convergenceFigureSize}
\setlength{\convergenceFigureSize}{0.25\textwidth}
\newlength{\convergenceLegendSize}
\setlength{\convergenceLegendSize}{0.375\textwidth}

\foreach \arch\archName in {a/A,b/B,c/C}{
\foreach \test\testName in {standard/Standard,adversarial/Adversarial,relu/ReLU}{
\begin{figure}
\centering
\begin{subfigure}[]{\convergenceFigureSize}
\centering
    \begin{adjustbox}{width=\textwidth}
    \includegraphics{convergence-plots/mnist-\arch-\test-100.pdf}
    \end{adjustbox}
\caption[Ex2]%
{{\small MNIST \archName{} \testName{} Fast-100}}
\end{subfigure}
\hfill
\begin{subfigure}[]{\convergenceFigureSize}
\centering
    \begin{adjustbox}{width=\textwidth}
    \includegraphics{convergence-plots/cifar10-\arch-\test-100.pdf}
    \end{adjustbox}
\caption[Ex2]%
{{\small CIFAR10 \archName{} \testName{} Fast-100}}
\end{subfigure}
\hfill
\begin{subfigure}[]{\convergenceFigureSize}
\centering
    \begin{adjustbox}{width=\textwidth}
    \includegraphics{convergence-plots/mnist-\arch-\test-1000.pdf}
    \end{adjustbox}
\caption[Ex2]%
{{\small MNIST \archName{} \testName{} Fast-1k}}
\end{subfigure}

\begin{subfigure}[]{\convergenceFigureSize}
\centering
    \begin{adjustbox}{width=\textwidth}
    \includegraphics{convergence-plots/cifar10-\arch-\test-1000.pdf}
    \end{adjustbox}
\caption[Ex2]%
{{\small CIFAR10 \archName{} \testName{} Fast-1k}}
\end{subfigure}
\hfill
\begin{subfigure}[]{\convergenceFigureSize}
\centering
    \begin{adjustbox}{width=\textwidth}
    \includegraphics{convergence-plots/mnist-\arch-\test-10000.pdf}
    \end{adjustbox}
\caption[Ex2]%
{{\small MNIST \archName{} \testName{} Fast-10k}}
\end{subfigure}
\hfill
\begin{subfigure}[]{\convergenceFigureSize}
\centering
    \begin{adjustbox}{width=\textwidth}
    \includegraphics{convergence-plots/cifar10-\arch-\test-10000.pdf}
    \end{adjustbox}
\caption[Ex2]%
{{\small CIFAR10 \archName{} \testName{} Fast-10k}}
\end{subfigure}

\vspace{1mm}
\begin{subfigure}[]{\textwidth}
    \centering
    \begin{adjustbox}{width=\convergenceLegendSize}%
    \begin{tikzpicture}[show background rectangle]
        \begin{customlegend}[legend columns=5,legend style={align=center,draw=none,column sep=2ex},
        legend entries={BIM,Deepfool,Fast Gradient,PGD,Uniform}
        ]
        \addlegendimage{solid,line legend, draw=plot1}
        \addlegendimage{solid,line legend, draw=plot2}
        \addlegendimage{solid,line legend, draw=plot3}
        \addlegendimage{solid,line legend, draw=plot4}
        \addlegendimage{solid,line legend, draw=plot5}
        \end{customlegend}
    \end{tikzpicture}
    \end{adjustbox}
\end{subfigure}

\caption{Mean difference between the distance of the closest adversarial examples and the exact decision boundary distance for MNIST \& CIFAR10 \archName{} \testName{}. A dashed line means that the attack found adversarial examples (of any distance) for only some inputs, while the absence of a line means that the attack did not find any adversarial examples. The loosely and densely dotted black lines respectively represent the balanced and strong attack pools. Both axes are logarithmic.}
\label{fig:convergence\archName\testName}
\end{figure}
}
}

\foreach \arch\archName in {a/A,b/B,c/C}{
\foreach \test\testName in {standard/Standard,adversarial/Adversarial,relu/ReLU}{
\begin{figure}
\centering
\begin{subfigure}[]{\convergenceFigureSize}
\centering
    \begin{adjustbox}{width=\textwidth}
    \includegraphics{convergence-plots/r2/mnist-\arch-\test-100.pdf}
    \end{adjustbox}
\caption[Ex2]%
{{\small MNIST \archName{} \testName{} Fast-100}}
\end{subfigure}
\hfill
\begin{subfigure}[]{\convergenceFigureSize}
\centering
    \begin{adjustbox}{width=\textwidth}
    \includegraphics{convergence-plots/r2/cifar10-\arch-\test-100.pdf}
    \end{adjustbox}
\caption[Ex2]%
{{\small CIFAR10 \archName{} \testName{} Fast-100}}
\end{subfigure}
\hfill
\begin{subfigure}[]{\convergenceFigureSize}
\centering
    \begin{adjustbox}{width=\textwidth}
    \includegraphics{convergence-plots/r2/mnist-\arch-\test-1000.pdf}
    \end{adjustbox}
\caption[Ex2]%
{{\small MNIST \archName{} \testName{} Fast-1k}}
\end{subfigure}

\begin{subfigure}[]{\convergenceFigureSize}
\centering
    \begin{adjustbox}{width=\textwidth}
    \includegraphics{convergence-plots/r2/cifar10-\arch-\test-1000.pdf}
    \end{adjustbox}
\caption[Ex2]%
{{\small CIFAR10 \archName{} \testName{} Fast-1k}}
\end{subfigure}
\hfill
\begin{subfigure}[]{\convergenceFigureSize}
\centering
    \begin{adjustbox}{width=\textwidth}
    \includegraphics{convergence-plots/r2/mnist-\arch-\test-10000.pdf}
    \end{adjustbox}
\caption[Ex2]%
{{\small MNIST \archName{} \testName{} Fast-10k}}
\end{subfigure}
\hfill
\begin{subfigure}[]{\convergenceFigureSize}
\centering
    \begin{adjustbox}{width=\textwidth}
    \includegraphics{convergence-plots/r2/cifar10-\arch-\test-10000.pdf}
    \end{adjustbox}
\caption[Ex2]%
{{\small CIFAR10 \archName{} \testName{} Fast-10k}}
\end{subfigure}

\vspace{1mm}
\begin{subfigure}[]{\textwidth}
    \centering
    \begin{adjustbox}{width=\convergenceLegendSize}%
    \begin{tikzpicture}[show background rectangle]
        \begin{customlegend}[legend columns=5,legend style={align=center,draw=none,column sep=2ex},
        legend entries={BIM,Deepfool,Fast Gradient,PGD,Uniform}
        ]
        \addlegendimage{solid,line legend, draw=plot1}
        \addlegendimage{solid,line legend, draw=plot2}
        \addlegendimage{solid,line legend, draw=plot3}
        \addlegendimage{solid,line legend, draw=plot4}
        \addlegendimage{solid,line legend, draw=plot5}
        \end{customlegend}
    \end{tikzpicture}
    \end{adjustbox}
\end{subfigure}

\caption{$R^2$ of linear model for the heuristic adversarial distances given the exact decision boundary distances for  MNIST \& CIFAR10 \archName{} \testName{}. A dashed line means that the attack found adversarial examples (of any distance) for only some inputs, while the absence of a line means that the attack did not find any adversarial examples. The loosely and densely dotted black lines respectively represent the balanced and strong attack pools. The x axis is logarithmic.}
\label{fig:convergenceR2\archName\testName}
\end{figure}
}
}

\end{document}